\documentclass[10pt]{article}
\input{macros}
\usepackage{mathrsfs}
\usepackage{todonotes}

\allowdisplaybreaks

\definecolor{wjs}{RGB}{200,0,50}

\begin{document}

\title{When Does Model Collapse Occur in Structured Interactive Learning?
}
\author{%
	Yuchen Wu\thanks{The first two authors contribute equally to this work.}\,\,\thanks{School of Operations Research and Information Engineering,  Cornell University; \,\, Email: \href{mailto:yuchen.wu@cornell.edu}{\texttt{yuchen.wu@cornell.edu}}. } \and
 $\mbox{Kangjie Zhou}^{\ast}$\thanks{Department of Statistics, Columbia University; \,\, Email: \href{mailto:kz2326@columbia.edu}{\texttt{kz2326@columbia.edu}}. }  \and
Weijie Su\thanks{Department of Statistics and Data Science, University of Pennsylvania;  \,\, Email: {\href{mailto:suw@wharton.upenn.edu}{\texttt{suw@wharton.upenn.edu}}}. } 
} 
\date{\today}
\maketitle

\begin{abstract}

The proliferation of generative artificial intelligence has given rise to an \emph{interactive learning} environment, where model parameters are continuously updated using not only data generated by natural processes, but also synthetic outputs produced by other models.
This paradigm introduces two major challenges: (1) training data are no longer drawn exclusively from the target population, undermining a core assumption of classical statistical learning, and (2) model training processes become inherently correlated, as models interact with one another through repeated exposure to each other's synthetic outputs in a potentially complex manner. 
Establishing reliable statistical inference in such structured interactive learning environments therefore remains an important open problem.
In particular, there is growing concern about \emph{model collapse}, a phenomenon in which the performance of generative models progressively degrades as they are trained on synthetic data produced by earlier model generations.

Prior work on model collapse primarily focuses on a single model trained on its own output, failing to capture model performance in multi-model interactive settings. 
In this work, we fill this gap by investigating the performance of generative models in an interactive learning environment with general interaction patterns. 
In particular, we formalize model interactions using directed graphs and show that the occurrence of model collapse depends critically on the topology of the interaction graph. We further derive an explicit necessary and sufficient condition characterizing when model collapse occurs, and establish finite-sample results for linear regression and asymptotic guarantees for general M-estimators.
We support our theoretical findings through extensive numerical experiments.
\end{abstract}

\tableofcontents

\section{Introduction}\label{sec:intro}

Generative modeling approaches are a powerful class of machine learning methods with broad applications \citep{yang2023diffusion,openai2023gpt4,gemini2023,abramson2024accurate}, 
driving numerous innovations across various fields such as image and text generation, drug discovery, and healthcare. Traditionally, generative models are trained on samples drawn from a target distribution. Given independent and identically distributed (i.i.d.) observations from this distribution, a probabilistic model is learned to approximate the underlying ground-truth distribution and subsequently generate new samples.

The classical pipeline for training generative models has been extensively studied and is relatively well understood. However, its core assumptions may no longer hold in the modern data science era, where large-scale generative artificial intelligence (AI), such as GPT \citep{openai2023gpt4}, Gemini \citep{gemini2023}, DALL-E \citep{betker2023improving} and Stable Diffusion \citep{podellsdxl}, have revolutionized how data are generated and used. 
Generative models are increasingly being trained on not only data produced by human or natural processes,
but also on
\emph{synthetic data}. 
We use several examples below to illustrate this trend.


\paragraph{Evolution of online data ecosystem.}
Generative AI has become deeply integrated into our daily life. As a consequence, the internet is increasingly saturated with synthetic content that is indistinguishable from real data, and future training on data scraped from the web will inevitably be influenced by AI-generated data.
Indeed, emerging evidence suggests that synthetic data has not only contaminated internet content \citep{sun2025we,spennemann2025delving}, but has also infiltrated benchmark datasets \citep{alemohammad2024selfconsuming} and even human annotations produced via crowdsourcing \citep{veselovsky2023artificial}.

\paragraph{Accessing real data can be difficult.}
In certain applications, data generated by natural processes can be costly or hard to obtain. 
In such settings, synthetic data produced by generative AI offers a practical solution to mitigate data scarcity.
For example,
synthetic Electronic Health Record (EHR) data has been employed in scenarios where the original dataset is incomplete or inaccessible due to privacy concerns \citep{pmlr-v68-choi17a,rankin2020reliability,vaid2023implications}. 
Synthetic data also plays a central role in the development of Large Language Models (LLMs), enabling significant self-improvement through LLM-generated data \citep{huang2023large, lee2024rlaif}.
Beyond these settings, synthetic data has been widely adopted across diverse domains, including computer vision, finance, education, and agriculture. 
We refer the reader to \cite{lu2023machine} for a comprehensive review.

\paragraph{Knowledge distillation.}
Generative AI has demonstrated remarkable potential across a wide range of applications, yet its real-world deployment may be constrained by substantial computational demands. 
In particular, the high inference costs of these models can hinder adoption, especially in resource-constrained environments (e.g., mobile devices and edge computing platforms).
Knowledge distillation offers a promising solution to address this challenge \citep{hinton2015distilling,xu2024survey}. 
Specifically, the distillation process starts with a large ``teacher" model, which typically has high capacity but can be expensive to deploy.
During training, a smaller and more efficient ``student" model is trained on the outputs of the teacher model. 
The goal is to build a model with fewer parameters and reduced inference-time costs, while maintaining performance that is comparable to the teacher model on selected tasks.

\paragraph{}
Another feature of this synthetic learning regime is that the trained generative models are deployed to produce additional synthetic data, which may re-enter the training pool and be used to train subsequent generations of models. 
For example, a significant fraction of web-scraped data is now AI-generated, and such content may be incorporated into the training pipelines of future LLMs \citep{grattafiori2024llama,achiam2023gpt,team2024gemini}. 
This process naturally induces \emph{interdependent training cycles} whose dynamics differ significantly from those of classical regimes.
We refer to Figure \ref{fig:hybrid} for an illustration of this process. 
In particular, there is growing concern that continuously incorporating synthetic data into training may lead to \emph{model collapse} \citep{Shumailov2024AIMC}, a phenomenon in which model performance degrades progressively with each training cycle until it becomes useless\footnote{The definition of model collapse varies across contexts, see \cite{schaeffer2025position} for a comprehensive discussion. In this work, we define model collapse as realized risk being significantly larger than the risk achieved when training on the same amount of naturally generated data. See \cref{eq:risk-ratio} for a precise formulation. }.
Model collapse has been empirically observed in various applications \citep{2023Will,Arcaute2023CombiningGA,Martnez2023TowardsUT,2023Nepotistically,2024The,alemohammad2024selfconsuming,bertrand2024on},
and has been rigorously characterized for a number of standard machine learning methods \citep{Shumailov2024AIMC,kazdan2025collapse,dohmatob2024model,feng2025beyond}.

\begin{figure}[!ht]
    \centering
    \vspace{0.7em}\includegraphics[width=0.75\linewidth]{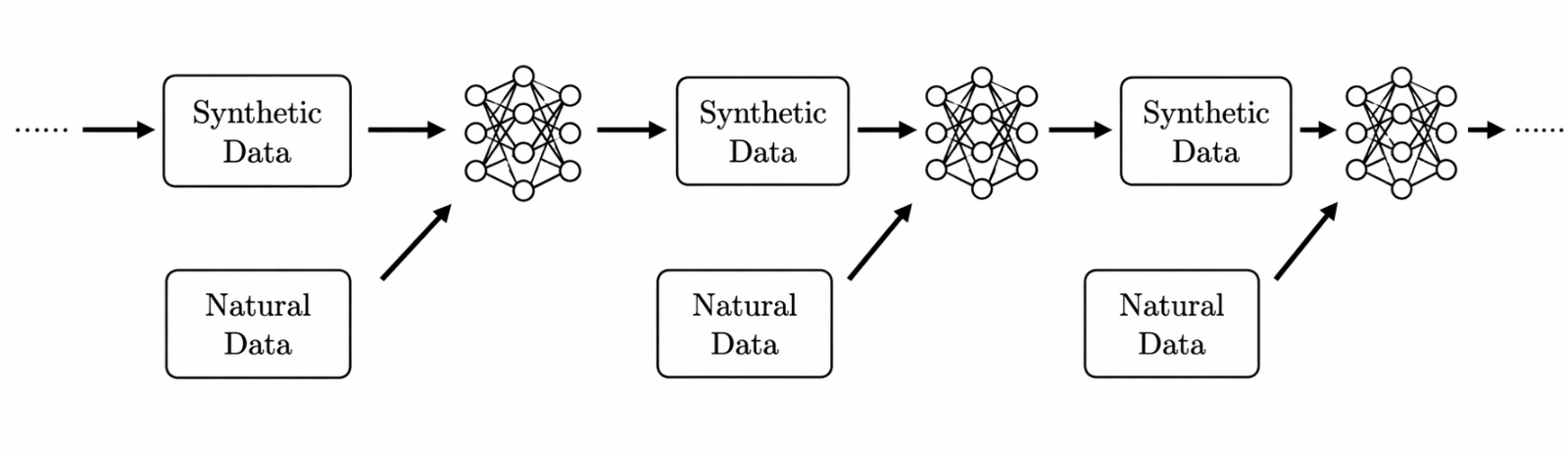}
    \caption{Illustration of the interdependent training cycle. In each iteration, models are updated using a mixture of synthetic and real data. Note that the synthetic data used for training may be generated by models from earlier generations, and not necessarily by the immediately preceding model. }
    \label{fig:hybrid}
\end{figure}

\subsection{Learning with self-generated content}
\label{sec:self_learning}

Theoretical guarantees for learning from natural or human-generated data abound  \citep{van2000asymptotic}, whereas comparatively fewer rigorous results exist for learning from synthetic data.
Prior work in this direction has largely focused on settings in which a single model is recursively trained on natural data and  \emph{its own outputs} \citep{alemohammad2024selfconsuming,Shumailov2024AIMC,Martnez2023TowardsUT,Arcaute2023CombiningGA,dohmatob2024model,dohmatob2024a,gerstgrasser2024is,bertrand2024on,feng2025beyond,kazdan2025collapse,he2025golden,garg2025preventing}.
Within this framework, one line of research examines the extreme setting where the training data at each iteration are entirely generated by the model from the previous iteration \citep{alemohammad2024selfconsuming,Shumailov2024AIMC,shumailov2023curse}.
This regime is known to suffer from severe model collapse. 
Another line of work considers a more practically relevant setting, where the model is updated at each training cycle using a mixture of clean and synthetic data. 
This hybrid training regime has been theoretically shown to mitigate model collapse for a broad class of learning methods, 
covering linear models \citep{gerstgrasser2024is,dohmatob2025strong,jain2024scaling}, Gaussian models \citep{kazdan2025collapse,he2025golden}, 
kernel methods \citep{jain2024scaling,kazdan2025collapse}, generalized linear models \citep{dey2024universality,he2025golden}, maximum likelihood estimation \citep{barzilai2025mle}, and softmax classifiers \citep{seddik2024how}. 
At a high level, model collapse tends to occur when the training process lacks sufficient access to natural data. 

\subsection{A structured interactive learning environment}

The papers discussed in Section \ref{sec:self_learning} focus primarily on the consequences of model training on a mixture of self-generated outputs and natural data, whereas relatively few works consider the more realistic setting in which models are trained on outputs generated also by other models.
Such \emph{interactive learning} regime captures a much broader class of practical applications. 
For instance, web-scale data generally consists of a mixture of real content and synthetic outputs produced by different generative AI models, as users may employ various AI tools to accomplish their tasks.
In knowledge distillation, the outputs of a teacher model are used to train a separate student model, whose generated content may subsequently enter the web and be used to update the teacher model.

In addition, in many real-world settings, interactions are structured and selective rather than all-to-all. 
For instance, in knowledge distillation, a student model learns exclusively from a preselected set of teacher models. 
In clinical applications, EHR foundation models trained on synthetic data typically rely on domain-specific synthetic EHR data rather than general web data. 
Moreover, language, regional, and licensing constraints may further restrict the data accessible to different AI developers.
The complexity of real-world model interaction highlights the need for a generic theoretical framework to capture the learning dynamics of modern generative AI systems, thereby raising the following important question:
\begin{center}
\emph{How to assess model performance in a structured interactive learning environment, where models are repeatedly trained on each other's synthetic outputs? }
\end{center}


\subsection{Main contributions}

In this paper, we develop a general framework for analyzing interactive learning under arbitrary interaction patterns. 
We summarize our main contributions below:

\paragraph{A structured interactive learning framework.} 
We introduce a general theoretical framework that models interactive learning among generative models using directed graphs, thereby enabling the representation of arbitrary interaction patterns between models.
As we show later, our framework is flexible enough to accommodate both synthetic and natural data sources.
Within this framework, we assume that models are updated using architectures similar to those of previous generations. 
For example, models in the GPT-1 through GPT-5 series descend from one another and partially inherit their predecessors' architectures \citep{radford2018improving,radford2019language,brown2020language,achiam2023gpt}. 

To the best of our knowledge, the only prior work on interactive learning is \cite{vu2025happens}, 
which studies linear regression under a dense interaction regime where every model learns from all others.
Their results suggest that interactive learning can improve model performance and homogenize model behaviors. 
Our framework generalizes their results by allowing for arbitrary model interaction patterns. 

\paragraph{Precise characterization of model collapse.} 
Under mild regularity conditions, we establish explicit necessary and sufficient conditions for model collapse (as defined in \cref{eq:risk-ratio}) in the interactive learning environment.
In particular, we show that model collapse depends critically on the topology of the interaction graph, i.e., whether a model receives information from ``unstable'' models that lack a path to ``stable'' data sources, such as natural data.
We present our findings formally in Section \ref{sec:setup}. 
Our results are consistent with prior work in single-model settings, where the occurrence of model collapse critically depends on the amount of natural data available to the model.

\paragraph{Comprehensive theoretical guarantee.} 
Our analysis covers a broad class of models used in modern machine learning. 
In particular, we provide a finite-sample characterization of model collapse under linear models, and an asymptotic characterization for general M-estimators learned via empirical risk minimization.
In both settings, we show that under mild assumptions, model collapse is governed by the topology of the interaction graph.
Our work provides theoretical insights into model behavior when multiple models co-evolve and the training datasets contain a mixture of natural data and synthetic data produced by different models.
A key takeaway of our results is that model error remains controllable as long as the interaction structure is sufficiently rich to ensure that every model receives information (directly or indirectly) from reliable data sources, regardless of the number of training iterations. 
This finding aligns with the empirical evidence that generative models can still be trained successfully even when the training data is fully contaminated with synthetic samples (for instance, in knowledge distillation, models are trained on entirely synthetic labels).


\subsection{Organization} 
The remainder of this paper is organized as follows. 
In \cref{sec:setup}, we introduce the interactive learning framework, using directed graphs to encode interactions between different models.
In \cref{sec:linear-regression}, we derive finite-sample results characterizing the occurrence of model collapse in the linear regression setting.
We derive in \cref{sec:M} asymptotic results for general M-estimators and illustrate them with concrete examples, such as generalized linear models (GLMs).
Finally, in \cref{sec:experiments}, we perform extensive numerical experiments to validate our theoretical findings across a range of models and network structures.  
Proofs of all main theorems and supporting lemmas are deferred to the appendices.

\subsection{Notation} 
For $n \in \N_+$, we let $[n] = \{1, 2, \cdots, n\}$. 
For two random objects $X$ and $Y$, we say $X \equiv Y$ if $\P(X = Y) = 1$, and write $X \perp\!\!\!\perp Y$ if $X$ is independent of $Y$. 
We use $1\{\cdot\}$ to represent indicator functions.
For $r \in \R$, we denote by $\lfloor r \rfloor $ the maximum integer that is no larger than $r$, and denote by $\lceil r \rceil$ the minimum integer that is no smaller than $r$. 
For $a, b \in \R$, we let $a \vee b = \max \{a, b\}$ and $a \wedge b = \min \{a, b\}$. 
For a matrix $A$, we denote by $\sigma_{\max}(A)$ the largest eigenvalue of $A$, and denote by $A^{\dagger}$ the pseudoinverse of $A$. 
For two symmetric matrices $A$ and $B$, we write $A \succeq B$ if $A - B$ is positive semidefinite, and $A \succ B$ if $A - B$ is positive definite. For matrices $A$ and $B$, we use $A \otimes B$ to represent the Kronecker product of $A$ and $B$.  We use $\norm{\cdot}_2$ to denote the Euclidean norm of a vector, and $\norm{\cdot}_{\op}$ and $\Tr (\cdot)$ to denote the operator norm and trace of a matrix, respectively.
For a set $S$ in Euclidean space, we denote its interior by $\operatorname{int} S$. 

\section{A directed graph framework for interactive learning}
\label{sec:setup}


We consider an interactive learning environment with $K$ models, where each model can learn from the outputs of other models.
Let the collection of models be $\cM = \{\mu_k: k \in [K]\}$. We represent their interactions using a directed graph $\cG = (\cM, \cE)$, 
where the nodes correspond to the models, 
and the set of directed edges $\cE$ encode their interactions.
For $\mu, \nu \in \cM$, we use $(\nu, \mu)$ to represent a directed edge connecting the origin $\nu$ to the destination $\mu$. 
We say $(\nu, \mu) \in \cE$ if and only if model $\mu$ learns from the outputs of model $\nu$. 
Self-learning is captured by allowing self-loops $(\mu, \mu) \in \cE$ whenever model $\mu$ learns from its own generated data, thereby encompassing the setting of \cref{sec:self_learning} (see \cref{ex:self_replace} for more details). 
Our framework also accommodates natural data sources, which can be represented as nodes in $\cG$ without incoming edges. 
Throughout this paper, we assume that all models share the common objective of learning a target population $P_{\mathrm{pop}}$.
Figure \ref{fig:graph} presents an illustrative example of an interaction graph. 

\begin{figure}[!ht]
\vspace{0.25cm}
    \centering    \includegraphics[width=0.36\linewidth]{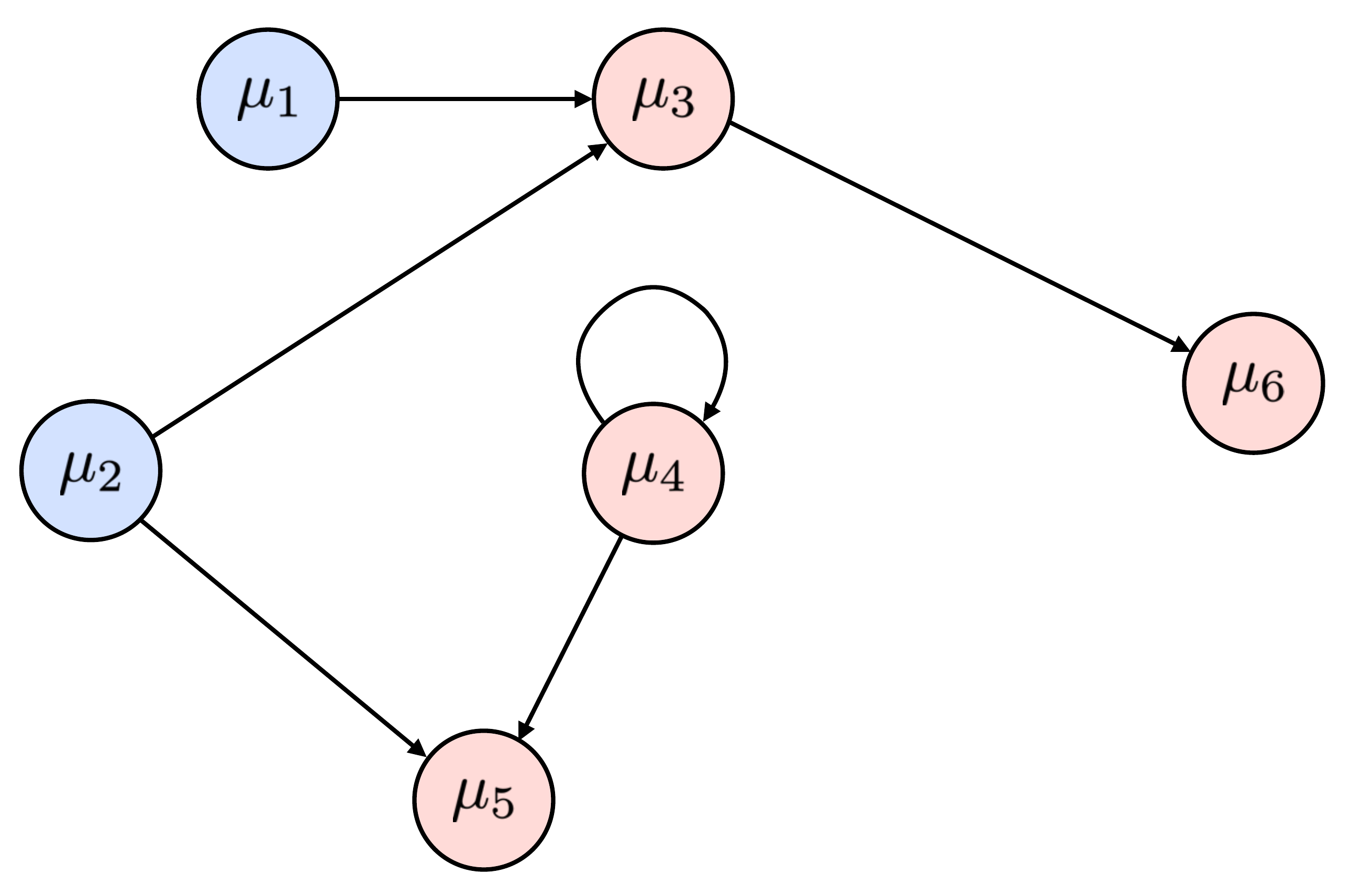}
    \vspace{0.25cm}
    \caption{An example of an interaction graph. In this example, $\cM = \{\mu_1, \mu_2, \mu_3, \mu_4, \mu_5, \mu_6\}$ and $\cE = \{(\mu_1, \mu_3), (\mu_2, \mu_3), (\mu_2, \mu_5),(\mu_3, \mu_6), (\mu_4, \mu_4), (\mu_4, \mu_5)\}$. 
    Nodes in $\cM_l$ are colored red, and nodes in $\cM_u$ are colored blue. 
    By definitions \eqref{eq:Ml-infty} and \eqref{eq:cM_cnc}, 
    we have $\cM_l^{\infty} = \{\mu_4\}$, $\cM_l^{\rm c} = \{\mu_4, \mu_5\}$, and $\cM_l^{\rm nc} = \{\mu_3, \mu_6\}$. 
    }
    \label{fig:graph}
\end{figure}

For $\mu \in \cM$, let $\Nin{\mu} = \{\nu: (\nu, \mu) \in \cE\}$ denote the set of models whose outputs are used to train $\mu$, and define $N_{\max} = \sup_{\mu \in \cM} |\Nin{\mu}|$. 
We set $\Nin{\mu} = \emptyset$ when $\mu$ does not update its parameters. 
We assume that interactions occur in rounds. 
Prior to any interactions, all models are initialized either using data drawn from the target distribution $P_{\mathrm{pop}}$ (corresponding to models trained on real data), or directly as $P_{\mathrm{pop}}$ itself, in which case the model represents a natural data source.
For $t \in \N$ and $\mu \in \cM$, we denote by $\hat P_{t, \mu}$ the learned distribution associated with model $\mu$ after the $t$-th interactive training cycle, where initialization is treated as the $0$-th cycle.
For each $t \in \N_+$, during the $t$-th interactive training cycle and for every directed edge $(\nu, \mu) \in \mathcal{E}$, model $\nu$ generates a dataset $\mathcal{D}_{t,\nu \to \mu}$ according to its current distribution $\hat P_{t-1,\nu}$, which is then passed to model $\mu$. 
Model $\mu$ then aggregates the datasets $\{\mathcal{D}_{t,\nu \to \mu} : \nu \in \Nin{\mu} \}$ to construct its updated distribution $\hat P_{t,\mu}$.
When $\Nin{\mu} = \emptyset$, model $\mu$ receives no incoming information and therefore does not update, and $\hat P_{t, \mu} = \hat P_{0, \mu}$ for all $t \in \N$. 

Depending on whether a model learns from other models, we partition $\cM$ into two disjoint subsets $\cM_l$ and $\cM_u$, where 
\begin{align}
\label{eq:MlMu}
\begin{split}
    & \cM_l = \left\{ \mu \in \cM: \, \Nin{\mu} \neq \emptyset \right\}, \\
    & \cM_u = \left\{ \mu \in \cM: \, \Nin{\mu} = \emptyset \right\}. 
\end{split}
\end{align}
For instance, for the interaction graph shown in Figure \ref{fig:graph}, 
we have $\cM_l = \{\mu_3, \mu_4, \mu_5, \mu_6\}$ and $\cM_u = \{\mu_1, \mu_2\}$. 
Models in $\cM_u$ do not update their parameters and can therefore be regarded as stable data sources. 
With these definitions, we summarize the interactive learning procedure as follows:
\begin{itemize}
    \item[-] \textbf{Initialization: } For $\mu \in \cM$, we either set $\hat P_{0, \mu} = P_{\mathrm{pop}}$ (in this case, $\mu \in \cM_u$ represents a natural data source), or learn $\hat P_{0, \mu}$ by training on a dataset $\cD_{0, \mu}$ drawn i.i.d. from $P_{\mathrm{pop}}$.
    We denote by $\cM_{\rm nature}$ the collection of natural data sources, and assume that $\cM_{\rm nature} \subseteq \cM_u$.  
    Namely, $\hat P_{0, \mu} = P_{\rm pop}$ for all $\mu \in \cM_{\rm nature} \subseteq \cM_u$, and $\hat P_{0, \mu}$ is learned from $\cD_{0, \mu}$ for all $\mu \in \cM \backslash \cM_{\rm nature}$. 
    \item[-] \textbf{Interactive training: }Recursively for $t \in \mathbb{N}_+$ and each directed edge $(\nu, \mu) \in \cE$, model $\nu$ generates a dataset $\cD_{t, \nu \to \mu}$ consisting of i.i.d. samples from $\hat P_{t - 1, \nu}$. 
    For each $\mu \in \cM_l$, we construct $\hat P_{t, \mu}$ by training on the aggregated dataset $\cup_{\nu \in \Nin{\mu}} \cD_{t, \nu \to \mu}$. 
    For all $\mu \in \cM_u$, we set $\hat P_{t, \mu} = \hat P_{t - 1, \mu}$. 
\end{itemize}

We study the evolution of $\hat P_{t, \mu}$ as $t \to \infty$, focusing on characterizing its distance to the target distribution $P_{\mathrm{pop}}$.
When $\hat P_{t,\mu}$ is trained on $n_{t,\mu}$ samples, a natural benchmark is the oracle estimator $\hat P^{\ast}_{t,\mu}$, obtained by applying the same training algorithm to $n_{t,\mu}$ i.i.d. samples drawn from $P_{\mathrm{pop}}$.
Specifically, 
we compare the risk $r_{t, \mu}$ associated with $\hat P_{t,\mu}$   
with the oracle risk $r_{t, \mu}^{\ast}$ associated with $\hat P^{\ast}_{t,\mu}$, where the precise definition of the risk depends on the specific context.  
We say that \emph{model collapse occurs for model $\mu$} if 
\begin{align}
\label{eq:risk-ratio}
    \limsup_{t \to \infty} \frac{r_{t, \mu}}{r_{t, \mu}^{\ast}} = \, \infty. 
\end{align}
As discussed in \cite{schaeffer2025position}, the definition of model collapse is context-dependent. In this work, we adopt the definition in \cref{eq:risk-ratio}. 

By construction, models in $\cM_u$ retain fixed estimates and therefore do not evolve over time.
For models in $\cM_l$, our results show that with reasonable risk choices, whether the risk ratio $r_{t,\mu} / r^{\ast}_{t,\mu}$ remains bounded as $t \to \infty$ 
depends critically on the structure of the interaction graph $\cG$.
%
To present our results, we define $\cM_l^{\infty}$ as the subset of nodes in $\cM_l$ that are not reachable from any node in $\mathcal M_u$:
\begin{align}
\label{eq:Ml-infty}
    \cM_l^{\infty} = \, \left\{\, \mu \in \cM_l:\, \forall \nu \in \cM_u, \mbox{ there is no directed path from } \nu \mbox{ to } \mu \, \right\}. 
\end{align}
Intuitively, $\mathcal M_l^{\infty}$ consists of the models in $\mathcal M_l$ that do not receive information from stable data sources. Taking the directed graph in Figure \ref{fig:graph} as an example, $\cM_l^{\infty} = \{\mu_4\}$ since there is no directed path from $\{ \mu_1, \mu_2 \}$ to $\mu_4$.
Our main results show that, for a model $\mu \in \cM_l$, whether it collapses or not is completely determined by the existence of a path from some model in $\mathcal M_l^{\infty}$ to $\mu$.
Partitioning $\mathcal M_l$ into two disjoint subsets, $\mathcal M_l^{\mathrm{c}}$ and $\mathcal M_l^{\mathrm{nc}}$:    
\begin{align}
\label{eq:cM_cnc}
\begin{split}
    & \cM_l^{\rm c}  = \,  \cM_l^{\infty} \cup \left\{ \,\mu \in \cM_l: \exists \nu \in \cM_l^{\infty},\, \mbox{such that there is a directed path from } \nu \mbox{ to } \mu \, \right\}, \\
    & \cM_l^{\rm nc} = \,  \left\{\, \mu \in \cM_l \backslash \cM_l^{\infty}: \forall \nu \in \cM_l^{\infty},\, \mbox{ there is no directed path from } \nu \mbox{ to } \mu \, \right\},
\end{split}
\end{align}
%
we establish that across a wide range of scenarios, models in $\cM_l^{\mathrm{c}}$ suffer from model collapse, whereas those in $\cM_l^{\mathrm{nc}}$ do not. For the network presented in Figure~\ref{fig:graph}, this means that $\mu_4, \mu_5 \in \cM_l^{\rm c}$ will collapse and $\mu_3, \mu_6 \in \cM_l^{\rm nc}$ will not collapse. We present a rigorous statement of these findings for linear regression in Section~\ref{sec:linear-regression}, followed by an extension to general M-estimators in Section~\ref{sec:M}. 

An important implication of our results is that, the set of models that will collapse in an interactive learning environment is uniquely determined by the graph structure, a property that is universal across many distinct learning scenarios. At a high level, one can think of $\mathcal{M}_l^{\infty}$ as the subset of models in $\mathcal{M}_l$ that do not receive information from stable data sources, causing their performance to vary widely. Consequently, models in $\mathcal{M}_l^{\mathrm{c}}$ are affected by these ``unstable'' data sources, leading to their collapse as the number of training cycles approaches infinity. Conversely, we demonstrate that models in $\mathcal{M}_l^{\mathrm{nc}}$ do not collapse because they exclusively receive information from ``stable'' sources. 
This dichotomy reveals the fundamental rationale behind model collapse in an interactive environment: it is not an inevitable consequence of learning from synthetic data, but rather hinges on contamination from unstable data sources.
For example, as shown in Figure~\ref{fig:graph}, model $\mu_6 \in \cM_l^{\mathrm{nc}}$ does not collapse despite training exclusively on synthetic data generated by $\mu_3$, since $\mu_3$ receives information only from stable data sources under our definition.

\subsection{Examples}
\label{sec:examples}


In this section, we present several illustrative examples of interaction graphs.

\begin{exm}[Complete replacement training regime]\label{ex:self_replace}
    In the complete replacement training regime, a single model is retrained at each cycle using data generated by itself with parameters from the previous round. The resulting training dynamics can thus be represented by a single-node interaction graph with $\cM = \{\mu\}$ and $\cE = \{(\mu, \mu)\}$.
    
\end{exm}


\begin{exm}[Accumulating training regime]

In this setting, a single model is retrained at each cycle using all data generated in previous rounds. 
Specifically, we denote by $\hat P_t$ the distribution estimate after training round $t$, with $\hat P_0$ representing the natural distribution. For $t \in \mathbb{N}_+$, the estimate $\hat P_t$ is obtained by training on the aggregated dataset $\cup_{i = 0}^{t - 1} \,\mathcal{D}_i$, where each $\mathcal{D}_i$ is sampled from $\hat P_i$.
We assume that training proceeds for a total of $T$ rounds.

We next show that this accumulating training regime can be represented by the interaction graph in Figure~\ref{fig:accumulate}, where there is a directed edge from $\mu_a$ to $\mu_b$ for any $a < b$.
To see this, note that $\mu_1$ in Figure \ref{fig:accumulate} may represent a natural data source, so that $\hat P_{s, \mu_1} = \hat P_0$ for all $s \in \mathbb{N}$. 
If we further assume that $\mu_1$ generates the same dataset $\mathcal{D}_0$ and passes it to all $\mu_i$ with $i \geq 2$ in every training cycle (indeed, $\mathcal{D}_0$ follows $\hat P_0 = \hat P_{s, \mu_1}$), then for all $s \geq 1$, we have $\hat P_{s, \mu_2} = \hat P_1$. 
We may therefore assume that $\mu_2$ generates the dataset $\mathcal{D}_1$ and passes it to all $\mu_i$ with $i \geq 3$ starting from round 2.  
Therefore, for all $s \geq 2$ we have $\hat P_{s, \mu_3} = \hat P_2$. 
Continuing this construction inductively, we conclude that for all $s \geq t - 1$ we have $\hat P_{s, \mu_t} = \hat P_{t - 1}$. This holds for all $t \in [T]$, hence capturing $\hat P_i$ for all $i \in \{0\} \cup [T - 1]$.



\end{exm}
    \begin{figure}[!ht]
        \centering
        \includegraphics[width=0.5\linewidth]{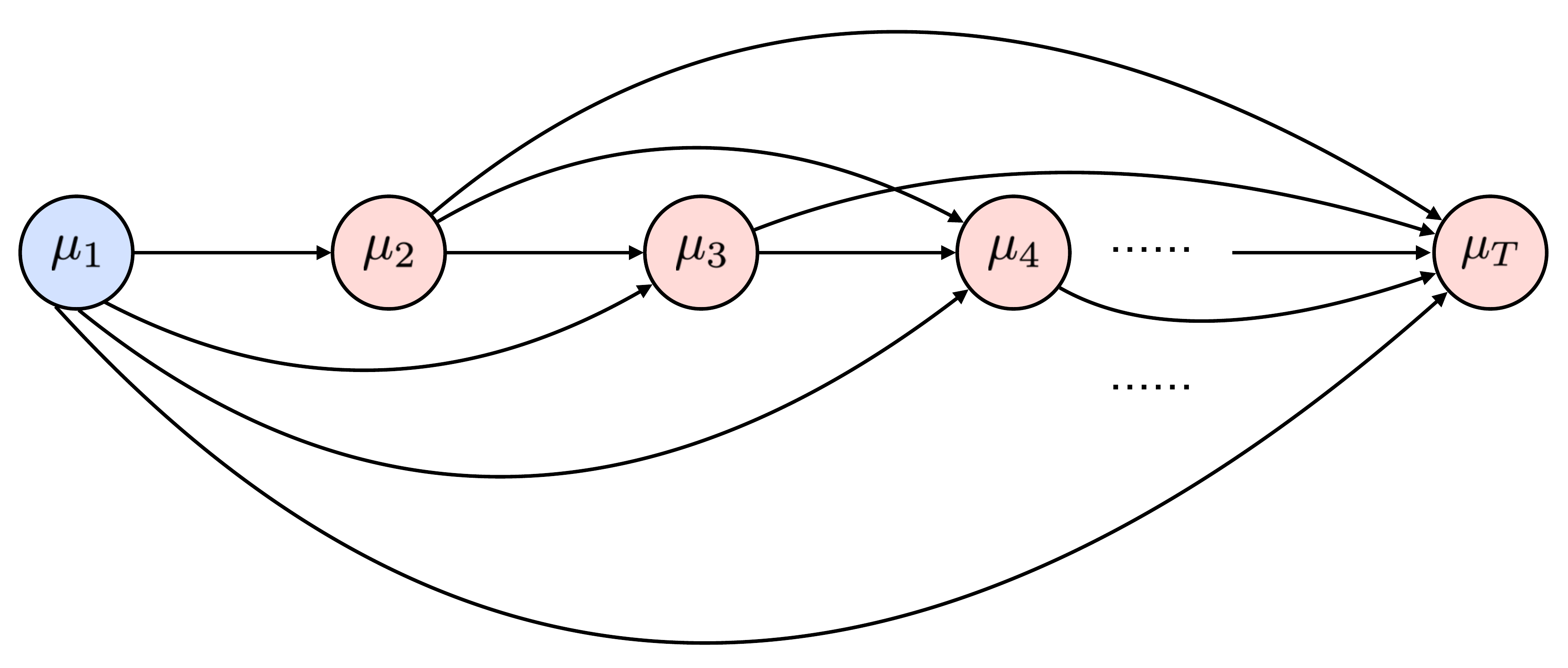}
        \vspace{0.5em}\caption{Interaction graph that represents the accumulating training regime. }
        \label{fig:accumulate}
    \end{figure}

\begin{exm}[Two patterns of model collapse]
\label{exm:two-patterns}
    In many real-world scenarios, interactions are highly structured, leading to different patterns of model collapse.  
 In this example, we consider a 5-node network $\cG = (\cM, \cE)$, where
    \begin{equation*}
        \cM = \, \left\{ \mu_i : 1 \le i \le 5 \right\}, \quad \cE = \, \left\{ (\mu_1, \mu_2), (\mu_5, \mu_2) \right\} \cup \left\{ (\mu_i, \mu_j): 3 \le i, j \le 5, i \neq j \right\}.
    \end{equation*}
    Namely, $\mu_3, \mu_4$ and $\mu_5$ distill from one another, $\mu_1$ represents a natural data source, and $\mu_2$ learns from both natural data and synthetic data produced by model $\mu_5$. 
    See Figure~\ref{fig:exm3} for an illustration. 
    In this example, we observe two types of collapsing models:
    \begin{itemize}
    \item \emph{Models that do not receive information from $\mu_1$: } Note that models $\mu_3, \mu_4$ and $\mu_5$ do not receive information (directly or indirectly) from the only stable data source $\mu_1$, hence they belong to the set $\cM_l^{\infty}$ and will collapse as the number of training cycles $t$ goes to infinity.
\item \emph{Models that receive information from $\mu_1$: }
Model $\mu_2$ is trained on a mixture of natural data produced by $\mu_1$ and synthetic data produced by $\mu_5$, and the proportion of natural data in the training set can be large\footnote{As we will see later, it can be any positive constant strictly less than 1.}.
Despite this, $\mu_2$ still collapses because it receives information from an unstable source $\mu_5$. This phenomenon is not captured by results in the single-model setting, where sufficient access to natural data can typically prevent model collapse. 
    \end{itemize}
\end{exm}

\begin{figure}[!ht]
\vspace{-0.5em}
    \centering
\includegraphics[width=0.33\linewidth]{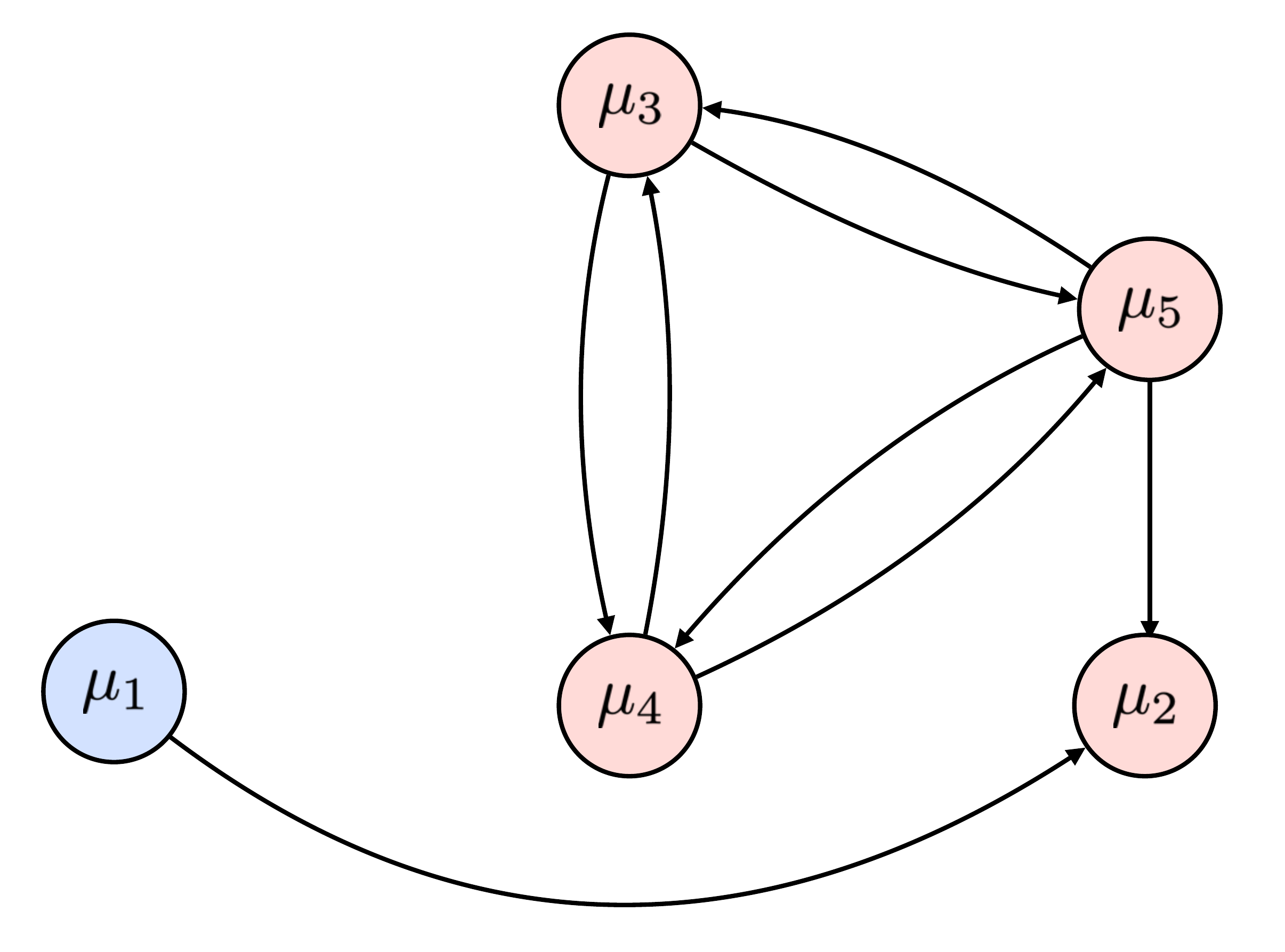}
    \caption{The $5$-node interaction graph that appears in Example \ref{exm:two-patterns}.}
    \label{fig:exm3}
\end{figure}

\begin{exm}[Model collapse can be sensitive to learning patterns]
\label{exm:one-diff}
    In this example, we show that model collapse can be sensitive to small changes in learning patterns. 
    Specifically, consider the two interaction graphs in the left and right panels of Figure \ref{fig:one-diff}, which differ by only a single edge (the left panel is obtained by removing the edge $\mu_2 \to \mu_1$ in the right panel).
    The left panel presents a hierarchical distillation learning pattern, where the model $\mu_1$ does not update its parameters and serves as a stable data source, and for every $2 \le i \le 5$, the model $\mu_i$ distills from $\{\mu_j: 1 \leq j \leq i - 1\}$.  
    In addition, we assume that apart from distilling from $\mu_4$ and $\mu_5$, model $\mu_3$ is also trained on data produced by a natural data source $\mu_6$. 
    Since $\mu_1$ is a stable data source and $\mu_2$, $\mu_3$, $\mu_4$ and $\mu_5$ all receive information from it, no model collapses in the left panel.

    In the right panel, we add a single edge $(\mu_2, \mu_1)$, which moves $\mu_1$ from the stable data source set $\mathcal{M}_u$ to the unstable data source set $\mathcal{M}_l^{\infty}$. Consequently, $\mu_2$, $\mu_3$, $\mu_4$ and $\mu_5$ all receive information from unstable sources, and all models collapse in the right panel (except for the natural data source $\mu_6$). 
    
    This example also suggests that models should be selective about the datasets they learn from, as blindly incorporating synthetic data may be detrimental. 
\end{exm}

\begin{figure}[!ht]
    \centering
\includegraphics[width=0.5\linewidth]{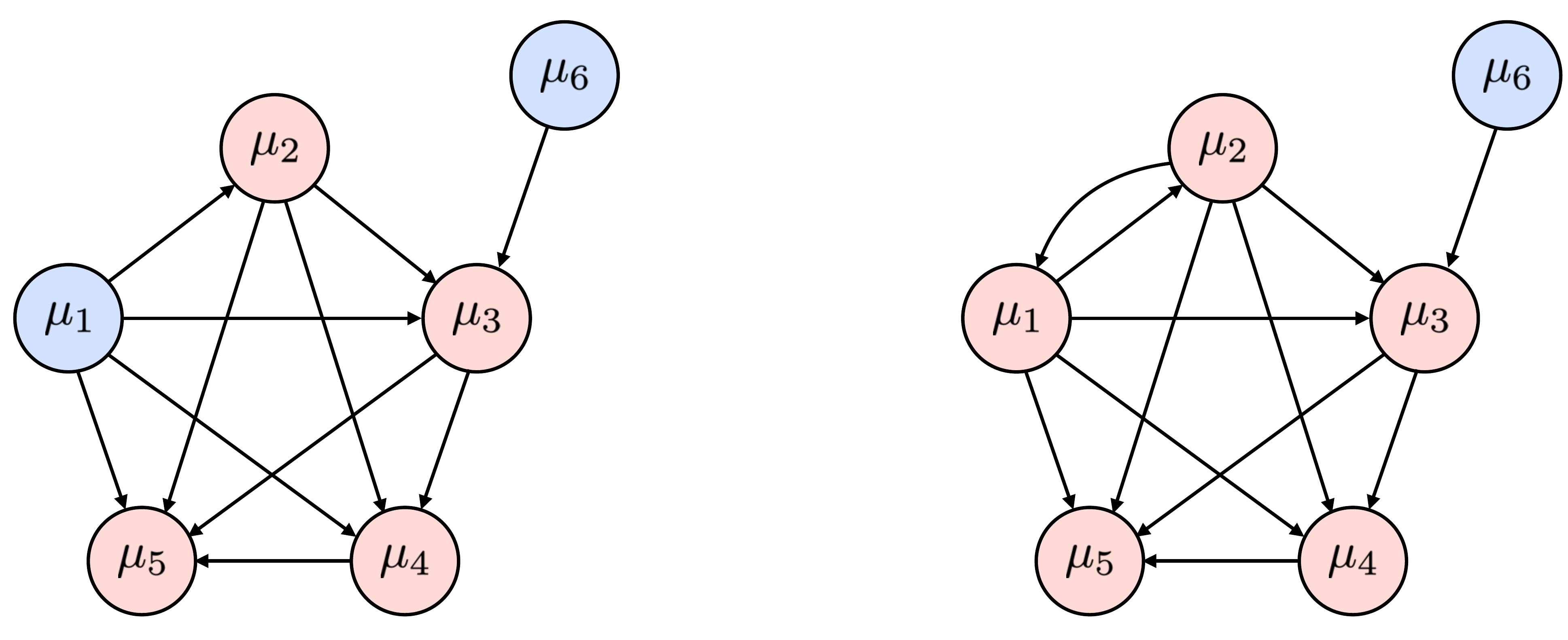}
    \caption{The interaction graph that appears in Example \ref{exm:one-diff}.}
    \label{fig:one-diff}
\end{figure}

\section{Results for linear regression}
\label{sec:linear-regression}

In this section, we present our main results for linear regression. 
Linear regression is a classical statistical model that relates a response vector $y \in \R^n$ to a covariate matrix $X \in \R^{n \times d}$ through a linear relationship: 
\begin{align}
\label{eq:linear-model}
    y = X \beta_{\ast} + \varepsilon. 
\end{align}
Here, $\beta_{\ast} \in \R^d$ is an unknown parameter vector, 
and $\varepsilon \in \R^n$ represents measurement noise. 
Given observations $(y, X)$, the goal is to estimate $\beta_{\ast}$ that defines the conditional distribution of $y$ given $X$. 

The most standard method for fitting model \eqref{eq:linear-model} is ordinary least squares (OLS). 
Specifically, when $X$ has full column rank, the OLS estimate for $\beta$ admits an explicit form 
\begin{align}
\label{eq:OLS}
    \hat\beta_{\rm OLS} = (X^\top X)^{-1} X^\top y. 
\end{align}
OLS forms the foundation of many classical and modern regression methods.
In this section, we study interactive learning in the linear regression setting using the OLS estimator.
As in other discriminative modeling frameworks, our objective is to model the conditional distribution of the response given the predictors, rather than the full joint distribution.
For linear regression, this means that $P_{\rm pop}$ in Section \ref{sec:setup} denotes the ground-truth conditional distribution parameterized by $\beta_{\ast}$, and $\hat P_{t, \mu}$ denotes an estimated conditional distribution parameterized by $\hat\beta_{t, \mu}$.

In what follows, we carry out a finite-sample analysis for linear regression. 
Asymptotic results for general M-estimators, including OLS, are deferred to Section~\ref{sec:M}.

\subsection{Interactive learning with linear regression}
\label{sec:interactive-linear}

We now describe the interactive learning regime for linear regression, under a given interaction graph $\cG = (\cM, \cE)$. 
This serves as a concrete example of the general setup introduced in Section~\ref{sec:setup}.

At initialization, for each model $\mu \in \cM$, we either set $\hat\beta_{0, \mu} = \beta_{\ast}$, in which case $\mu$ represents a natural data source (and hence $\mu \in \cM_{\rm nature} \subseteq \cM_u$), 
or compute the OLS estimator $\hat\beta_{0, \mu}$ using nature-generated data $(y_{0, \mu}, X_{0, \mu})$. 
In particular, assuming $X_{0, \mu}$ has full column rank, we compute
\begin{align}
\label{eq:linear-round0}
\begin{split}
    & y_{0, \mu} = X_{0, \mu} \beta_{\ast} + \varepsilon_{0, \mu} \in \mathbb{R}^{n_{0, \mu}}, \\
    & \hat\beta_{0, \mu} = (X_{0, \mu}^{\top} X_{0, \mu})^{-1} X_{0, \mu}^{\top} y_{0, \mu}.
\end{split}
\end{align}
At the $t$-th training cycle for $t = 1, 2, \cdots$, models in $\cM_u$ remain unchanged, whereas each model $\mu$ in $\cM_l$ updates its parameters using data generated by the models in $\Nin{\mu}$ (assuming $\sum_{\nu \in \Nin{\mu}} X_{t, \nu \to \mu}^{\top} X_{t, \nu \to \mu}$ has full column rank): 
\begin{align}\label{eq:learning_update}
\begin{split}
    & \hbeta_{t, \mu} = \Big( \sum_{\nu \in \Nin{\mu}} X_{t, \nu \to \mu}^{\top} X_{t, \nu \to \mu} \Big)^{-1} \sum_{\nu \in \Nin{\mu}} X_{t,\nu \to \mu}^{\top} y_{t, \nu \to \mu}, \\
    & y_{t, \nu \to \mu} = X_{t, \nu \to \mu} \hbeta_{t - 1, \nu} + \eps_{t, \nu \to \mu} \in \R^{n_{t, \nu \to \mu}}. 
\end{split}
\end{align}
Here, $X_{t, \nu \to \mu}$ is an $n_{t, \nu \to \mu} \times d$ data matrix.
The estimator $\hat\beta_{t, \mu}$ corresponds to the OLS solution obtained by stacking the datasets $(y_{t, \nu \to \mu}, X_{t, \nu \to \mu})$ over all $\nu \in \Nin{\mu}$. 
Note that the training datasets $(y_{t, \nu \to \mu}, X_{t, \nu \to \mu})$ for different $(\nu, \mu)$ are not necessarily independent.
\begin{exm}\label{exm:two_node_linear}
    To illustrate this updating rule for the parameters $\{ \hbeta_{t, \mu} \}_{t \ge 1, \mu \in \cM}$, we consider a simple two-node network: $\cM = \{\mu_1, \mu_2\}$ and $\cE = \{(\mu_1, \mu_2), (\mu_2, \mu_2)\}$. In this setup, $\mu_1$ acts as a static natural data source, while $\mu_2$ learns from both $\mu_1$ and its own outputs from the previous training cycle. The general update equations~\eqref{eq:learning_update} for $t \ge 1$ then reduce to:
    $$
    \begin{aligned}
    \hbeta_{t, \mu_1} &= \hbeta_{t-1, \mu_1} = \cdots = \hbeta_{0, \mu_1} = \beta_{\ast}, \\
    \hbeta_{t, \mu_2} &= \Big( X_{t, \mu_1 \to \mu_2}^{\top} X_{t, \mu_1 \to \mu_2} + X_{t, \mu_2 \to \mu_2}^{\top} X_{t, \mu_2 \to \mu_2} \Big)^{-1} \Big( X_{t, \mu_1 \to \mu_2 }^{\top} y_{t, \mu_1 \to \mu_2 } + X_{t, \mu_2 \to \mu_2 }^{\top} y_{t, \mu_2 \to \mu_2 } \Big),
\end{aligned}
$$
where for each $k \in \{1, 2\}$, the dataset $(X_{t, \mu_k \to \mu_2 }, y_{t, \mu_k \to \mu_2 }) \subseteq (X_{t, \mu_k }, y_{t, \mu_k })$ is generated via:
$$
y_{t, \mu_k \to \mu_2 } = X_{t, \mu_k \to \mu_2 } \hbeta_{t - 1, \mu_k} + \eps_{t, \mu_k \to \mu_2}, \quad (X_{t, \mu_k \to \mu_2 }, \eps_{t, \mu_k \to \mu_2}) \subseteq (X_{t, \mu_k }, \eps_{t, \mu_k }).
$$
\end{exm}

\subsection{A matrix-form update equation}
\label{sec:matrix-form}

In the linear regression setting of Section~\ref{sec:interactive-linear}, the update equations for individual model parameters are given in \cref{eq:learning_update}. To establish our main results, we next assemble these equations to form a single matrix-form update equation.
We begin by introducing some notation.
For $t \in \N$ and $\mu, \nu \in \cM$, we define 
\begin{align}
\label{eq:cT-def}
    \cT_{t, \mu, \nu} = \left\{ \begin{array}{ll}
       \big( \sum_{m \in \Nin{\mu}} X_{t, m \to \mu }^{\top} X_{t, m \to \mu} \big)^{-1} X_{t, \nu \to \mu}^{\top} X_{t, \nu \to \mu}  &  \mbox{if }  t \in \N_+,\, \mu \in \cM_l, \,\mbox{and }\nu \in \Nin{\mu}, \\
       I_d & \mbox{if } t \in \N_+, \, \mu \in \cM_u, \,\mbox{and } \mu = \nu, \\
       I_d & \mbox{if }t = 0 \mbox{ and }\mu = \nu,  \\
        0_{d \times d} & \mbox{otherwise. }
    \end{array} \right.
\end{align}
The $d \times d$ matrix $\cT_{t, \mu, \nu}$ quantifies the information transmitted from model $\nu$ to model $\mu$ during training cycle $t$.  
The identity matrix represents full information transfer (e.g., copying parameter values from the source model), whereas the zero matrix indicates no information transfer.

When the rows of each design matrix are i.i.d.\ according to a common distribution with invertible second-moment matrix $\Sigma$, and the sample size is sufficiently large, the law of large numbers implies concentration:  $X_{t, \nu \to \mu}^{\top} X_{t, \nu \to \mu} \approx n_{t, \nu \to \mu} \Sigma$. 
Applying this approximation for all models in $\Nin{\mu}$, we obtain $ \sum_{m \in \Nin{\mu}} X_{t, m \to \mu }^{\top} X_{t, m \to \mu}  \approx (\sum_{m \in \Nin{\mu}} n_{t, m \to \mu}) \Sigma$. 
Consequently, for $t \in \N_+$ and $(\nu, \mu) \in \cE$, we have $\cT_{t, \mu, \nu} \approx n_{t, \nu \to \mu} / (\sum_{m \in \Nin{\mu}} n_{t, m \to \mu}) I_d$. 
This observation motivates the definition of $\cT_{t, \mu, \nu}^{\ast}$ as the ``population'' version of $\cT_{t, \mu, \nu}$: 
\begin{align}
\label{eq:cT-star-def}
    \cT_{t, \mu, \nu}^{\ast} = \left\{ \begin{array}{ll}
       ( n_{t, \nu \to \mu} / (\sum_{m \in \Nin{\mu}} n_{t, m \to \mu} )) I_d  &  \mbox{if }  t \in \N_+,\, \mu \in \cM_l, \,\mbox{and }\nu \in \Nin{\mu}, \\
       I_d & \mbox{if } t \in \N_+, \, \mu \in \cM_u, \,\mbox{and } \mu = \nu, \\
       I_d & \mbox{if }t = 0 \mbox{ and }\mu = \nu,  \\
        0_{d \times d} & \mbox{otherwise. }
    \end{array} \right.
\end{align}
Recall that $\cM = \{\mu_1, \mu_2, \cdots, \mu_K\}$. 
For $t \in \N$, we define  $\cT_t \in \R^{dK \times dK}$ (resp. $\cT_t^{\ast} \in \R^{dK \times dK}$) as a $K \times K$ block matrix with the $(i, j)$-th block being $\cT_{t, \mu_i, \mu_j}$ (resp. $\cT_{t, \mu_i, \mu_j}^{\ast}$).
We also define the collection of estimated coefficients at time $t$ as:
\begin{align}
\label{eq:collected-betat}
    \hbeta_t = (\hbeta_{t, \mu_1}^{\top}, \hbeta_{t, \mu_2}^{\top}, \cdots, \hbeta_{t, \mu_K}^{\top})^{\top} \in \R^{dK}. 
\end{align}
We make the convention that $\hat\beta_{-1} = (\beta_{\ast}^{\top}, \beta_{\ast}^{\top}, \cdots, \beta_{\ast}^{\top})^{\top} \in \mathbb{R}^{dK}$.

For $\mu \in \cM_l$ and $t \in \N_+$, define the ``mixed noise vector''
\begin{align}
\label{eq:v-t-mu}
    v_{t, \mu} = \, \Big( \sum_{\nu \in \Nin{\mu}} X_{t, \nu \to \mu}^{\top} X_{t, \nu \to \mu} \Big)^{-1} \sum_{\nu \in \Nin{\mu}} X_{t, \nu \to \mu}^{\top} \eps_{t, \nu \to \mu} \in \R^d.
\end{align}
For $\mu \in \cM_u$ and $t \in \N_+$, we simply set $v_{t, \mu} = 0_d$.
When $t = 0$, we let $v_{0, \mu} = 0_d$ for $\mu \in \cM_{\rm nature}$, and for $\mu \in \cM \setminus \cM_{\rm nature}$, we define:
\begin{align}
\label{eq:v-t-mu2}
	v_{0, \mu} = \big( X_{0, \mu}^{\top} X_{0, \mu} \big)^{-1} X_{0, \mu}^{\top} \, \varepsilon_{0, \mu} \in \R^d. 
\end{align}
Let $v_t = (v_{t, \mu_1}^\top,  v_{t, \mu_2}^{\top}, \cdots, v_{t, \mu_K}^{\top})^{\top}\in \R^{dK}$. 
With the definitions in \eqref{eq:cT-def}, \eqref{eq:collected-betat}, and \eqref{eq:v-t-mu}, we obtain the following recursive formula for all $t \in \N$: 
\begin{align}
\label{eq:hbeta-update}
    \hbeta_t = \, \cT_t \hbeta_{t-1} + v_t.
\end{align}
Note that for $t \in \N$, the matrix $\cT_t^{\ast}$ can be expressed as the Kronecker product of a $K \times K$ matrix $P_t$ and a $d \times d$ identity matrix. 
For $i, j \in [K]$, let $P_{t, \mu_i, \mu_j}$ denote the $(i,j)$-th entry of $P_t$, defined as follows:
\begin{equation}\label{eq:markov_transition_def}
    P_{t, \mu, \nu} = \left\{ \begin{array}{ll}
        n_{t, \nu \to \mu} / (\sum_{m \in \Nin{\mu}} n_{t, m \to \mu}) &  \mbox{if }  t \in \N_+,\, \mu \in \cM_l, \,\mbox{and }\nu \in \Nin{\mu}, \\
       1 & \mbox{if }t \in \N_+, \, \mu \in \cM_u, \,\mbox{and } \mu = \nu, \\
       1 & \mbox{if }t = 0 \mbox{ and }\mu = \nu,  \\
        0 & \mbox{otherwise. }
    \end{array} \right.
\end{equation}
It follows that $\cT_t^* = P_t \otimes I_d$.

\subsection{Assumptions for the linear regression setting}
\label{sec:assumptions-linear}

In this section, we state the assumptions required to establish our linear regression results, beginning with the assumption of temporal independence. 

\begin{ass}[Temporal independence]
\label{ass:independent_covariate}
    For $t \in \N_+$, let $\cX_t = (X_{t, \nu \to \mu})_{(\nu, \mu) \in \cE}$ denote the collection of design matrices from the $t$-th training cycle, and let $E_t = (\eps_{t, \nu \to \mu})_{(\nu, \mu) \in \cE}$ denote the corresponding collection of noise vectors.
    When $t = 0$, we define $\cX_0 = (X_{0, \mu})_{\mu \in \cM \backslash \cM_{\rm nature}}$ and $E_0 = (\eps_{0, \mu})_{\mu \in \cM \backslash \cM_{\rm nature}}$. 
    We assume that $(\cX_t, E_t)$ are mutually independent for all ${t \in \N}$.     
\end{ass}

\begin{rem}
    Assumption \ref{ass:independent_covariate} does not preclude deterministic design matrices: the covariates may be fixed or random, and in the deterministic case, independence holds trivially. 
\end{rem}

We next state the assumptions imposed on the data transmitting pattern. 
In words, Assumption \ref{assumption:data-distribution} states that at each round, every model $\nu$ generates data using the parameters fitted in the previous round and distributes a subset of this data to every model $\mu$ for which $(\nu, \mu) \in \cE$.

\begin{ass}[Data transmitting pattern]
\label{assumption:data-distribution}
    For each $t \in \N_+$ and $\mu \in \cM$, we denote by $(y_{t, \mu}, X_{t, \mu})$ a dataset produced by model $\mu$ in the $t$-th training cycle, generated according to the model fitted in the previous round: 
    \begin{align*}
        y_{t, \mu} = X_{t, \mu} \hat\beta_{t - 1, \mu} + \varepsilon_{t, \mu} \in \R^{n_{t, \mu}}. 
    \end{align*}
    We assume that for every $(\nu, \mu) \in \cE$, the dataset $(y_{t, \nu \to \mu}, X_{t, \nu \to \mu})$ sent from $\nu$ to $\mu$ is a subset of $(y_{t, \nu}, X_{t, \nu})$.
\end{ass}

\begin{rem}
    Note that for $\nu \in \cM$, the datasets $\{(y_{t, \nu \to \mu}, X_{t, \nu \to \mu}): (\nu, \mu) \in \cE \}$ may overlap with each other, and are thus not necessarily independent. 
\end{rem}

We then assume that for each $(\nu, \mu) \in \cE$, a nontrivial portion of the data used by model $\mu$ to update its parameters is generated by model $\nu$. 
Additionally, we assume that the number of training samples in each round does not grow too quickly.

\begin{ass}\label{ass:sample-size-ratio}
There exists $\alpha \in (0, 1)$, such that for any $t \in \N_+$ and $(\nu, \mu) \in \cE$, 
\begin{align*}
	p_{t, \nu \to \mu} = \frac{n_{t, \nu \to \mu}}{\sum_{m \in \Nin{\mu}} n_{t, m \to \mu} } \geq \alpha. 
\end{align*}
Recall $n_{t, \mu}$ is from Assumption \ref{assumption:data-distribution}. 
For $t \in \N_+$, we define $n_{t, \max} = \sup_{\mu \in \cM} n_{t, \mu}$. 
We also define $n_{0, \max} = \sup_{\mu \in \cM \backslash \cM_{\rm nature}} n_{0, \mu}$. 
We assume that $\sup_{s \in \{0\} \cup [t]} n_{s, \max} / t  \to 0$ as $t \to \infty$. 
\end{ass}

By definition, for all $t \in \N_+$ and $\Nin{\mu} \neq \emptyset$, it holds that $\sum_{\nu \in \Nin{\mu}} p_{t, \nu \to \mu} = 1$. 
Recall that $P_t$ is defined in \cref{eq:markov_transition_def}. 
For $t \ge s \ge 0$, we define $J_{t, s + 1} = P_t P_{t - 1} \cdots P_{s + 1}$, with the convention that $J_{t,t+1} = I_K$. 
One can verify that each row of $J_{t, s+1}$ sums to 1.  
By definition, we know that a model $\nu_1$ in $\cM_l^{\rm nc}$ receives information from at least one model in $\mathcal{M}_u$.
In other words, there exists a directed path from some $\nu_2 \in \mathcal{M}_u$ to $\nu_1$. Let $\mathcal{S}(\nu_1;0)$ denote a model in $\mathcal{M}_u$ that achieves the shortest path to $\nu_1$, and let $\ell(\nu_1)$ denote the length of this path. 
We further denote this path by $$\mathcal{S}(\nu_1; 0) \to \mathcal{S}(\nu_1; 1) \to \cdots \to \mathcal{S}(\nu_1; \ell(\nu_1)) = \nu_1. $$   
Let $T_0 = \sup_{\nu_1 \in \cM_l^{\mathrm{nc}}} \ell(\nu_1)$.
Leveraging Assumption \ref{ass:sample-size-ratio}, we know that for all $\nu_1 \in \cM_l^{\rm nc}$, $t \in \N_+$ and $t_0 \geq T_0$,
\begin{align*}
    \sum_{\nu_2 \in \cM_u} J_{t + t_0, t + 1, \nu_1, \nu_2} \geq \prod_{i = 1}^{\ell(\nu_1)} p_{t + i, \mathcal{S}(\nu_1; i - 1) \to \mathcal{S}(\nu_1; i)} \geq \alpha^{T_0}. 
\end{align*}
The above lower bound further implies the existence of a constant $\omega \in (0, 1)$, such that for any $t_0 \geq T_0$, 
\begin{align}
\label{eq:94}
    \sup_{t \in \N_+, \, \nu_1 \in \cM_l^{\rm nc}} \sum_{\nu_2 \in \cM_l} J_{t + t_0, t + 1, \nu_1, \nu_2} = 1 - \inf_{t \in \N_+, \, \nu_1 \in \cM_l^{\rm nc}} \sum_{\nu_2 \in \cM_u} J_{t + t_0, t + 1, \nu_1, \nu_2}  \leq 1 - \alpha^{T_0} \leq \omega \in (0, 1). 
\end{align}
%
We next state the assumptions on the noise vectors. Specifically, we assume that they are independent of the design matrices and that their variances are bounded both above and below.

\begin{ass}
\label{assumption:noise-vectors}
We impose the following assumptions on the noise vectors: 
\begin{enumerate}
    \item Recall that $E_t$, $\cX_t$ are defined in Assumption \ref{ass:independent_covariate}. 
    We assume that for all $t \in \N$,  $E_t$ is a mean-zero random vector that is independent of $\cX_t$.
    \item For $t \in \N_+$, we let $\eps_t = (\eps_{t, \mu_1}^{\top}, \eps_{t, \mu_2}^{\top}, \cdots, \eps_{t, \mu_K}^{\top})^{\top}$. 
    For $t = 0$, we let $\eps_0 = (\eps_{0, \mu}^{\top})_{\mu \in \cM \backslash \cM_{\rm nature}}^{\top}$.
    We assume that there exists $\rho_1 > 0$, such that for all $t \in \N$, $\mathrm{Cov}[\eps_t] \preceq \rho_1 I$. 
    \item  
    We assume that there exists $\rho_2 > 0$, such that for all $t \in \N$, $\mathrm{Cov}[\eps_t] \succeq \rho_2 I$. 
    %
\end{enumerate}

\end{ass}

We next state assumptions on the design matrices. 
In particular, we assume that $\cT_{t, \mu, \nu}$ defined in \cref{eq:cT-def} concentrates around its population counterpart $\cT^{\ast}_{t, \mu, \nu}$ defined in \cref{eq:cT-star-def}. Recall that $X_{t, \mu} \in \mathbb{R}^{n_{t, \mu} \times d}$ is defined in Assumption \ref{assumption:data-distribution}. 
We further assume that, the operator norm of $X_{t, \mu}^\top X_{t, \mu} / n_{t, \mu}$ is bounded with high probability.

\begin{ass}\label{ass:concentration}
There exist $\delta, \kappa > 0$ and $\gamma_1, \gamma_2  \in (0, 1)$, such that for any $t \in \N_+$, $\mu \in \cM$, and $(\nu, \mu) \in \cE$, 
\begin{align*}
        & \P \left( \norm{\cT_{t, \mu, \nu} - \cT_{t, \mu, \nu}^* }_{\op} \ge \delta \right) \le \, \gamma_1, \\
        & \P\left( \norm{X_{t, \mu}^\top X_{t, \mu} / n_{t, \mu} }_{\op}  \geq \kappa\right) \leq \gamma_2. 
\end{align*}
When $t = 0$, for any $\mu \in \cM \backslash \cM_{\rm nature}$, we assume
\begin{align*}
    \P\left( \norm{X_{0, \mu}^\top X_{0, \mu} / n_{0, \mu} }_{\op}  \geq \kappa\right) \leq \gamma_2. 
\end{align*}
In addition, we assume $\delta$, $\gamma_1$ and $\gamma_2$ are small enough such that
\begin{align*}
	& \Big(1 + \frac{\delta + 2\gamma_1}{\alpha}\Big)^{K} - \alpha^{K} \in (0, c(\alpha, K) ), \\
	& c_1 = \Big(1 + \frac{\delta + 2\gamma_1}{\alpha}\Big) (1 - \alpha^K)^{1 / K}  \in (0, 1), \\
	& c_2 =  \alpha^K - \frac{K (\delta + 2 \gamma_1)}{\alpha(1 - \alpha^K)^{1 / K + 1} (1 - c_1)^2} > \frac{9 K^3 \gamma_2}{\alpha^{K+1}}, \\
    & \omega + K \Big( \gamma_1 + \sqrt{2\gamma_1} + \sqrt{\delta} N_{\max}^{2T_0 - 1} \Big) \sum_{j = 1}^{2T_0 - 1} N_{\max}^j \in (0, 1), 
\end{align*}
where $c (\alpha, K) = 1 - \alpha^K / 2$  depends only on $\alpha$ and $K$, and we recall that $N_{\max} = \sup_{\mu \in \cM} |\Nin{\mu}|$, $\alpha  \in (0, 1)$ is from \cref{ass:sample-size-ratio}, and $(\omega, T_0)$ are from \cref{eq:94}. 
\end{ass}
\begin{rem}
\label{rem:matrix-concentration}
    When the rows of the design matrices are i.i.d. samples from a common distribution and the sample size is sufficiently large, Assumption~\ref{ass:concentration} holds by standard matrix concentration inequalities, with the exact bound determined by the tail behavior of the underlying distribution.
\end{rem}

Finally, we assume that the design matrices have full column rank and that the inverse of their Gram matrices has finite expectation.

\begin{ass}
\label{assumption:full-column-rank}
We assume the following: 
\begin{enumerate}
    \item For all $t \in \N_+$, $(\nu_1, \nu_2) \in \cE$, and $\mu \in \cM \backslash \cM_{\rm nature}$, the matrices $X_{t, \nu_1 \to \nu_2}$ and $X_{0, \mu}$ have full column rank with probability one. 
    \item There exists a constant $C_0 > 0$, such that
\begin{equation*}
    \sup_{t \in \N_+, \, (\nu_1, \nu_2) \in \cE} \E\Big[ \Tr\big[(X_{t, \nu_1 \to \nu_2}^{\top} X_{t, \nu_1 \to \nu_2})^{-1}\big]\Big] \le \, C_0, \quad \sup_{\mu \in \cM \backslash \cM_{\rm nature}} \E\Big[ \Tr\big[(X_{0, \mu }^{\top} X_{0, \mu })^{-1}\big]\Big] \le \, C_0.
\end{equation*}
\end{enumerate}

\end{ass}

\begin{rem}
    By \cite{srivastava2013covariance}, if the rows of $X \in \R^{n \times p}$ are independent and identically distributed with a finite $(2 + w)$-th moment for some $w > 0$ and an invertible covariance matrix, and if $n \geq C' d$ for some constant $C' > 0$, then $\E[(X^\top X)^{-1}] $ exists and is finite.
\end{rem}

\subsection{Main results}

In this section, we present our main results for the linear regression setting.
For comparison, let $\hat\beta_{t, \mu}^{\ast}$ denote the OLS estimator computed using the same covariate matrix as $\hat\beta_{t, \mu}$, but with responses generated according to the ground-truth linear model \eqref{eq:linear-model}. 
Specifically, $\hat\beta_{t, \mu}^{\ast}$ is trained on the dataset obtained by stacking $(X_{t, \nu \to \mu}, y_{t, \nu \to \mu}^{\ast})$, where, unlike $y_{t, \nu \to \mu}$ (which is generated using $\hat\beta_{t-1, \nu}$), the responses $y_{t, \nu \to \mu}^{\ast}$ are generated from the true conditional model: 
\begin{align*}
    y_{t, \nu \to \mu}^{\ast} = X_{t, \nu \to \mu} \beta_{\ast} + \varepsilon_{t, \nu \to \mu}.
\end{align*}
In this case, 
\begin{align*}
	\hat\beta_{t, \mu}^{\ast} = \beta_{\ast} + \Big( \sum_{\nu \in \Nin{\mu}} X_{t, \nu \to \mu}^{\top} X_{t, \nu \to \mu}\Big)^{-1} \sum_{\nu \in \Nin{\mu}} X_{t, \nu \to \mu}^{\top} \varepsilon_{t, \nu \to \mu},   
\end{align*}
where we recall that $(X_{t, \nu \to \mu}, \varepsilon_{t, \nu \to \mu})$ are introduced in Section \ref{sec:interactive-linear}.

As is standard in linear regression, we use the mean squared error to quantify the discrepancy between the fitted and target distributions.
Specifically, for $t \in \N$ and $\mu \in \cM$, we define $r_{t, \mu} = \E[\|\hat\beta_{t, \mu} - \beta^{\ast}\|_2^2]$ and $r_{t, \mu}^{\ast} = \E[\|\hat\beta_{t, \mu}^{\ast} - \beta^{\ast}\|_2^2]$. 
Recall from \cref{eq:risk-ratio} that, we say a model $\mu$ collapses if and only if the limit superior of the ratio between these two quantities diverges as $t \to \infty$.
Our findings imply that, under the assumptions in Section \ref{sec:assumptions-linear}, whether a model collapses depends only on its relative position within the interaction graph $\cG$.
In particular, our first main theorem states that every model in $\cM_l^{\mathrm{c}}$ experiences model collapse.
\begin{thm}
\label{thm:linear-collapse}
	Under Assumptions \ref{ass:independent_covariate}\,--\,\ref{assumption:full-column-rank}, for every $\mu \in \cM_l^{\rm c}$, we have 
	\begin{align*}
		\lim_{t \to \infty} \frac{\E[\|\beta_{\ast} - \hat\beta_{t, \mu}\|_2^2]}{\E[\|\beta_{\ast} - \hat\beta_{t, \mu}^{\ast}\|_2^2]} = \infty. 
	\end{align*}
\end{thm}

\begin{proof}[Proof of Theorem \ref{thm:linear-collapse}]
    We prove Theorem \ref{thm:linear-collapse} in Appendix \ref{appendix:thm:linear-collapse}. 
\end{proof}

Our second main theorem states that all models in $\cM_l^{\rm nc}$ do not collapse, in the sense that the associated risk ratios stay bounded as $t \to \infty$.

\begin{thm}
\label{thm:linear-non-collapse}
    Under Assumptions \ref{ass:independent_covariate}\,--\,\ref{assumption:full-column-rank}, and additionally assuming that the model performances are comparable, in the sense that
    \begin{align*}
        \sup_{t \geq 1}\frac{\sup_{s \in \{0 \} \cup [t]} \sup_{\nu \in \cM} \,\E[\|\beta_{\ast} - \hat\beta_{s, \nu}^{\ast}\|_2^2]}{\E[\|\beta_{\ast} - \hat\beta_{t, \mu}^{\ast}\|_2^2]} < \infty \qquad \mbox{for all }\mu \in \cM_l^{\rm nc}, 
    \end{align*}
    then for every $\mu \in \cM_l^{\rm nc}$, we have 
    \begin{align*}
		\limsup_{t \to \infty} \frac{\E[\|\beta_{\ast} - \hat\beta_{t, \mu}\|_2^2]}{\E[\|\beta_{\ast} - \hat\beta_{t, \mu}^{\ast}\|_2^2]} < \infty. 
	\end{align*}
\end{thm}

\begin{proof}[Proof of Theorem \ref{thm:linear-non-collapse}]
    We prove Theorem \ref{thm:linear-non-collapse} in Appendix \ref{appendix:thm:linear-non-collapse}. 
\end{proof}

\begin{rem}
The additional assumption in Theorem \ref{thm:linear-non-collapse} requires that the MSEs of different models, when trained solely on naturally generated data, do not differ substantially. Otherwise, if one model is intrinsically much worse than another when trained on nature-produced data (for example, due to having significantly fewer training samples), and its outputs are used by another model for parameter updates, then we would expect the second model to collapse regardless of the communication pattern.     
\end{rem}

\section{Results for M-estimation}
\label{sec:M}
In this section, we study general M-estimators defined as minimizers of the empirical risk induced by a loss function $L$. 
Specifically, we consider generative models defined by a vector $\beta \in \mathbb{R}^d$ and a function $\varphi: \R^d \times \R \mapsto \R^p$. 
We denote by $P_{\beta}$ the distribution associated with $\beta \in \R^d$, 
and assume that $z \sim P_{\beta}$ admits the representation
\begin{align*}
    z \overset{d}{=} \varphi(\beta, \varepsilon),
\end{align*}
where $\varepsilon \sim \mu_{\varepsilon}$ for some distribution $\mu_{\varepsilon}$ on $\R$. 
Given $z_1, z_2, \cdots, z_n \sim_{i.i.d.} P_{\beta}$, we estimate $\beta$ via empirical risk minimization (ERM):
\begin{align}
\label{eq:ERM}
    \hat\beta =  \, \arg\min_{\beta \in \mathcal{B}} \frac{1}{n} \sum_{i=1}^{n} L (\beta, z_i),
\end{align}
where $\mathcal{B} \subseteq \R^d$. In the case of unconstrained optimization, we simply take $\mathcal{B} = \R^d$. 
Estimators defined by \cref{eq:ERM} are quite general and cover a broad class of model training procedures. 

%
%

\subsection{Interactive learning with general M-estimators}

We now describe the interactive learning procedure for M-estimation under the interaction graph $\cG= (\cM, \cE)$, where $\cM = \{\mu_1, \mu_2, \cdots, \mu_K\}$. 
As in earlier sections, $\beta_* \in \R^d$ denotes the ground-truth parameter, and $\hat\beta_{t, \mu} \in \R^d$ represents the estimate retained by model $\mu$ after the $t$-th training cycle.
At initialization, we set $\hbeta_{0, \mu} = \beta_*$ for $\mu \in \cM_{\rm nature} \subseteq \cM_u$, and for the remaining models, $\hbeta_{0, \mu}$ is set to be the M-estimator obtained via solving the ERM problem \eqref{eq:ERM} over a dataset $Z_{0, \mu} \in \R^{n_{0, \mu} \times p}$. 
Here, the rows of $Z_{0, \mu}$ are i.i.d. generated from the ground-truth distribution $P_{\beta_{\ast}}$. 
To be specific, for $\mu \in \cM \backslash \cM_{\rm nature}$, 
\begin{align*}
    \hat\beta_{0, \mu} = \arg\min_{\beta \in \mathcal{B}} \frac{1}{n_{0, \mu}} \sum_{i = 1}^{n_{0, \mu}} L(\beta, z_{0, \mu, i}), 
\end{align*}
where $z_{0, \mu, i} \in \R^{p}$ denotes the $i$-th row of $Z_{0, \mu}$, generated according to $z_{0, \mu, i} = \varphi(\beta_{\ast}, \varepsilon_{0, \mu, i})$ for $\varepsilon_{0, \mu, i} \sim_{i.i.d.} \mu_{\varepsilon}$. 
In the $t$-th training cycle for $t \in \N_+$, 
models in $\cM_u$ remain unchanged and we set $\hat\beta_{t, \mu} = \hat\beta_{t - 1, \mu}$ for all $\mu \in \cM_u$. 
As for a model $\mu$ in $\cM_l$, we update its parameters based on data generated by models in $\Nin{\mu}$: 
\begin{align*}
    & \hat\beta_{t, \mu} = \arg\min_{\beta \in \mathcal{B}} \,  \frac{1}{\sum_{\nu \in \Nin{\mu}} n_{t, \nu \to \mu} } \sum_{\nu \in \Nin{\mu}} \sum_{i = 1}^{n_{t, \nu \to \mu}} L(\beta, z_{t, \nu \to \mu, i}), \\
    & z_{t, \nu \to \mu, i} \sim_{i.i.d.} P_{\hat\beta_{t - 1, \nu }} \qquad \mbox{for\, } i = 1, 2, \cdots, n_{t, \nu \to \mu},  
\end{align*}
where $z_{t, \nu \to \mu, i}$ is the $i$-th row of the matrix $Z_{t, \nu \to \mu} \in \R^{n_{t, \nu \to \mu} \times p}$, and $Z_{t, \nu \to \mu}$ represents the data passed from model $\nu$ to model $\mu$ during the $t$-th training cycle. 
We write $z_{t, \nu \to \mu, i} = \varphi(\hat\beta_{t - 1, \nu}, \varepsilon_{t, \nu \to \mu, i})$ for $\varepsilon_{t, \nu \to \mu, i} \sim \mu_{\varepsilon}$. 
Analogous to Assumption~\ref{assumption:data-distribution} in the linear setting, we assume that $Z_{t, \nu \to \mu}$ is a subset of $Z_{t, \nu} \in \R^{n_{t, \nu} \times p}$ indexed by $\mathcal{I}_{t, \nu \to \mu} \subseteq [n_{t, \nu}] $, i.e., the rows of $Z_{t, \nu \to \mu}$ consist of $\{ z_{t, \nu, i}: i \in \mathcal{I}_{t, \nu \to \mu} \}$.
Here, $z_{t, \nu, i} \in \R^p$ denotes the $i$-th row of $Z_{t, \nu}$.  
We further assume that the rows of $Z_{t, \nu}$ are sampled i.i.d. from $P_{\hat\beta_{t - 1, \nu}}$, and that the matrices $Z_{t, \nu}$ are mutually independent.
Note that $\{Z_{t, \nu \to \mu}: (\nu, \mu) \in \cE\}$ are not necessarily mutually independent and may overlap with each other.

    
%

\subsection{Assumptions for the M-estimation setting}
\label{sec:assumptions-M}

In this section, we summarize the assumptions needed to establish our main results for M-estimation, beginning with our requirement on data independence. 

\begin{ass}\label{ass:iid_data}
    For all $(t, \mu) \in \N_+ \times \cM$ or $(t, \mu) \in \{0\} \times (\cM \backslash \cM_{\rm nature})$, we denote by $z_{t, \mu, i} \in \R^p$ the $i$-th row of $Z_{t, \mu} \in \R^{n_{t, \mu} \times p}$. 
    For $t \in \N_+$, we assume that $(z_{t, \mu, i})_{i \in [n_{t, \mu}]}$ are i.i.d. sampled from $P_{\hat\beta_{t - 1, \mu}}$, and we write $z_{t, \mu, i} = \varphi(\hat\beta_{t - 1, \mu}, \varepsilon_{t, \mu, i})$. 
    For $\mu \in \cM \backslash \cM_{\rm nature}$, we assume that $(z_{0, \mu, i})_{i \in [n_{0, \mu}]}$ are i.i.d. sampled from $P_{\beta_{\ast}}$, with $z_{0, \mu, i} = \varphi(\beta_{\ast}, \varepsilon_{0, \mu, i})$. 
    We further assume that all relevant $\varepsilon_{t, \mu, i}$ are mutually independent and follow $\mu_{\varepsilon}$.
\end{ass}

We work in an asymptotic regime where the ratios of sample sizes in different training cycles converge to fixed constants as the sample sizes grow to infinity, as formalized in Assumption~\ref{ass:limit_proportion}.

\begin{ass}\label{ass:limit_proportion}
    For $t \in \N_+$, let $n_t = \sum_{\mu \in \cM} n_{t, \mu}$ denote the total number of samples used for training in the $t$-th training cycle, and let $n_0 = \sum_{\mu \in \cM \backslash \cM_{\rm nature}} n_{0, \mu}$ denote the total number of samples used for initialization.  
    We assume that the sample sizes tend to infinity simultaneously, and there exist two positive constants $\underline{\alpha}, \overline{\alpha} > 0$, such that the following holds:
    \begin{enumerate}
        \item For all $t \in \N_+$ and $(\nu, \mu) \in \cE$, we assume $n_{t, \nu \to \mu} / n_t \to \bar{p}_{t, \nu \to \mu} \geq \underline{\alpha}$ as the sample sizes tend to infinity. 
        For all $\mu \in \cM \backslash \cM_{\rm nature}$, we assume $n_{0, \mu} / n_0 \to \bar{p}_{0, \mu} \ge \underline{\alpha}$ as the sample sizes tend to infinity. 
        
        \item For all $t, s \in \N$, we assume that $n_t / n_{s} \to b_{t, s} \in [ \underline{\alpha}, \overline{\alpha} ]$ as the sample sizes tend to infinity.

        \item For all $t \in \N_+$ and $\nu_1, \nu_2 \in \cM$, we assume that the proportion of common data points used to train $\nu_1$ and $\nu_2$ in the $t$-th round converges as the sample sizes tend to infinity, i.e.,
        \begin{equation*}
         \frac{1}{n_t} \left\vert \left\{ (t, \nu, i) : \nu \in \Nin{\nu_1}, \, i \in \mathcal{I}_{t, \nu \to \nu_1} \right\} \cap \left\{ (t, \nu, i) : \nu \in \Nin{\nu_2}, \, i \in \mathcal{I}_{t, \nu \to \nu_2} \right\} \right\vert \to \, q_{t, \nu_1 \cap \nu_2} \in [0, 1].
        \end{equation*}
    \end{enumerate}

\end{ass}

Under Assumption \ref{ass:limit_proportion}, we define a matrix $\bar P_t \in \mathbb{R}^{K \times K}$ as follows: for $i, j \in [K]$, let $\bar P_{t,\mu_i,\mu_j}$ denote the $(i,j)$-th entry of $\bar P_t$, defined by:
\begin{equation}\label{eq:bar-P}
    \bar P_{t, \mu, \nu} = \left\{ \begin{array}{ll}
        \bar p_{t, \nu \to \mu} / (\sum_{m \in \Nin{\mu}} \bar p_{t, m \to \mu}) &  \mbox{if }  t \in \N_+,\, \mu \in \cM_l, \,\mbox{and }\nu \in \Nin{\mu}, \\
       1 & \mbox{if }t \in \N_+, \, \mu \in \cM_u, \,\mbox{and } \mu = \nu, \\
       1 & \mbox{if }t = 0 \mbox{ and }\mu = \nu,  \\
        0 & \mbox{otherwise. }
    \end{array} \right.
\end{equation}

Next, we assume that the model is well specified and  satisfies Fisher consistency, in the sense that the corresponding population risk is uniquely minimized at the true data-generating parameter. 
We further assume that the population Hessian at $\beta_*$ is positive definite.
These assumptions are imposed to ensure the consistency and asymptotic normality of the estimated parameters.

\begin{ass}\label{ass:well_specified}
We assume that the mapping $\beta \mapsto L(\beta, z)$ is twice differentiable for any $z \in \mathbb{R}^p$, and that the mapping $(\beta_1, \beta_2) \mapsto \E [L (\beta_1, \varphi(\beta_2, \veps))]$ is continuous.
For any $\beta \in \mathcal{B}$, we assume that
    \begin{equation*}
        \beta = \arg\min_{\beta' \in \mathcal{B}} \E \left[  L ( \beta', \, \varphi(\beta, \varepsilon)) \right]
    \end{equation*}
    is the unique minimizer. 
    In the above equation, the expectation is taken over $\varepsilon \sim \mu_{\varepsilon}$. Further, we assume that the population Hessian is strictly positive at $\beta_*$:
    \begin{equation*}
        \E \left[ \nabla_{\beta}^2 L (\beta, \varphi (\beta_*, \veps) ) \right] \Big\vert_{\beta  = \beta_*} \succ 0.
    \end{equation*}
    We also assume that for all $\beta_1, \beta_2 \in \mathcal{B}$:
    \begin{align*}
        \E[L(\beta_1, \varphi(\beta_2, \varepsilon))] < \infty, \qquad \E[\nabla_{\beta_1}L(\beta_1, \varphi(\beta_2, \varepsilon)) ] < \infty, \qquad \E[\nabla_{\beta_1}^2 L(\beta_1, \varphi(\beta_2, \varepsilon))] < \infty,
    \end{align*}
    and in addition, $\E[\|\nabla_{\beta}L(\beta, \varphi(\beta_\ast, \varepsilon)) |_{\beta = \beta_{\ast}}\|_2^2] < \infty$. 
\end{ass}
We additionally assume that either $L$ is a convex function, or $\mathcal{B}$ is a compact subset of $\R^d$. This condition guarantees that the ERM problem~\eqref{eq:ERM} has a well-defined solution.
\begin{ass}
\label{ass:L-convex}
We assume $\beta_* \in \operatorname{int} \mathcal{B}$. 
In addition, we assume at least one of the following conditions hold: (1) $L (\beta, z)$ is convex in $\beta$; (2) $\mathcal{B}$ is compact. 
\end{ass}

We next state our key assumptions on the loss function $L$. 
In particular, we require that $L$ and its Hessian belong to Glivenko-Cantelli function classes, while the Jacobian of $L$ forms a Donsker class. We begin by recalling the definitions of Glivenko-Cantelli and Donsker classes. For more details, see Chapter 2 of \cite{van1996weak}.

\begin{defn}[Glivenko–Cantelli class]
\label{def:GC-class}
We say that a function class $\mathcal{F}$ is $P$-Glivenko–Cantelli, if 
\begin{align*}
    \sup_{f \in \mathcal{F}} \Big| \frac{1}{n} \sum_{i=1}^{n} f(X_i) - \mathbb{E}[f(X_1)] \Big| \overset{a.s.}{\to} 0 
\end{align*}
as $n \to \infty$, for $X_i \sim_{i.i.d.} P$. 
\end{defn}

\begin{defn}[Donsker class]
    We say that a function class $\mathcal{F}$ is $P$-Donsker, if 
\begin{align*}
    & \mathbb{G}_n \overset{d}{\to}\mathbb{G} \qquad \mathrm{ in }\,\,\,\ell^{\infty}(\mathcal{F}),  \\
    & \mathbb{G}_n(f) = \frac{1}{\sqrt{n}} \sum_{i = 1}^n \big( f(X_i) - \E[f(X_1)] \big)\qquad \mathrm{ for }\,\,\, f \in \mathcal{F}. 
\end{align*}
Here, $X_i \sim_{i.i.d.} P$, and $\mathbb{G}$ is a mean-zero Gaussian process with covariance
\begin{align*}
    \mathrm{Cov}\big( \mathbb{G}(f), \, \mathbb{G}(g) \big) = \mathrm{Cov}\big( f(X), \, g(X) \big), \qquad X \sim P, \,\,\,\, \forall f, g \in  \mathcal{F},
\end{align*}
known as the $P$-Brownian Bridge.
\end{defn}

We now state our assumptions on the loss function $L$:  

\begin{ass}\label{ass:glivenko_cantelli}
    For any compact set $\Omega \subseteq \R^d$, we assume the following: 
\begin{enumerate}
    \item The function classes 
        \begin{align*}
        & \left\{ \varepsilon \mapsto L \left( \beta_1,\, \varphi (\beta_2, \, \varepsilon )\right) \big\vert \, (\beta_1, \beta_2) \in \Omega^2 \right\}, \\
        & \left\{ \varepsilon \mapsto \nabla_{\beta_1}^2 L \left( \beta_1,\, \varphi (\beta_2, \, \varepsilon )\right) \big\vert \,(\beta_1, \beta_2) \in \Omega^2 \right\}
    \end{align*}
    are $\mu_{\varepsilon}$-Glivenko-Cantelli.
    \item For any fixed $\beta_1 \in \R^d$, the function class 
    \begin{equation*}
		\left\{ \varepsilon \mapsto \nabla_{\beta_1} L \left( \beta_1,\, \varphi (\beta_2, \, \varepsilon )\right) \big\vert \,\beta_2 \in \Omega \right\}
	\end{equation*}
    is $\mu_{\varepsilon}$-Donsker. 
\end{enumerate}

\begin{rem}
    Assumption~\ref{ass:glivenko_cantelli} ensures that tools from empirical process theory (such as the uniform law of large numbers and the uniform central limit theorem) can be applied to analyze the asymptotic behavior of M-estimators in this multi-round interactive learning setting. 
    Notably, these conditions are weaker than many of the regularity assumptions commonly imposed in the literature, and hold for a broad class of models, including several canonical generalized linear models (GLMs). See \cref{sec:example_assumption} for further details.
\end{rem}    
\end{ass}

Finally, we assume that the population risk has sufficient regularity to allow for certain interchange of differentiation and integration. 

\begin{ass}\label{ass:population_regularity}
    The mapping $(\beta_1, \beta_2) \mapsto \E_{\varepsilon \sim \mu_{\varepsilon}} [ \nabla_{\beta_1} L ( \beta_1,\, \varphi (\beta_2, \, \varepsilon )) ]$ is continuously differentiable, and
    \begin{equation*}
    	\nabla_{\beta_1} \E_{\varepsilon \sim \mu_{\varepsilon}} [ \nabla_{\beta_1} L ( \beta_1,\, \varphi (\beta_2, \, \varepsilon )) ] = \, \E_{\varepsilon \sim \mu_{\varepsilon}} [ \nabla_{\beta_1}^2 L ( \beta_1,\, \varphi (\beta_2, \, \varepsilon )) ].
    \end{equation*} 
\end{ass}

Later in Section~\ref{sec:example_assumption}, we provide examples demonstrating that the assumptions listed above are satisfied by many commonly used M-estimators. 

\subsection{Main results}

Following the notation in Section~\ref{sec:matrix-form}, we define $\hat{\beta}_t = (\hbeta_{t, \mu_1}^{\top}, \hbeta_{t, \mu_2}^{\top}, \cdots, \hbeta_{t, \mu_K}^{\top})^{\top} \in \R^{dK}$ as the collection of learned parameters after the $t$-th training cycle, with the convention that $\hat\beta_{-1} = (\beta_{\ast}^{\top}, \beta_{\ast}^{\top}, \cdots, \beta_{\ast}^{\top})^{\top} \in \mathbb{R}^{dK}$. Let $\hbeta_{t, \mu}^*$ denote the M-estimator trained using the same number of samples as $\hbeta_{t, \mu}$, except that the samples are drawn directly from $P_{\beta_*}$. Using standard arguments in asymptotic statistics (cf. Chapter 3.2 of \cite{van1996weak}), we know that $\hbeta_{t, \mu}^*$ is consistent and asymptotically normal as the sample sizes tend to infinity (see Appendix \ref{sec:consistency-normality} for a proof). 
Further, the asymptotic covariance of $\hbeta_{t, \mu}^*$ is given by
\begin{equation}\label{eq:defn_V_star}
    V_* = \, \E \left[ \nabla_{\beta}^2 L (\beta, \varphi (\beta_{\ast},  \varepsilon )) \right]^{-1} \E \left[ \nabla_{\beta} L (\beta, \varphi (\beta_{\ast},  \varepsilon )) \nabla_{\beta} L (\beta, \varphi (\beta_{\ast},  \varepsilon ))^\top \right] \E \left[ \nabla_{\beta}^2 L (\beta, \varphi (\beta_{\ast},  \varepsilon )) \right]^{-1} \Big|_{\beta = \beta_{\ast}}.
\end{equation}
%
Recall the limiting proportions $\bar{p}_{0, \mu}$, $\bar{p}_{t, \nu \to \mu}$, and $q_{t, \mu_1 \cap \mu_2}$ from \cref{ass:limit_proportion}, we define the covariance matrices $\Sigma_0 \in \R^{dK \times dK}$ and $V_t \in \R^{dK \times dK}$ for $t \in \mathbb{N}_+$ as follows: For $i, j \in [K]$, the $(i, j)$-th block of $\Sigma_0$ is a $d \times d$ matrix proportional to $V_*$:
\begin{equation}\label{eq:def_Sigma_0}
    \Sigma_{0, \mu_i, \mu_j} = \, \frac{\bone \{ \mu_i = \mu_j, \, \mu_i \not \in \cM_{\rm nature} \} }{\bar{p}_{0, \mu_i}} V_*,
\end{equation}
and the $(i, j)$-th block of $V_t$ is also a constant multiple of $V_*$:
\begin{equation}\label{eq:def_V_t}
    V_{t, \mu_i, \mu_j} = \, \frac{q_{t, \mu_i \cap \mu_j} \bone \{ \mu_i \in \cM_l, \, \mu_j \in \cM_l \} }{(\sum_{\nu \in \Nin{\mu_i}} \bar p_{t, \nu \to \mu_i})(\sum_{\nu \in \Nin{\mu_j}} \bar p_{t, \nu \to \mu_j})} V_*.
\end{equation}
With these definitions, our first main result establishes the large-sample properties of $\hat{\beta}_{t}$ as the sample sizes grow to infinity. 

%
\begin{thm}\label{thm:m_est_variance}
Under Assumptions \ref{ass:iid_data}-\ref{ass:population_regularity}, 
as the sample sizes grow to infinity, each $\hat{\beta}_{t, \mu}$ is consistent and asymptotically normal. 
In particular,
\begin{equation*}
    \sqrt{n_t} \left( \hat{\beta}_{t} - \hat\beta_{-1} \right) \stackrel{d}{\to} \normal \left( 0, \Sigma_{t} \right),
\end{equation*}
where the asymptotic covariance $\Sigma_{t} \in \R^{dK \times dK}$ is defined recursively as follows: 
\begin{equation}\label{eq:asym_var_recursive}
    \Sigma_{t} = \, b_{t, t-1} (\bar P_t \otimes I_d) \Sigma_{t-1} (\bar P_t \otimes I_d)^\top + V_t,
\end{equation}
where we recall that $\bar P_t$ is defined in \cref{eq:bar-P}, $b_{t, t-1}$ is from Assumption \ref{ass:limit_proportion}, and $\Sigma_0$ and $V_t$ are defined via \cref{eq:def_Sigma_0,eq:def_V_t}, respectively.
\end{thm}
\begin{proof}[Proof of \cref{thm:m_est_variance}]
We prove \cref{thm:m_est_variance} in \cref{sec:proof_m_variance}.
\end{proof}
Building on the recursive formula \eqref{eq:asym_var_recursive}, our second main result characterizes the limiting behavior of $\Sigma_t$ as $t \to \infty$. 
To be specific, we define $r_{t,\mu} = \Tr(\Sigma_{t,\mu,\mu})$ and $r_{t,\mu}^* = \Tr(V_{\ast})$ as proxies for the mean squared errors of $\hat{\beta}_{t,\mu}$ and $\hat{\beta}_{t,\mu}^*$, respectively, where we recall that $\hat{\beta}_{t,\mu}^*$ denotes the benchmark estimator trained exclusively on natural data using the same algorithm and the same number of training samples.
We then compute the risk ratio 
\begin{equation*}
    \frac{r_{t, \mu}}{r_{t, \mu}^*} = \, \frac{\Tr (\Sigma_{t, \mu, \mu})}{\Tr (V_*)}
\end{equation*}
to quantify the asymptotic relative risk of $\hat{\beta}_{t, \mu}$ compared to that of $\hat{\beta}_{t, \mu}^*$ in the large-sample limit. 
Similarly as before, we say that a model $\mu$ collapses if the limit superior of this ratio diverges as $t \to \infty$. 
Since the transition matrices $\overline{P}_t$ are completely determined by the network structure, we should expect the same qualitative behavior as in the linear regression setting: As $t \to \infty$, every model in $\cM_l^{\rm c}$ collapses, while every model in $\cM_l^{\rm nc}$ remains non-collapsing.
We establish this rigorously in Theorem \ref{thm:network_collapse} below. 
\begin{thm}\label{thm:network_collapse}
Under Assumptions \ref{ass:iid_data}-\ref{ass:population_regularity}, the following holds: 
    \begin{itemize}
        \item [(a)] For all $\mu \in \cM_l^{\rm c}$:
            \begin{equation*}
                \lim_{t \to \infty} \frac{r_{t, \mu}}{r_{t, \mu}^*} = \lim_{t \to \infty} \frac{\Tr (\Sigma_{t, \mu, \mu})}{\Tr (V_*)} = \, \infty.
            \end{equation*}
        \item [(b)] For all $\mu \in \cM_l^{\rm nc}$:
            \begin{equation*}
                \limsup_{t \to \infty} \frac{r_{t, \mu}}{r_{t, \mu}^*} = \limsup_{t \to \infty} \frac{\Tr (\Sigma_{t, \mu, \mu})}{\Tr (V_*)} < \, \infty.
            \end{equation*}
    \end{itemize}
\end{thm}
\begin{proof}[Proof of \cref{thm:network_collapse}]
    We prove \cref{thm:network_collapse} in \cref{sec:proof_m_collapse}.
\end{proof}


\subsection{Examples}\label{sec:example_assumption}

In this section, we present several examples of $(L, \varphi)$ pairs that satisfy Assumptions \ref{ass:well_specified}-\ref{ass:population_regularity}.  
In particular, we focus on the widely used generalized linear models (GLMs).
Specifically, for a GLM with true parameter $\beta_{\ast} \in \mathbb{R}^d$, the data $z = (x, y)$ satisfies the following condition:  
%
\begin{align}
\label{eq:glm_specification}
\E_{\beta_{\ast}} [y \vert x] = A' (\beta_{\ast}^\top x),
\end{align}
where 
$A$ is a strictly convex and twice continuously differentiable cumulant function, $x$ represents the vector of covariates, and $y$ denotes the response variable.
We present a few examples below. 
%
\begin{itemize}
    \item [(i)] Linear regression: $A (\xi) = \xi^2 / 2$, and
    \begin{equation*}
        y = \, \beta_*^\top x + \veps, \quad \veps \perp\!\!\!\perp x, \,\, \E [\veps] = 0, \,\, \E [\veps^2] < \infty.
    \end{equation*}
    \item [(ii)] Logistic regression: $A(\xi) = \log (1 + \exp (\xi) )$, and
    \begin{equation*}
        y \,\vert\, x \sim \, \operatorname{Bernoulli} \left( \frac{\exp(\beta_*^\top x)}{1 + \exp(\beta_*^\top x)} \right).
    \end{equation*}
    \item [(iii)] Poisson regression: $A(\xi) = \exp(\xi)$, and
    \begin{equation*}
        y \,\vert\, x \sim \, \operatorname{Poisson} \big( e^{\beta_*^\top x } \big).
    \end{equation*}
\end{itemize}
We consider the following standard loss function $L$ for GLMs: 
\begin{equation}\label{eq:glm_loss}
    L (\beta, z) = \, - y \beta^\top x + A (\beta^\top x).
\end{equation}
Below, we show that under mild moment conditions on the covariates, Assumptions~\ref{ass:well_specified}--\ref{ass:population_regularity} are satisfied for these specific GLMs.
In Proposition~\ref{prop:verify_glm_assumption} in the appendix, we present some general conditions under which a GLM satisfies Assumptions~\ref{ass:well_specified}--\ref{ass:population_regularity}.
\begin{thm}\label{thm:glm_example}
    The following is true:
    \begin{itemize}
        \item [(a)] For linear regression, Assumptions~\ref{ass:well_specified}--\ref{ass:population_regularity} are satisfied, provided that $\E[\norm{x}_2^4] < \infty$.
        \item [(b)] For logistic regression, Assumptions~\ref{ass:well_specified}-\ref{ass:population_regularity} are satisfied, provided that $\E [\| x \|_2^3] < \infty$.
        \item [(c)] For Poisson regression, Assumptions~\ref{ass:well_specified}-\ref{ass:population_regularity} are satisfied, provided that $\E [\exp (R \| x \|_2)] < \infty$ for any $R > 0$.
    \end{itemize}
\end{thm}
\begin{proof}[Proof of \cref{thm:glm_example}]
    We prove \cref{thm:glm_example} in \cref{sec:proof_glm_example}.
\end{proof} 

\section{Numerical experiments}\label{sec:experiments}
In this section, we present several numerical experiments to validate our theoretical findings, including synthetic experiments on linear regression, M-estimation, and learning with non-convex single-index models, as well as real data experiments with MNIST and CIFAR-10.

\subsection{Synthetic data experiment I}
\label{sec:experimentI}
\paragraph{Setting and algorithms. }
In our first experiment, we consider two interaction graphs presented in Figure \ref{fig:2-graphs}. 
In the left panel, $\cG = (\cM, \cE)$ contains 5 nodes, where $\cM = \{ \mu_1, \mu_2, \mu_3, \mu_4, \mu_5 \}$ and $\cE = \{ (\mu_1, \mu_2), (\mu_2, \mu_5), (\mu_3, \mu_4), (\mu_3, \mu_5), (\mu_4, \mu_3) \}$.  
In this interaction graph, we have $\cM_l = \{ \mu_2, \mu_3, \mu_4, \mu_5 \}$ and $\cM_u = \{ \mu_1 \}$. 
Moreover, $\cM_l^{\rm c} = \{ \mu_3, \mu_4, \mu_5 \}$ and $\cM_l^{\rm nc} = \{ \mu_2 \}$.
Models in $\cM_u$ are colored blue and those in $\cM_l$ are colored red.

In the right panel of Figure \ref{fig:2-graphs}, we consider a slightly more complicated interaction graph that contains 8 nodes, where $\cM = \{ \mu_i : 1 \le i \le 8 \}$ and $\cE = \{ (\mu_1, \mu_3), (\mu_2, \mu_4), (\mu_3, \mu_4), (\mu_4, \mu_8), (\mu_7, \mu_8) \} \cup \{ (\mu_i, \mu_j): 5 \le i, j \le 7, \, i \neq j \}$. 
By definition, we have $\cM_l = \{ \mu_i : 3 \le i \le 8 \}$ and $\cM_u = \{ \mu_1, \mu_2 \}$. 
In addition, $\cM_l^{\rm c} = \{ \mu_5, \mu_6, \mu_7, \mu_8 \}$ and $\cM_l^{\rm nc} = \{ \mu_3, \mu_4 \}$.

\begin{figure}[!ht]
    \centering
{{\includegraphics[width=5cm]{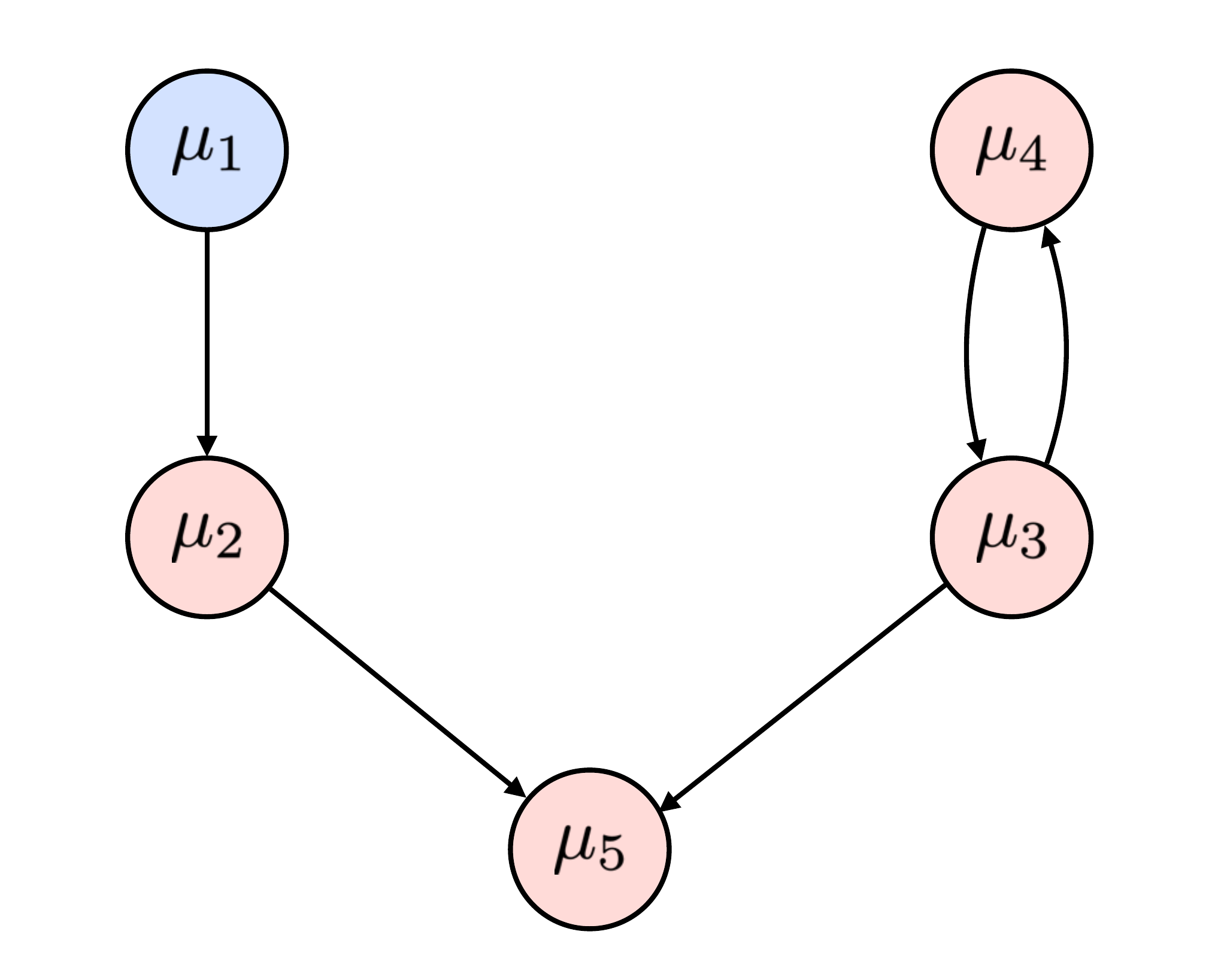} }}%
\qquad \qquad \qquad  
{{\includegraphics[width=8cm]{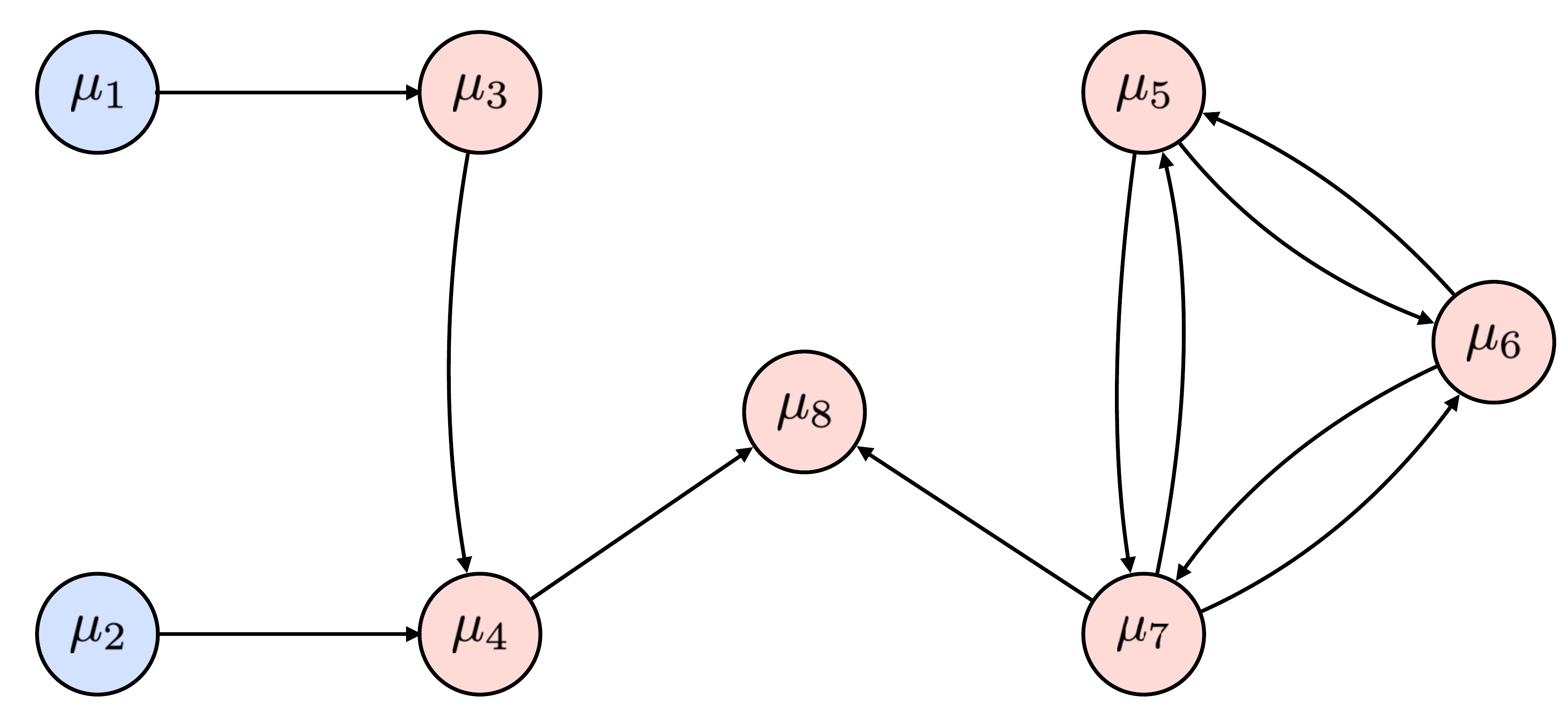} }}%
\vspace{0.5em}
    \caption{Interaction graphs considered in synthetic data experiment I. }%
    \label{fig:2-graphs}%
\end{figure}
For edges in both examples, we use a common sample size by setting all $n_{t, \nu \to \mu}$ and $n_{0, \mu}$ equal to a fixed value $n_{\rm sample} \in \mathbb{N}_+$.
Here, we set $n_{\rm sample} = 1000$, $d = 5$, and train the models in $\cM$ according to the interaction graphs specified in Figure \ref{fig:2-graphs} for $T = 50$ rounds. 
We further assume that samples associated with different edges are independent of each other. 
We generate the ground-truth parameter $\beta_{\ast}$ randomly at the beginning of each experiment, and hold it fixed throughout the subsequent training cycles. 
We conduct numerical simulations for three models: (a) linear regression as described in \cref{sec:linear-regression}; (b) logistic regression as described in \cref{sec:example_assumption}; and (c) a single-index model with a quadratic link function, where
\begin{equation*}
    z = \, (x, y), \quad y = \, (\beta_*^\top x)^2 + \veps, \,\, \mbox{and} \,\, L (\beta, z) = \, \frac{1}{2} \big( y - (\beta^\top x)^2 \big)^2.
\end{equation*}
Throughout, we assume that the noise in both the linear regression and single-index models follows a standard Gaussian distribution.
The key difference between model (c) and the other two models is that its loss function is non-convex, making the global minimizer of the empirical risk computationally challenging to obtain.
In our experiments, we use the Broyden–Fletcher–Goldfarb–Shanno (BFGS) algorithm to minimize the empirical risk and obtain the estimator $\hat{\beta}_{t, \mu}$ for model (c). 
For models (a) and (b), we fit them using standard GLM implementations available in common Python libraries.

\paragraph{Results of the experiments.}
For both interaction graphs in Figure~\ref{fig:2-graphs} and each model $\mu \in \cM_l$, we compute the risk ratio 
\begin{equation*}
    \frac{r_{t, \mu}}{r_{t, \mu}^*} = \, \frac{\E [\| \hbeta_{t, \mu} - \beta_* \|_2^2]}{\E [\| \hbeta_{t, \mu}^* - \beta_* \|_2^2]}. 
\end{equation*}
We plot the risk ratios as a function of $t$ for $t \in \{1,\cdots,50\}$. 
The mean squared errors of $\hat{\beta}_{t,\mu}$ and $\hat{\beta}^*_{t,\mu}$ are estimated via Monte Carlo simulation using $1000$ independent trials, with each trial based on $1000$ independent samples.
We present and discuss the results for the two interaction graphs in Figure~\ref{fig:2-graphs} separately below.
Experimental results for linear regression are presented in Figure \ref{fig:exp-linear}, those for logistic regression in Figure \ref{fig:exp-logistic}, and those for the single-index model in Figure \ref{fig:exp-single-index}.

\begin{itemize}
    \item[--] \textbf{Left panel of Figure \ref{fig:2-graphs}: }Our theoretical results in \cref{sec:linear-regression,sec:M} imply that models $\mu_3$, $\mu_4$ and $\mu_5$ will collapse as $t \to \infty$, while $\mu_2$ will not.
We confirm this prediction in the left panels of Figures~\ref{fig:exp-linear}-\ref{fig:exp-single-index} for settings (a)-(c), all of which exhibit the same qualitative behavior. 
From this figure, we see that the risk ratio for $\mu_2$ remains bounded because it belongs to $\mathcal{M}_l^{\rm nc}$ and receives information directly from $\mu_1 \in \cM_u$. 
In contrast, the risk ratios for $\mu_3$, $\mu_4$ and $\mu_5$ exhibit linear growth with respect to $t$, confirming that they will collapse as $t \to \infty$.
Notably, the risk ratios of $\mu_3$ and $\mu_4$ are nearly identical and they grow more rapidly than that of $\mu_5$. This is because that $\mu_3$ and $\mu_4$ are confined to the loop $\mu_3 \leftrightarrow \mu_4$, whereas $\mu_5$ benefits from information propagated along the directed path $\mu_1 \to \mu_2 \to \mu_5$.
    \item[--] \textbf{Right panel of Figure \ref{fig:2-graphs}:}
    We similarly plot the risk ratios for $\mu_3$ and $\mu_4$ (which learn only from stable data sources), $\mu_8$ (which learns from both stable and unstable data sources), and $\mu_5$, $\mu_6$, and $\mu_7$ (which are confined to an isolated synthetic loop hence learn only from unstable data sources), and observe the same qualitative behavior as in the first interaction graph.
    We additionally observe that, although neither $\mu_3$ nor $\mu_4$ collapses, the risk ratio of $\mu_4$ is slightly larger than that of $\mu_3$. 
    This is because that $\mu_3$ learns only from the natural data source $\mu_1$, whereas $\mu_4$ also distills information from $\mu_3$.
\end{itemize}

\begin{figure}[!ht]
\centering
\includegraphics[width=0.52\linewidth]{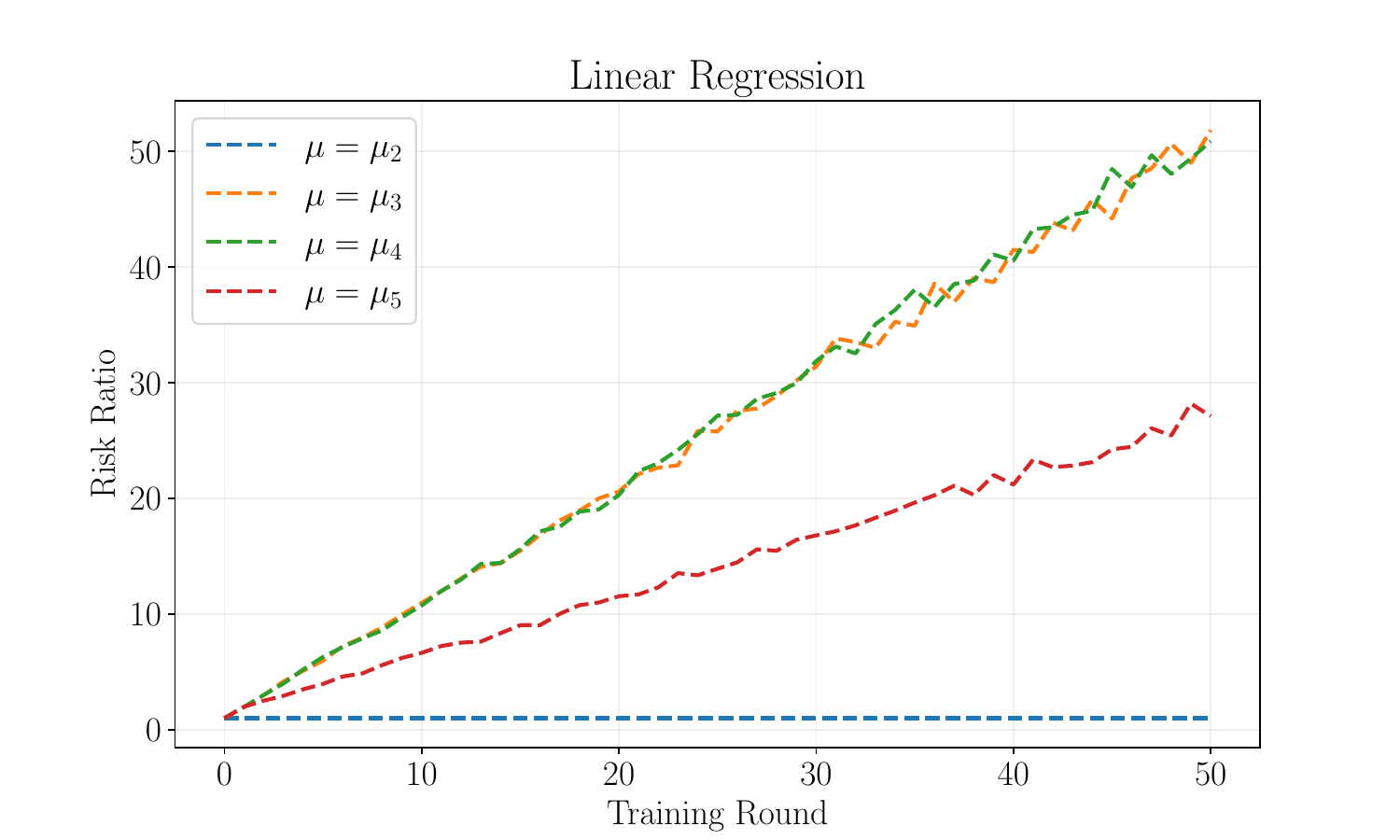} 
\hspace{-3em}
\includegraphics[width=0.52\linewidth]{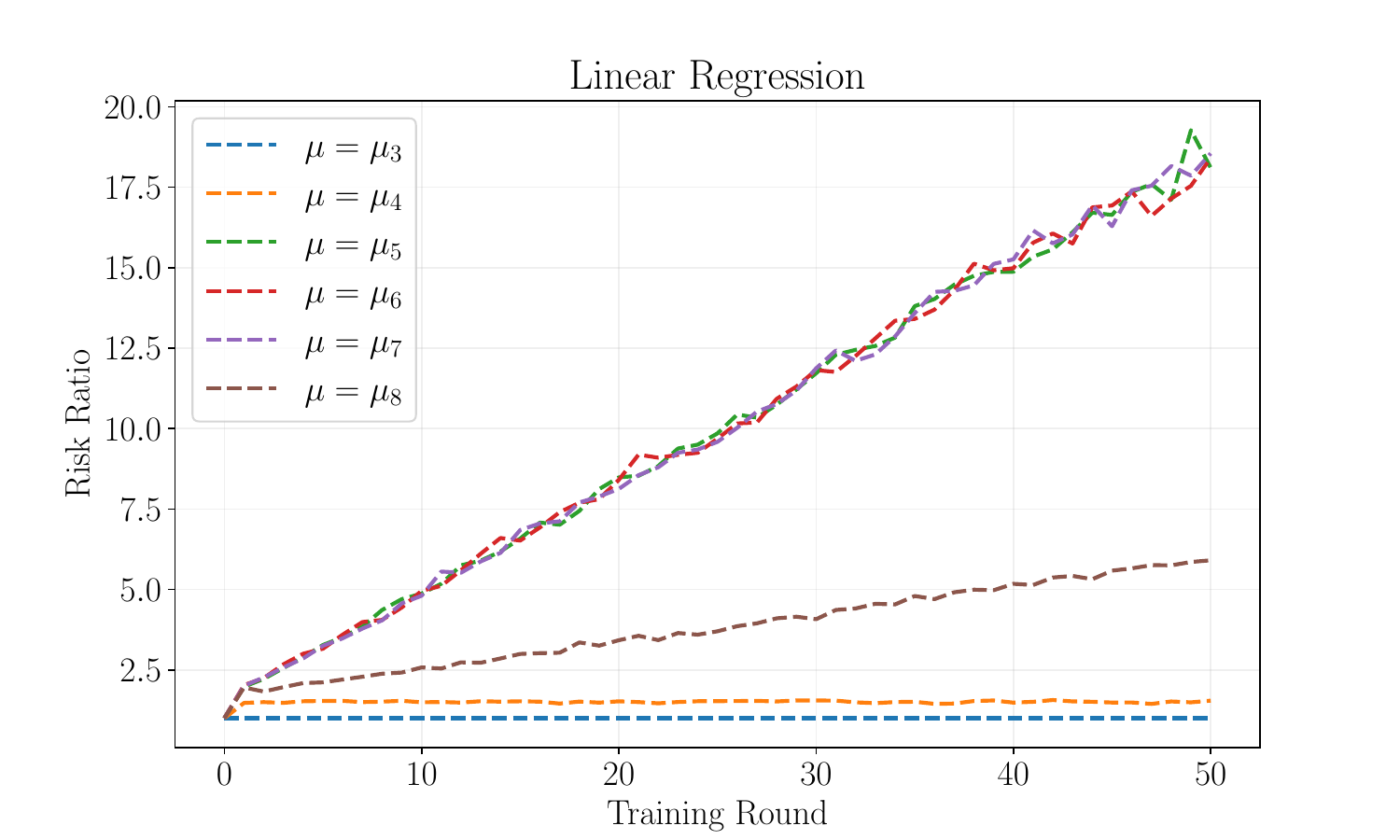}
    \caption{Plots of the risk ratios under the linear regression setting over the first 50 training cycles. The left panel shows results for the 5-node interaction graph in Figure~\ref{fig:2-graphs}, while the right panel shows results for the 8-node interaction graph in the same figure. }
    \label{fig:exp-linear}
\end{figure}

\begin{figure}[!ht]
\centering
\includegraphics[width=0.5\linewidth]{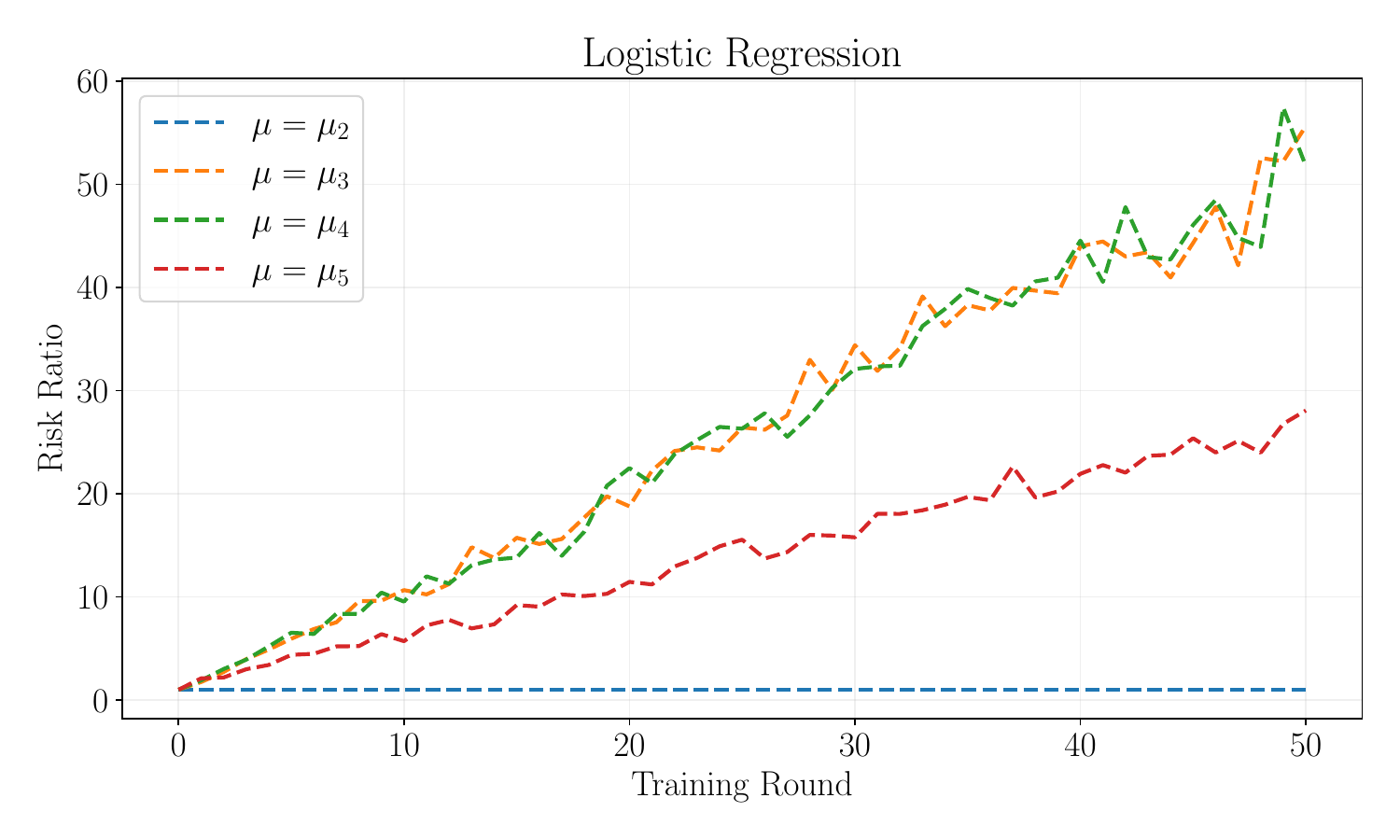} 
\hspace{-0.6em}
\includegraphics[width=0.5\linewidth]{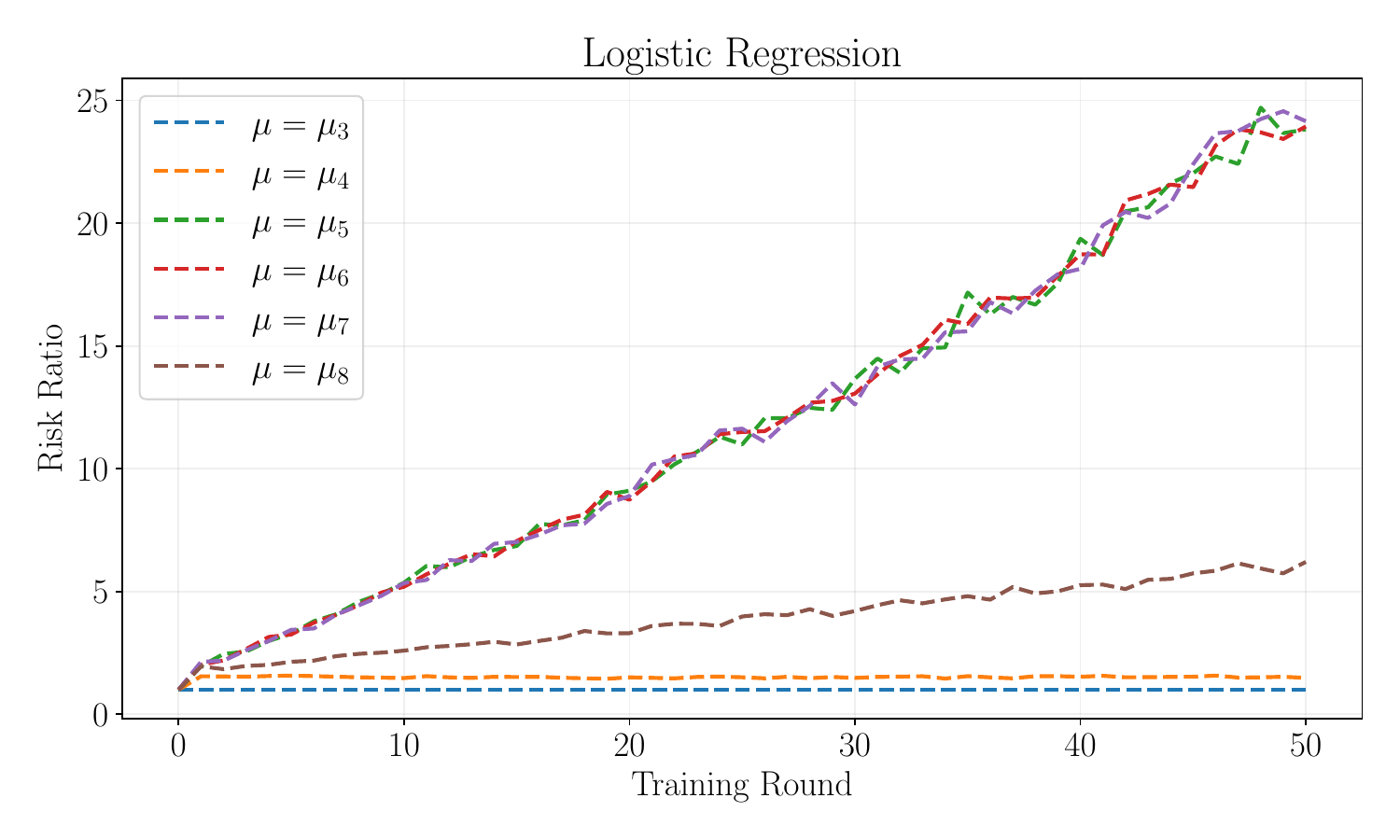}
    \caption{Plots of the risk ratios under the logistic regression setting over the first 50 training cycles. The left panel shows results for the 5-node interaction graph in Figure~\ref{fig:2-graphs}, while the right panel shows results for the 8-node interaction graph in the same figure. }
    \label{fig:exp-logistic}
\end{figure}

\begin{figure}[!ht]
\centering
\includegraphics[width=0.5\linewidth]{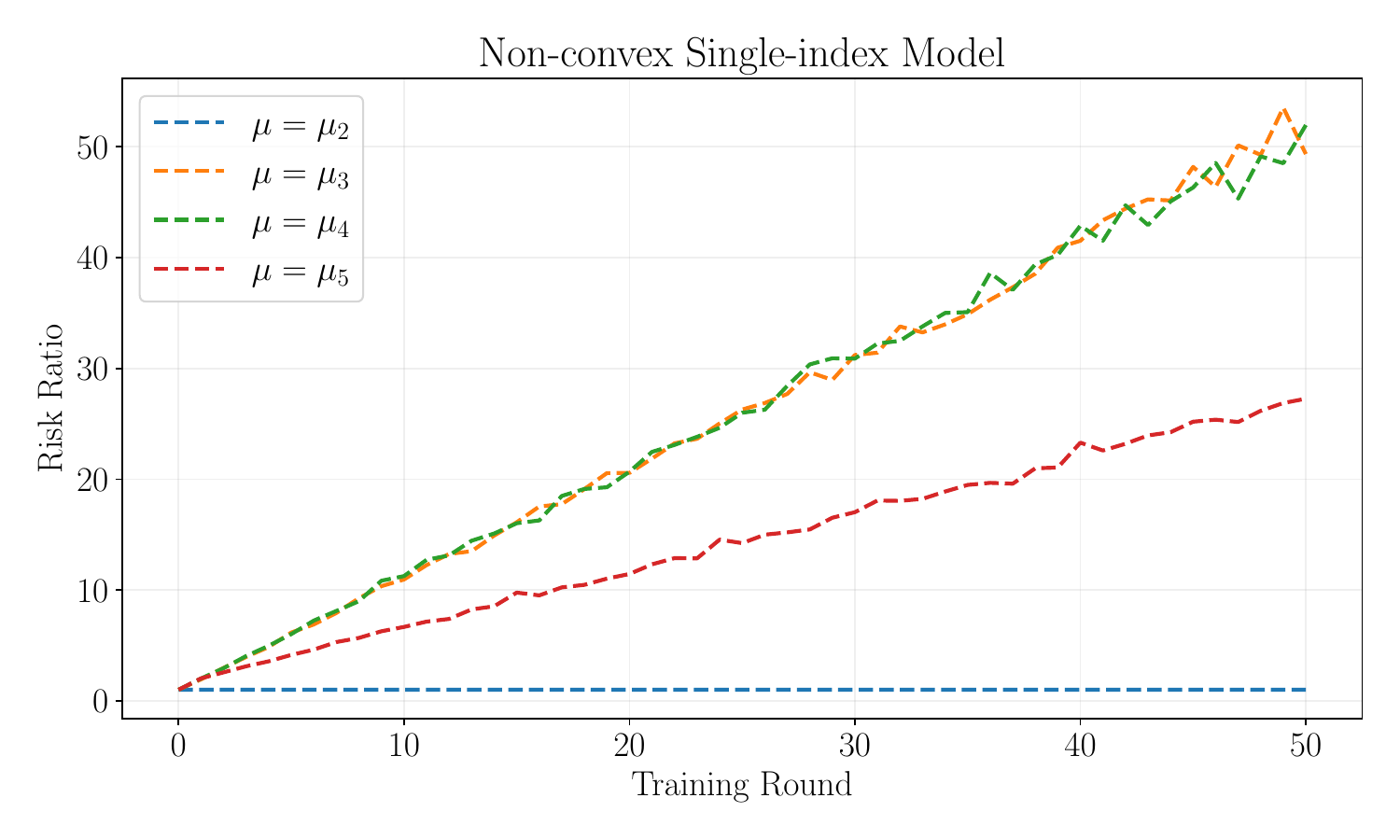} 
\hspace{-0.8em}
\includegraphics[width=0.5\linewidth]{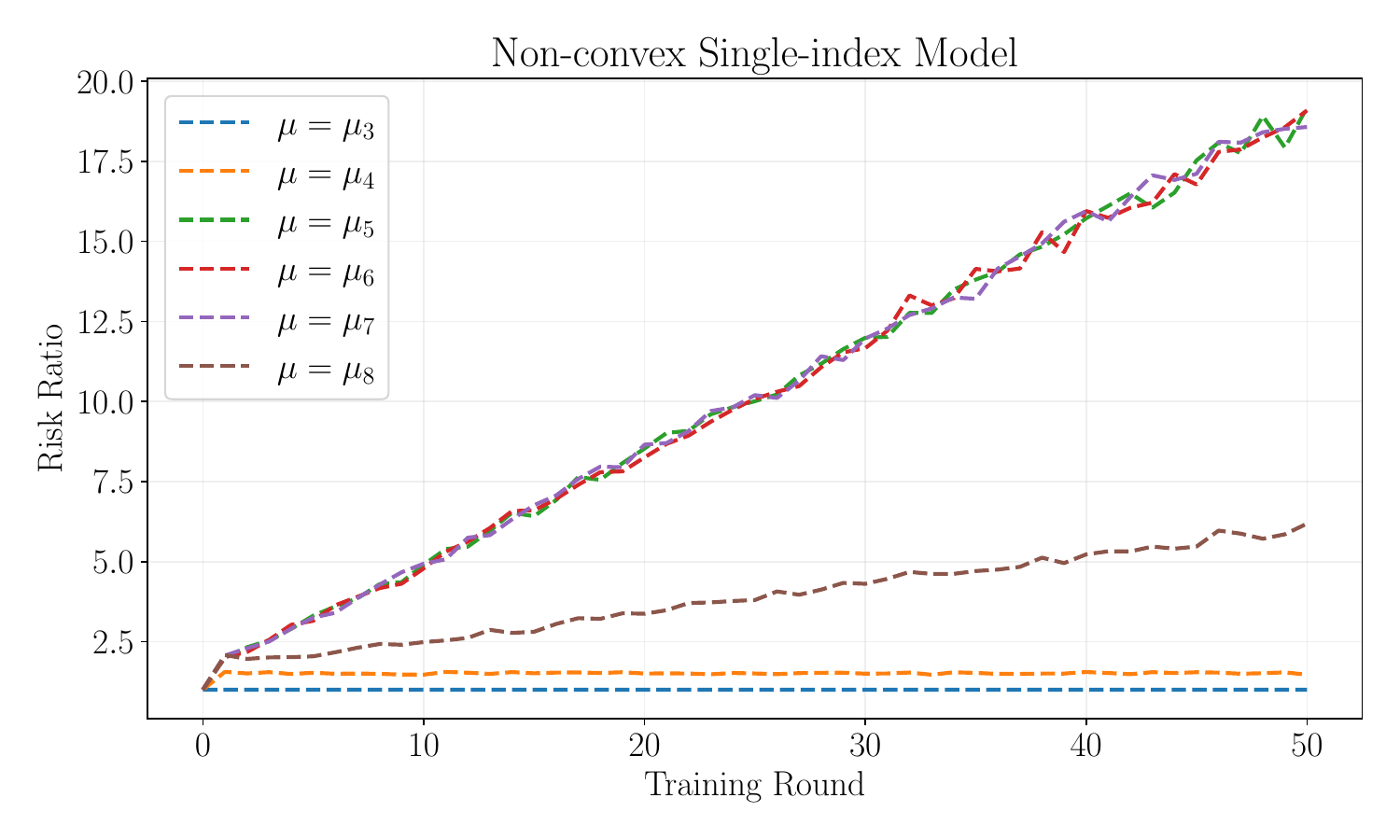}
    \caption{Plots of the risk ratios under the non-convex single index model setting over the first 50 training cycles. The left panel shows results for the 5-node interaction graph in Figure~\ref{fig:2-graphs}, while the right panel shows results for the 8-node interaction graph in the same figure. }
    \label{fig:exp-single-index}
\end{figure}

\subsection{Synthetic data experiment II}
In our second experiment, we compare the two interaction graphs presented in \cref{exm:one-diff}. 
Similar to the experiments in Section~\ref{sec:experimentI}, we plot the evolution of risk ratios for models in the two interaction graphs under (a) linear regression, (b) logistic regression, and (c) a single-index model with a quadratic link function.
The remaining settings are identical to those in Section~\ref{sec:experimentI}.

The numerical results are presented in Figure~\ref{fig:one_diff_compare}, from which we observe a sharp difference between the two graphs in terms of model stability: adding a single edge $\mu_2 \to \mu_1$ to the first graph in \cref{exm:one-diff} destroys its stable structure. 
Specifically, in the left panel of Figure~\ref{fig:one_diff_compare}, the models $\mu_i$ for $1 \le i \le 5$ do not collapse, with risk ratios remaining bounded by $5$. In contrast, in the right panel, all risk ratios grow approximately linearly with the number of training rounds, exceeding $10$ after only a few iterations, and with the largest surpassing $80$ after 50 rounds.
Once again, this confirms that model collapse can be highly sensitive to the underlying learning patterns. 

\begin{figure}[!ht]
     \centering
     \begin{subfigure}[b]{\textwidth}
         \centering
         \includegraphics[width=\textwidth]{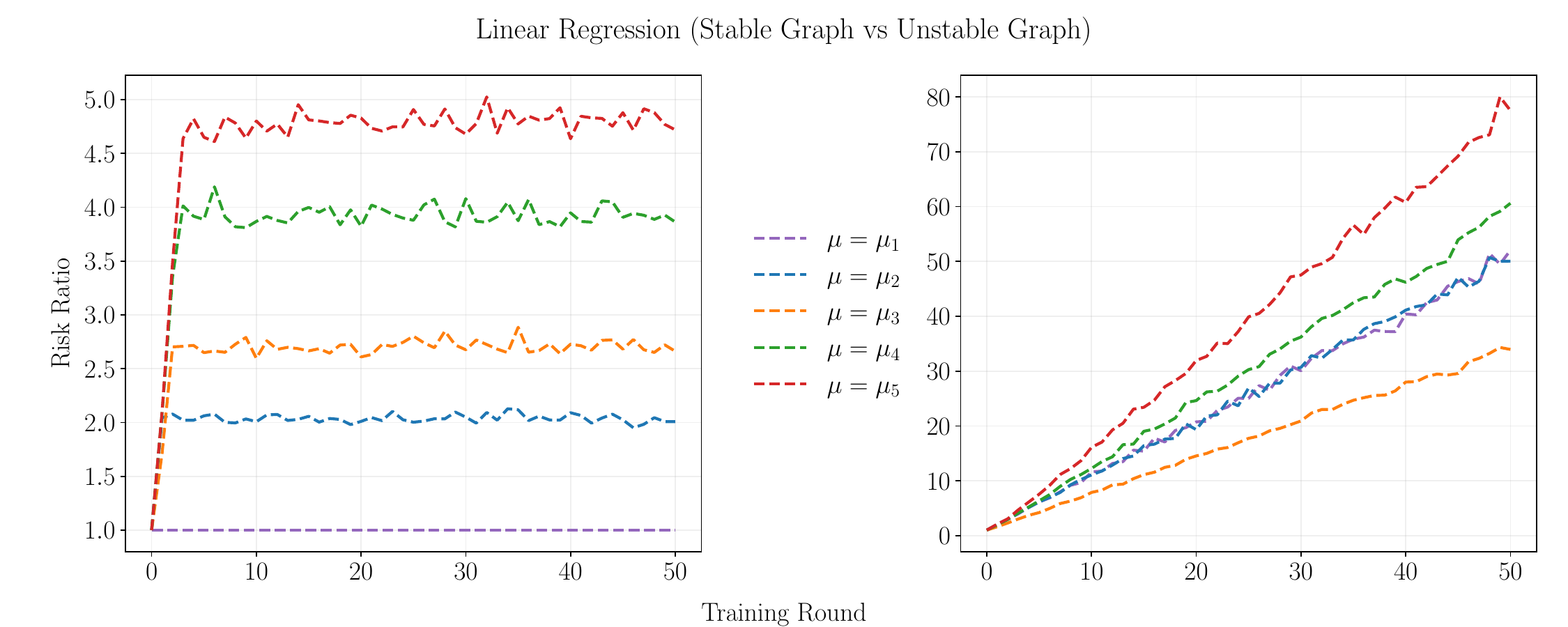}
     \end{subfigure}
     \hfill
     \begin{subfigure}[b]{\textwidth}
         \centering
         \includegraphics[width=\textwidth]{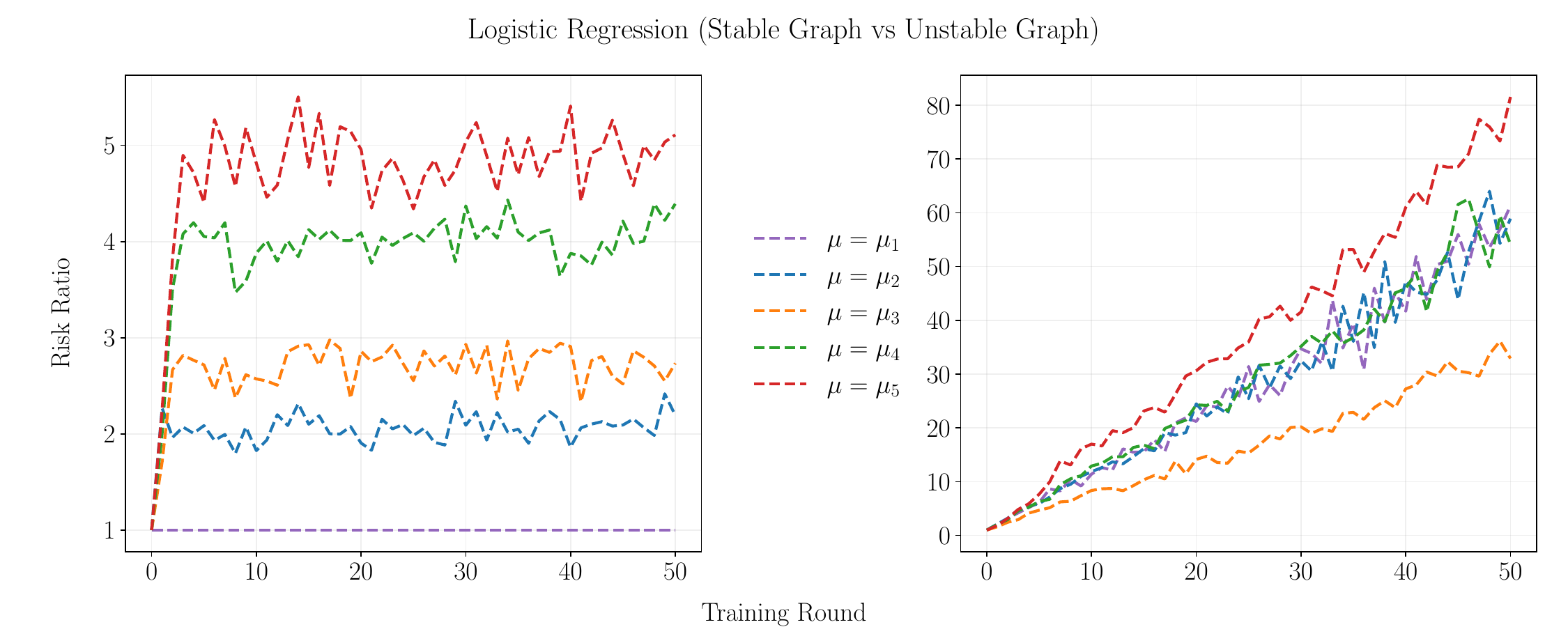}
     \end{subfigure}
     \hfill
     \begin{subfigure}[b]{\textwidth}
         \centering
         \includegraphics[width=\textwidth]{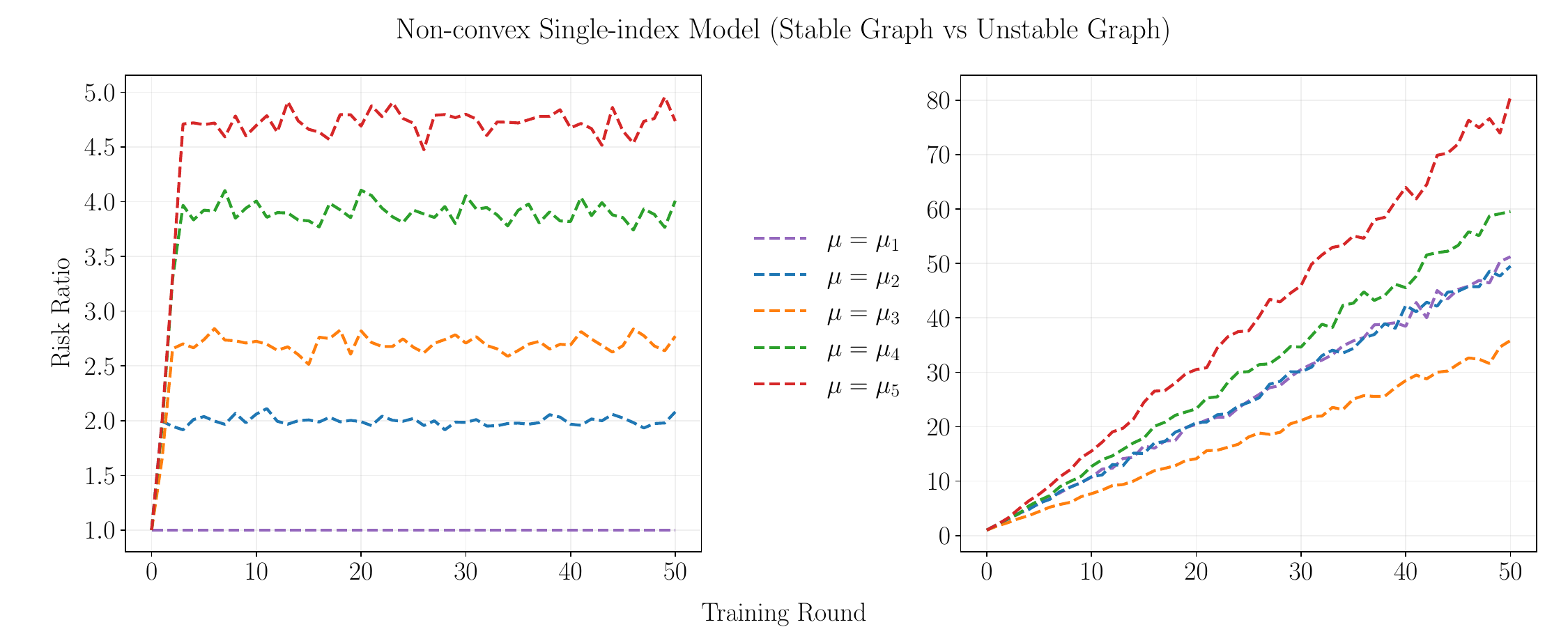}
     \end{subfigure}
        \caption{Comparison of risk ratios for models in the two interaction graphs of \cref{exm:one-diff}. Top: linear regression. Middle: logistic regression. Bottom: Single-index model with quadratic link function. Left: the stable graph represented by the left panel of Figure~\ref{fig:one-diff}. Right: the unstable graph represented by the right panel of Figure~\ref{fig:one-diff}.}
        \label{fig:one_diff_compare}
\end{figure}

\subsection{Real data experiments with MNIST and CIFAR-10}

In this section, we present numerical experiments on some real-world datasets to validate our theoretical findings on model collapse. We first conduct a real data experiment using the MNIST dataset \citep{lecun1998mnist}. We again consider the interaction graphs in Example~\ref{exm:one-diff}, and employ Generative Adversarial Networks (GANs) to train generative models for producing synthetic samples. 
As before, we assume that all $n_{t, \nu \to \mu}$ and $n_{0,\mu}$ are equal to a fixed value $n_{\rm sample} \in \mathbb{N}_+$, and set $n_{\rm sample} = 5000$. We further assume that samples associated with different edges are mutually independent.
To evaluate the quality of the generated samples, we compute the Fr\'echet Inception Distance (FID), a widely used metric for assessing the quality of generated images. Lower FID values indicate better image quality. 
Note that this setting does not fall within the framework proposed in Section~\ref{sec:M}, since the FID value is not the loss function used during training. 
In addition, to mimic the continuing-training procedure used in practice, we initialize each round of model training using the parameters obtained from the previous round. 
Although this example is not covered by our theoretical results, we qualitatively observe the same phenomena, suggesting that our theory may extend well beyond the current assumptions. 

We report the model FID ratios as a function of the number of rounds in Figure~\ref{fig:mnist_fid}.
Note that in the left panel of Figure~\ref{fig:mnist_fid}, the FID ratios initially decrease and then stabilize as the training rounds proceed, potentially due to the use of previous-round parameters for model initialization. 
In contrast, in the right panel of the same figure the FID ratios increase with the training round. 

\begin{figure}[!ht]
\centering
\includegraphics[width=0.5\linewidth]{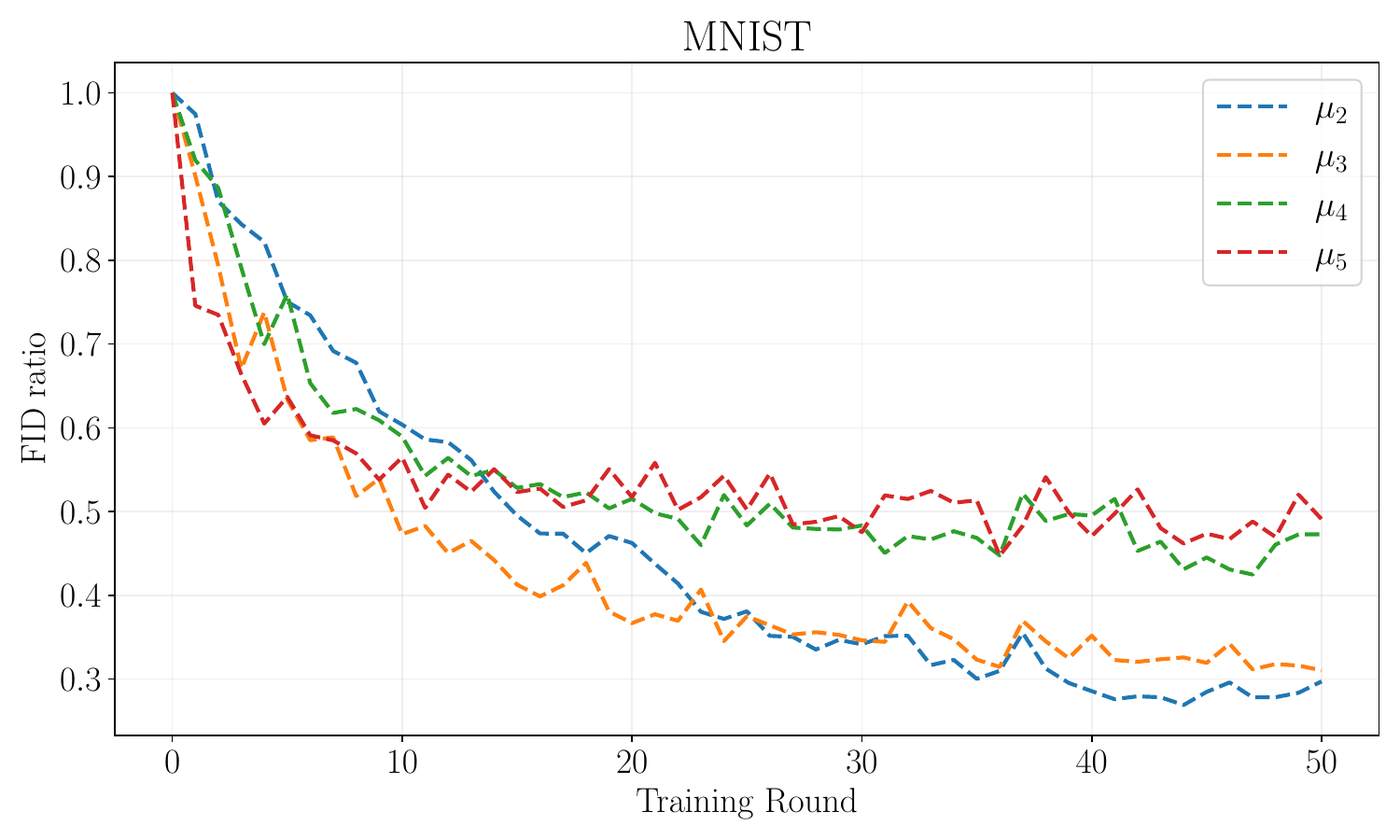} 
\hspace{-0.6em}
\includegraphics[width=0.5\linewidth]{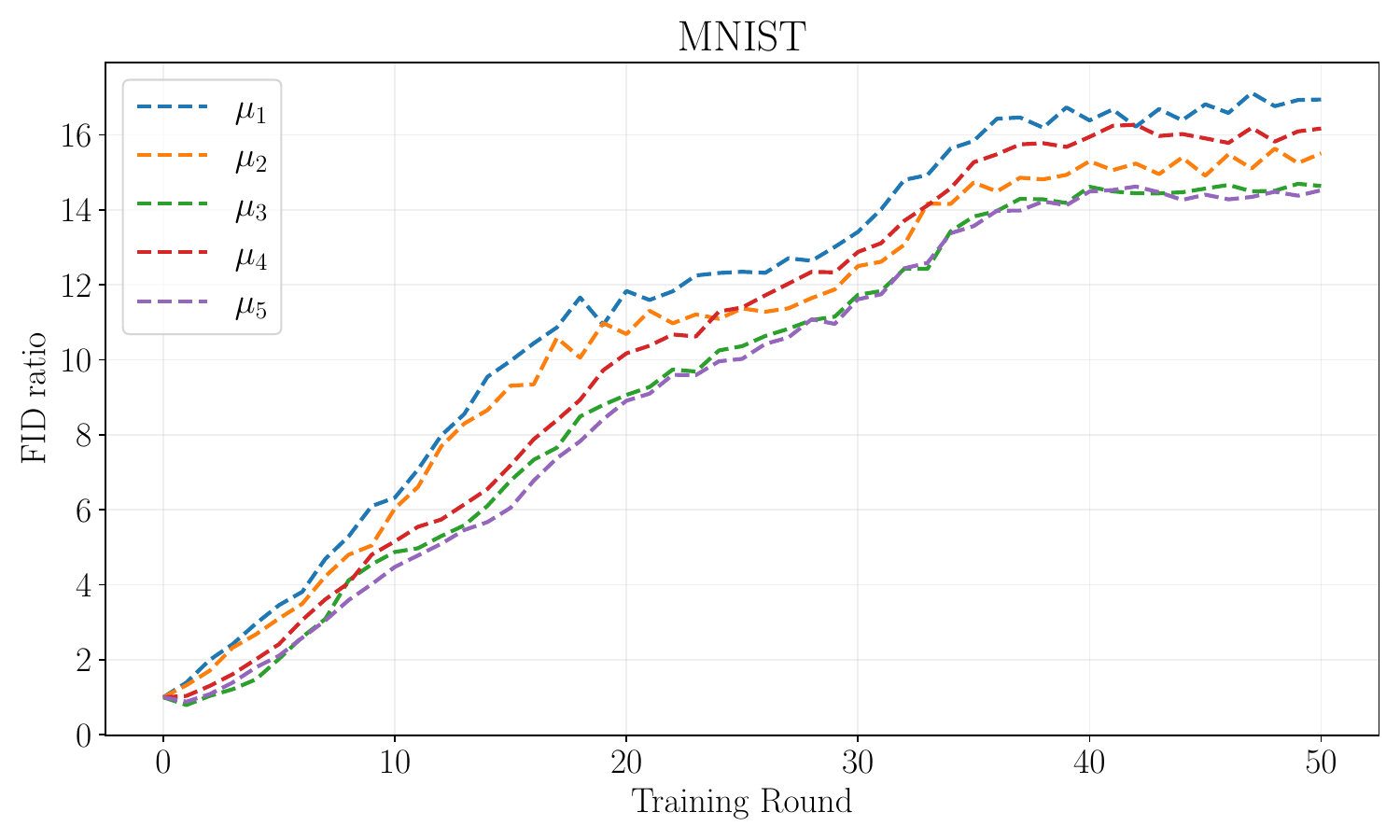}
    \caption{FID ratios achieved by GANs trained on MNIST over 50 rounds in the interactive learning setting. The left panel shows results for the first interaction graph in Example~\ref{exm:one-diff}, and the right panel corresponds to the second interaction graph.  }
    \label{fig:mnist_fid}
\end{figure}

We also conduct the same experiment on the CIFAR-10 dataset \citep{krizhevsky2009learning} and observe similar behavior. The corresponding results are presented in Figure~\ref{fig:cifar_fid}.
\begin{figure}[!ht]
\centering
\includegraphics[width=0.5\linewidth]{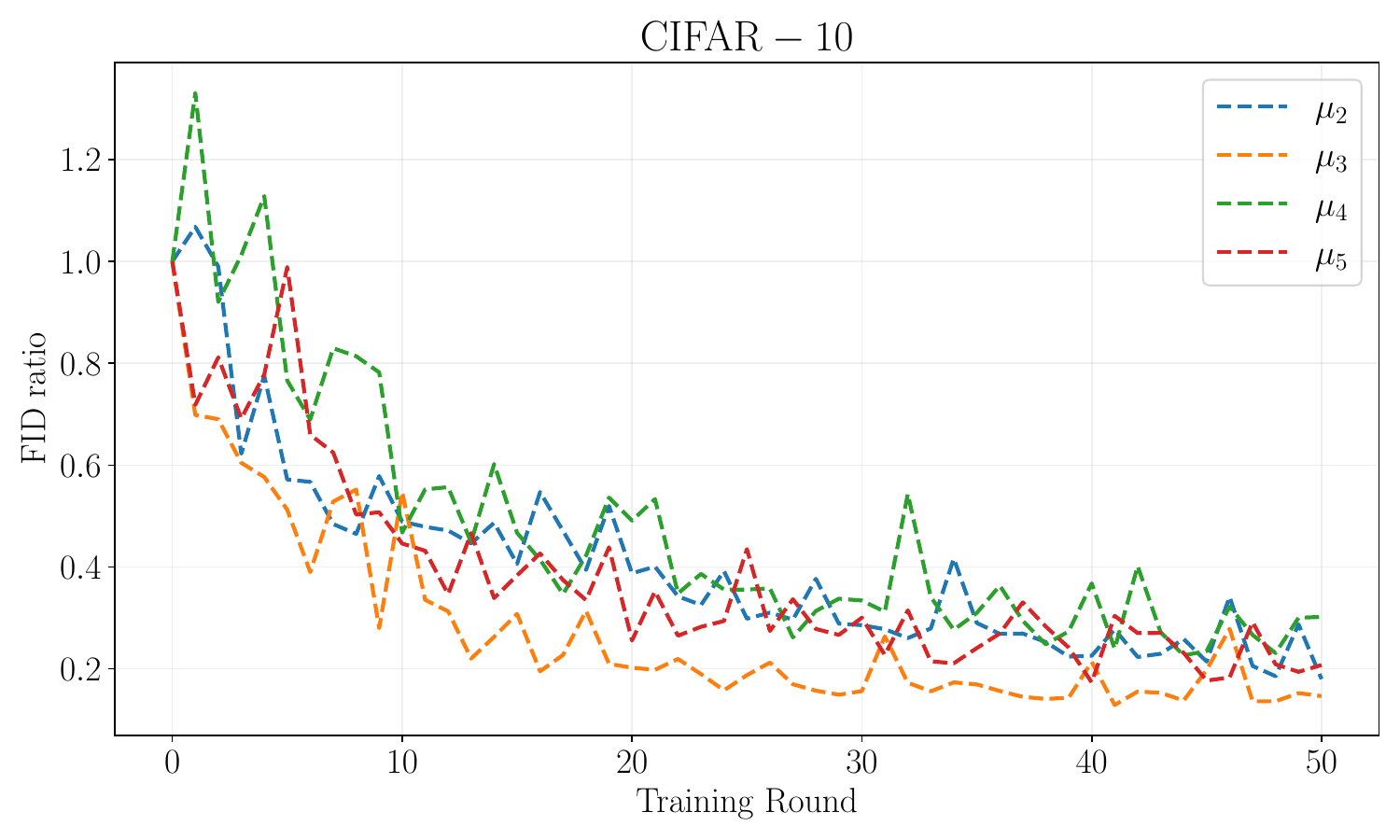} 
\hspace{-0.6em}
\includegraphics[width=0.5\linewidth]{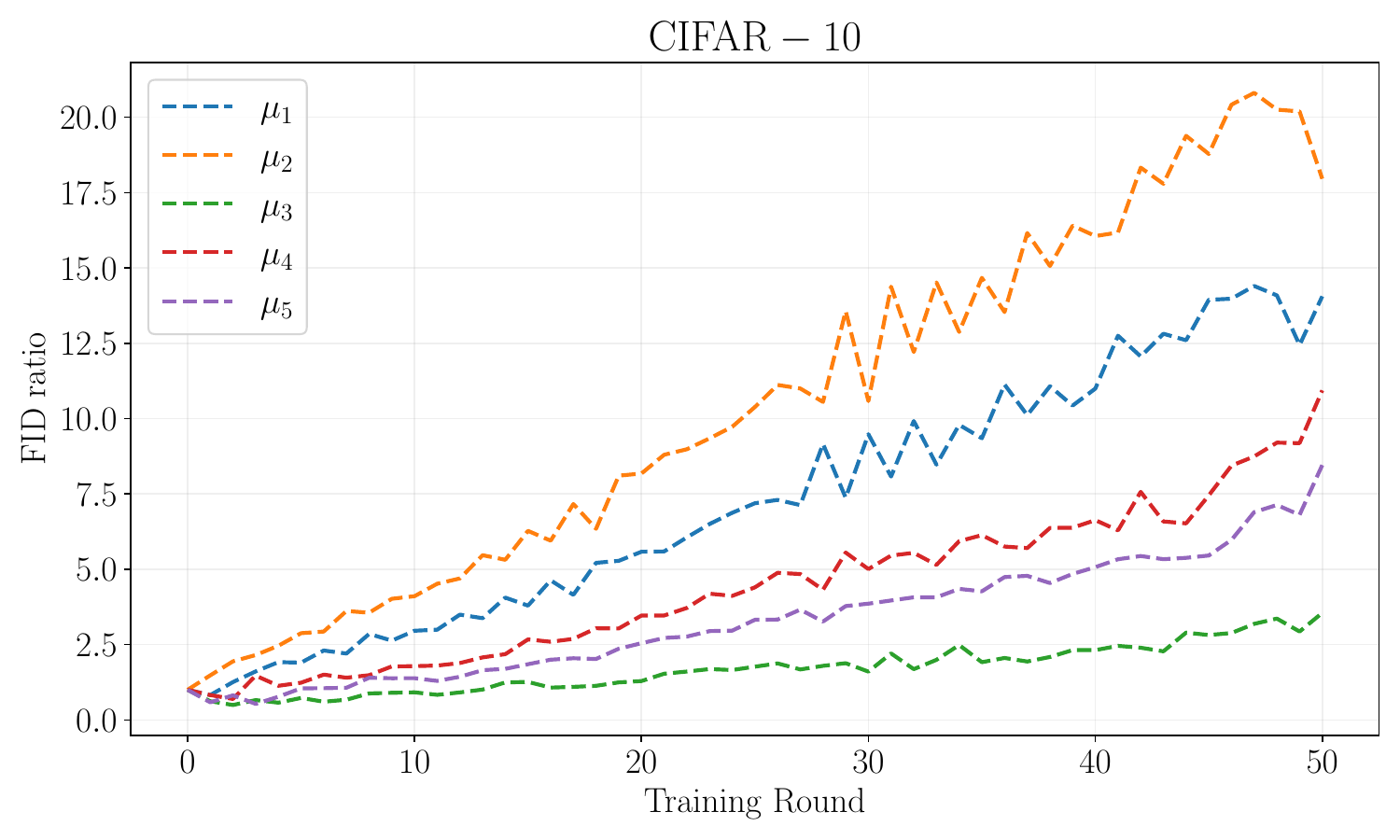}
    \caption{FID ratios achieved by GANs trained on CIFAR-10 over 50 rounds in the interactive learning setting. Similar to Figure \ref{fig:mnist_fid}, the left panel shows results for the first interaction graph in Example~\ref{exm:one-diff}, and the right panel corresponds to the second interaction graph.  }
    \label{fig:cifar_fid}
\end{figure}






\newpage

\bibliographystyle{plainnat}
\bibliography{model_collapse_ref}

\newpage 

\appendix

\section{Proofs of the main results}

We present in this section proofs of our main results. 

\subsection{Proof of Theorem \ref{thm:linear-collapse}}
\label{appendix:thm:linear-collapse}

Iterating \cref{eq:hbeta-update}, we see that 
\begin{equation}\label{eq:hat_beta_expression}
    \hbeta_t - \hat\beta_{0} = \, v_t + \cT_t v_{t-1} + \cdots + \cT_t \cdots \cT_3 v_2 + \cT_t \cdots \cT_2 v_1 = \sum_{s=1}^{t} \Omega_{t, s+1} v_s,
\end{equation}
where for $t \ge s$, we define $\Omega_{t, s} = \cT_t \cdots \cT_s \in \R^{dK \times dK}$, with the convention that $\Omega_{t, t + 1} = I_{dK}$. 
The matrix $\Omega_{t, s}$ can be viewed as a $K \times K$ block matrix, and we denote its $(i, j)$-th block by $\Omega_{t, s, \mu_i, \mu_j} \in \mathbb{R}^{d \times d}$. 

The next lemma is a direct consequence of Assumption \ref{assumption:data-distribution}. 

\begin{lem}
\label{lemma:data-transmission}
    Under Assumption \ref{assumption:data-distribution}, 
    for all $t \in \N_+$, there exist $N_t \in \N_+$, a set of vectors $\{\zeta_{t, 1}, \zeta_{t, 2 }, \cdots, \zeta_{t, N_t}\}$ and a set of matrices $\{Z_{t, 1}, Z_{t, 2}, \cdots, Z_{t, N_t}\}$, such that the following holds: 
\begin{enumerate}
    \item 
    For all $s \in [t]$ and  $(\nu, \mu) \in \cE$, 
    $\eps_{s, \nu \to \mu}$ and $X_{s, \nu \to \mu}$ admit the following decompositions:
\begin{align*}
    & \eps_{s, \nu \to \mu} = (\eps_{s, \nu \to \mu, 1}^{\top}, \eps_{s, \nu \to \mu, 2}^{\top}, \cdots, \eps_{s, \nu \to \mu, h_{s, \nu \to \mu}}^{\top})^{\top}, \\
    & X_{s, \nu \to \mu } = (X_{s, \nu \to \mu, 1}^{\top}, \, X_{s, \nu \to \mu, 2}^{\top}, \, \cdots, X_{s, \nu \to \mu, h_{s, \nu \to \mu}}^{\top})^{\top},
\end{align*}
where $h_{s, \nu \to \mu} \in \N_+$, 
and the number of rows of the matrix $X_{s, \nu \to \mu, h}$ is equal to the length of the vector $\eps_{s, \nu \to \mu, h}$ for all $h \in [h_{s, \nu \to \mu}]$. 
Additionally, 
there exists $i \in [N_t]$, such that $\eps_{s, \nu \to \mu, h} \equiv \zeta_{t, i}$ and $X_{s, \nu \to \mu, h} \equiv Z_{t, i}$. 
On the other hand, for any $t \in \N_+$ and $i \in [N_t]$, there exists $(s, \nu \to \mu, h)$ such that $\eps_{s, \nu \to \mu, h} \equiv \zeta_{t, i}$ and $X_{s, \nu \to \mu, h} \equiv Z_{t, i}$. 

\item There exists $C > 0$ that does not depend on $t$, such that $N_t \leq Ct$, and (recall $\rho_2$ is from Assumption \ref{assumption:noise-vectors})
\begin{align*}
    \cov[\zeta_t] = \cov\big[(\zeta_{t, 1}^{\top}, \, \zeta_{t, 2}^{\top}, \, \cdots, \zeta_{t, N_t}^{\top})^{\top} \big] \succeq \rho_2 I. 
\end{align*}

\end{enumerate}
\end{lem}

\begin{proof}[Proof of \cref{lemma:data-transmission}]
    We prove \cref{lemma:data-transmission} in Appendix \ref{proof:lemma:data-transmission}.
\end{proof}

For $\mu \in \cM_l$, $t \in \N_+$ and $i \in [N_t]$, we define
\begin{align}
    & \cP_{t,i} = \{(s,\, \nu_1 \to \nu_2, h): \eps_{s, \nu_1 \to \nu_2, h} \equiv \zeta_{t, i}\}, \label{eq:cPti}\\
    & A_{t, i, \mu} = \sum_{(s, m \to \nu, h) \in \cP_{t, i}} \Omega_{t, s + 1, \mu, \nu} \cT_{s, \nu, m} (X_{s, m \to \nu}^{\top} X_{s, m \to \nu })^{-1} Z_{t, i}^{\top} Z_{t, i}. \label{eq:A-def}
\end{align}
The set $\cP_{t, i}$ consists of indices for all pairs $(\eps_{s, \nu_1 \to \nu_2, h}, X_{s, \nu_1 \to \nu_2, h})$ such that $(\eps_{s, \nu_1 \to \nu_2, h}, X_{s, \nu_1 \to \nu_2, h}) \equiv (\zeta_{t, i}, Z_{t, i})$. 
By the first point of \cref{lemma:data-transmission}, for all $t \in \N_+$ and $i \in [N_t]$ we have $\cP_{t, i} \neq \emptyset$. 
The next lemma establishes additional properties of $\cP_{t, i}$. 

\begin{lem}
\label{lemma:Pti}
    Under Assumption \ref{assumption:data-distribution}, there does not exist $s_1 \neq s_2$ with $s_1, s_2 \in [t]$, $h_1, h_2 \in \N_+$ and $\nu_1, \nu_2, \mu_1, \mu_2 \in \cM$, such that $(s_1, \nu_1 \to \mu_1, h_1), (s_2, \nu_2 \to \mu_2, h_2) \in \cP_{t, i}$ for some $t \in \N_+$ and $i \in [N_t]$.
    Consequently, the index $s \in  [t]$ associated with each $\cP_{t, i}$ is unique, and we denote it by $s_{t, i}$.  
\end{lem}
\begin{proof}[Proof of \cref{lemma:Pti}]
By \cref{ass:independent_covariate} and \ref{assumption:noise-vectors}, for any $s_1 \neq s_2$, $h_1, h_2 \in \N_+$ and $\nu_1, \nu_2, \mu_1, \mu_2 \in \cM$, the noise vectors $\veps_{s_1, \nu_1 \to \mu_1, h_1}$ and $\veps_{s_2, \nu_2 \to \mu_2, h_2}$ are independent and both nondegenerate. Therefore, they cannot be equal to the same $\zeta_{t, i}$, completing the proof.
\end{proof}

With $A_{t, i, \mu}$, we find a more convenient expression for $\hat\beta_{t, \mu}$, as stated in the next lemma.

\begin{lem}
\label{lemma:hat_beta_expression}
Under Assumptions \ref{assumption:data-distribution} and \ref{assumption:full-column-rank}, for all $\mu \in \cM_l$ and $t \in \N_+$, it holds that 
\begin{align*}
    \hbeta_{t, \mu} = \hbeta_{0, \mu} + \sum_{i = 1}^{N_t} A_{t, i, \mu} (Z_{t, i}^{\top} Z_{t, i})^{\dagger} Z_{t, i}^{\top} \zeta_{t, i}.
\end{align*}
    
\end{lem}

\begin{proof}[Proof of \cref{lemma:hat_beta_expression}]
    We prove \cref{lemma:hat_beta_expression} in Appendix \ref{proof:lemma:hat_beta_expression}. 
\end{proof}

Note that $\mbox{Cov}[\zeta_t] \succeq \rho_2 I$ by the second point of \cref{lemma:data-transmission}. 
Hence, we obtain the following lower bound: 
\begin{align}
\label{eq:beta-lower}
\E\big[\|\hat\beta_{0, \mu} - \hbeta_{t, \mu}\|_2^2\big] \geq \, \rho_2 \sum_{i = 1}^{N_t} \E\Big[  \Tr \big[ A_{t, i, \mu} (Z_{t, i}^{\top}Z_{t, i})^{\dagger} A_{t, i, \mu}^{\top} \big] \Big] = \rho_2 \sum_{i = 1}^{N_t} \E\Big[  \Tr \big[A_{t, i, \mu}^{\top} A_{t, i, \mu} (Z_{t, i}^{\top}Z_{t, i})^{\dagger} \big] \Big]. 
\end{align}
For any $t \in \N_+$ and $i \in [N_t]$, as $\cP_{t, i} \neq \emptyset$ by the first point of \cref{lemma:data-transmission}, there exists $(s_{t, i}, \nu_{t, i, 1} \to \nu_{t, i, 2}, h_{t, i}) \in \cP_{t, i}$. 
Therefore, for any $t \in \N_+$, $i \in [N_t]$ and $\mu \in \cM_l$, 
\begin{align}
\label{eq:Tr-lower-bound}
\begin{split}
     & \E\Big[  \Tr \big[A_{t, i, \mu}^{\top} A_{t, i, \mu} (Z_{t, i}^{\top}Z_{t, i})^{\dagger} \big] \Big] \\
     & \overset{(i)}{\geq} \,  \E\Big[  \Tr \big[A_{t, i, \mu}^{\top} A_{t, i, \mu} (X_{s_{t, i}, \nu_{t, i, 1}}^{\top} X_{s_{t, i}, \nu_{t, i, 1}})^{-1} \big] \Big]  \\
      & \overset{(ii)}{\geq} \,  n_{s_{t, i}, \nu_{t, i, 1}}^{-1} \kappa^{-1} \E\Big[  \Tr \big[A_{t, i, \mu}^{\top} A_{t, i, \mu} 1 \{X_{s_{t, i}, \nu_{t, i, 1}}^{\top} X_{s_{t, i}, \nu_{t, i, 1}} / n_{s_{t, i}, \nu_{t, i, 1}} \preceq \kappa I_d \} \big] \Big] \\
     &{\geq} \,  n_{s_{t, i}, \max}^{-1} \kappa^{-1} \E\Big[  \Tr \big[A_{t, i, \mu}^{\top} A_{t, i, \mu} 1 \{X_{s_{t, i}, \nu_{t, i, 1}}^{\top} X_{s_{t, i}, \nu_{t, i, 1}} / n_{s_{t, i}, \nu_{t, i, 1}} \preceq \kappa I_d \} \big] \Big] \\
     &\geq \,  n_{s_{t, i}, \max}^{-1} \kappa^{-1} \E\Big[  \Tr \big[A_{t, i, \mu}^{\top} A_{t, i, \mu} \min_{v \in \cM} 1 \{X_{s_{t, i}, v}^{\top} X_{s_{t, i}, v} / n_{s_{t, i}, v} \preceq \kappa I_d \} \big] \Big]  , 
\end{split}
\end{align}
where we recall that $\kappa$ is from \cref{ass:concentration}, and $n_{s, \max}$ is from Assumption \ref{ass:sample-size-ratio}. 
In the above equations, $(i)$ is because $X_{s_{t, i}, \nu_{t, i, 1}}^{\top} X_{s_{t, i}, \nu_{t, i, 1}}$ is invertible by Assumption \ref{assumption:full-column-rank}, and $(ii)$ is because 
\begin{align*}
	&  \Tr \big[A_{t, i, \mu}^{\top} A_{t, i, \mu} (X_{s_{t, i}, \nu_{t, i, 1}}^{\top} X_{s_{t, i}, \nu_{t, i, 1}})^{-1}\big] \\
	& = \Tr \big[ A_{t, i, \mu}^{\top} A_{t, i, \mu} (X_{s_{t, i}, \nu_{t, i, 1}}^{\top} X_{s_{t, i}, \nu_{t, i, 1}})^{-1} 1 \{X_{s_{t, i}, \nu_{t, i, 1}}^{\top} X_{s_{t, i}, \nu_{t, i, 1}} / n_{s_{t, i}, \nu_{t, i, 1}} \preceq \kappa I_d \} \big] \\
	& \qquad + \Tr \big[A_{t, i, \mu}^{\top} A_{t, i, \mu} (X_{s_{t, i}, \nu_{t, i, 1}}^{\top} X_{s_{t, i}, \nu_{t, i, 1}})^{-1} \big(1 - 1 \{X_{s_{t, i}, \nu_{t, i, 1}}^{\top} X_{s_{t, i}, \nu_{t, i, 1}} / n_{s_{t, i}, \nu_{t, i, 1}} \preceq \kappa I_d \}\big)\big] \\
	& \geq n_{s_{t, i}, \nu_{t, i, 1}}^{-1} \Tr\big[ A_{t, i, \mu}^{\top} A_{t, i, \mu} (X_{s_{t, i}, \nu_{t, i, 1}}^{\top} X_{s_{t, i}, \nu_{t, i, 1}} / n_{s_{t, i}, \nu_{t, i, 1}})^{-1} 1 \{X_{s_{t, i}, \nu_{t, i, 1}}^{\top} X_{s_{t, i}, \nu_{t, i, 1}} / n_{s_{t, i}, \nu_{t, i, 1}} \preceq \kappa I_d \} \big] \\
	& \geq n_{s_{t, i}, \nu_{t, i, 1}}^{-1} \kappa^{-1} \Tr\big[ A_{t, i, \mu}^{\top} A_{t, i, \mu} 1 \{X_{s_{t, i}, \nu_{t, i, 1}}^{\top} X_{s_{t, i}, \nu_{t, i, 1}} / n_{s_{t, i}, \nu_{t, i, 1}} \preceq \kappa I_d \} \big]. 
\end{align*} 

The next lemma analyzes the sum $\sum_{i = 1}^{N_t}A_{t, i, \mu}$.

\begin{lem}
\label{lemma:sumA}
Under Assumptions \ref{assumption:data-distribution} and \ref{assumption:full-column-rank}, for all $t \in \N_+$ we have 
\begin{align*}
    \sum_{i = 1}^{N_t} A_{t, i, \mu} = \sum_{s = 1}^t \sum_{\nu \in \cM_l} \sum_{m \in \cM} \Omega_{t, s + 1, \mu, \nu} \cT_{s, \nu, m}.
\end{align*}
\end{lem}

\begin{proof}[Proof of \cref{lemma:sumA}]
    We prove \cref{lemma:sumA} in Appendix \ref{proof:lemma:sumA}. 
\end{proof}

We use $v_1 \to v_2 \to \cdots \to v_x$ to represent a path connecting nodes $v_1, v_2, \cdots, v_x \in \cM$, with $(v_i, v_{i + 1}) \in \cE$ for $i = 1, 2, \cdots, x - 1$, and we define $x-1$ as the length of this path. 
For $\ell \in \N_+$, we denote by $\cL_\ell$ the collection of all paths of length $\ell$ in $\cG$. 
Throughout this paper, we make the convention that $\cL_0 = \emptyset$.
Then by \cref{lemma:sumA}, 
\begin{align}
\label{eq:sumA-equality}
\begin{split}
    \E\Big[ \sum_{i = 1}^{N_t} A_{t, i, \mu} \Big] = & \E\Big[ \sum_{s = 1}^t \sum_{\nu \in \cM_l} \sum_{m \in \cM} \Omega_{t, s + 1, \mu, \nu} \cT_{s, \nu, m} \Big] \\
    = & \sum_{x = 1}^t \sum_{v_1 \to v_2 \to \cdots \to v_x \to \mu \in \cL_x} \E\Big[  \cT_{t, \mu, v_x} \cT_{t - 1, v_x, v_{x - 1}} \cdots \cT_{t - x + 1, v_2, v_1} \Big] \\
    \overset{(i)}{=} & \sum_{x = 1}^t \sum_{v_1 \to v_2 \to \cdots \to v_x \to \mu \in \cL_x} \E[  \cT_{t, \mu, v_x}] \E[ \cT_{t - 1, v_x, v_{x - 1}}] \cdots \E[ \cT_{t - x + 1, v_2, v_1}], 
\end{split}
\end{align}
where $(i)$ is by Assumption \ref{ass:independent_covariate}. 
For $x \in \N_+$,  $\ell \in \{0\} \cup [x]$, and $v \in \cM_l^{\infty}$, we denote by $\cL_{x, \ell, v}$ the collection of paths of length $x$ whose last visit to $\mathcal{M}_l^{\infty}$ occurs at model $v \in \mathcal{M}_l^{\infty}$ at length $x - \ell$: 
\begin{align}
\label{eq:Lxellv}
    \cL_{x, \ell, v} = \Big\{ v_1 \to v_2 \to \cdots \to v_x \to v_{x + 1}: \, v_{x - \ell + 1} = v \in \cM_l^{\infty}, \, v_{x - \ell + 2}, v_{x - \ell + 3}, \cdots, v_{x + 1} \notin \cM_l^{\infty}   \Big\}. 
\end{align}
For $x \in \N_+$,
we denote by $\cL_{x}^{\ast}$ the collection of paths of length $x$ that do not hit models in $ \cM_l^{\infty}$: 
\begin{align}
\label{eq:Lxstar}
    \cL_x^{\ast} = \Big\{ v_1 \to v_2 \to \cdots \to v_x \to v_{x + 1}: \, v_i \notin \cM_l^{\infty} \mbox{ for all }i \in [x + 1]\Big\}. 
\end{align}
Similarly, we make the convention that $\cL_0^{\ast} = \cL_{0, \ell, v} = \emptyset$. 
With definitions \eqref{eq:Lxellv} and \eqref{eq:Lxstar},
we then reformulate the last line of \cref{eq:sumA-equality} as 
\begin{align}
\label{eq:two-sums}
\begin{split}
	\E\Big[ \sum_{i = 1}^{N_t} A_{t, i, \mu} \Big] = & \sum_{x = 1}^{t} \sum_{v \in \cM_l^{\infty}} \sum_{\ell = 0}^{x} \sum_{v_1 \to v_2 \to \cdots \to v_{x} \to \mu \in \cL_{x, \ell, v}} \E\big[ \cT_{t,  \mu, v_{x}} \big] \E \big[ \cT_{t - 1, v_{x}, v_{x - 1}}\big] \cdots \E \big[\cT_{t - x + 1, v_2, v_1} \big] \\
	& + \sum_{x = 1}^{t } \sum_{v_1 \to v_2 \to \cdots \to v_{x} \to \mu \in \cL_{x}^{\ast}} \E\big[ \cT_{t,  \mu, v_{x}} \big] \E \big[ \cT_{t - 1, v_{x}, v_{x - 1}}\big] \cdots \E \big[\cT_{t - x + 1, v_2, v_1} \big].
\end{split}
\end{align}
We next analyze the above two sums on the right hand side of \cref{eq:two-sums}, starting from the second sum. 

\begin{lem}
\label{lemma:sum-decay}
    Under Assumptions \ref{ass:independent_covariate}-\ref{ass:sample-size-ratio}, for any $\mu \in \cM_l \backslash \cM_l^{\infty}$, $t \in \N_+$ and $x \in [t]$, it holds that 
\begin{align*}
    & \Big \| \sum_{v_1 \to v_2 \to \cdots \to v_{x} \to \mu \in \cL_{x}^{\ast}} \E\big[ \cT_{t,  \mu, v_{x}} \big] \E \big[ \cT_{t - 1, v_{x}, v_{x - 1}}\big] \cdots \E \big[\cT_{t - x + 1, v_2, v_1} \big] \Big\|_{\op} \\
    & \leq \sum_{v_1 \to v_2 \to \cdots \to v_{x} \to \mu \in \cL_{x}^{\ast}} \big\|\E\big[ \cT_{t,  \mu, v_{x}} \big] \big\|_{\op} \big\| \E \big[ \cT_{t - 1, v_{x}, v_{x - 1}}\big] \big\| _{\op} \cdots \big\| \E \big[\cT_{t - x + 1, v_2, v_1} \big]\big\|_{\op} \\
    & \leq \Big(1 + \frac{\delta + 2\gamma_1}{\alpha}\Big)^{K} \left[ \Big(1 + \frac{\delta + 2\gamma_1}{\alpha}\Big)^{K} - \alpha^{K} \right]^{\lceil x/ K\rceil}. 
\end{align*}
By Assumption \ref{ass:concentration}, we further know that the last line above is no larger than $2 c(\alpha, K)^{\lceil x / K \rceil}$. 
On the other hand, for any $\mu \in \cM_l^{\infty}$, straightforwardly we have 
\begin{align*}
	 \sum_{v_1 \to v_2 \to \cdots \to v_{x} \to \mu \in \cL_{x}^{\ast}} \E\big[ \cT_{t,  \mu, v_{x}} \big] \E \big[ \cT_{t - 1, v_{x}, v_{x - 1}}\big] \cdots \E \big[\cT_{t - x + 1, v_2, v_1} \big] = 0_{d \times d}. 
\end{align*}
\end{lem}
\begin{proof}[Proof of \cref{lemma:sum-decay}]

We prove \cref{lemma:sum-decay} in Appendix \ref{proof:lemma:sum-decay}. 
\end{proof}

We then switch to analyze the first sum on the right hand side of \cref{eq:two-sums}. 

\begin{lem}
\label{lemma:sum-grow}
Recall that $\cM_l^{\rm c}$ and $\cM_l^{\rm nc}$ are defined in \cref{eq:cM_cnc}.
Under Assumptions \ref{ass:independent_covariate}-\ref{ass:sample-size-ratio}, if $\mu \in \cM_l^c$, then for $t \geq K$, it holds that
	\begin{align*}
		& \Big\| \sum_{x = 1}^{t} \sum_{v \in \cM_l^{\infty}} \sum_{\ell = 0}^{x} \sum_{v_1 \to v_2 \to \cdots \to v_{x} \to \mu \in \cL_{x, \ell, v}} \E\big[ \cT_{t,  \mu, v_{x}} \big] \E \big[ \cT_{t - 1, v_{x}, v_{x - 1}}\big] \cdots \E \big[\cT_{t - x + 1, v_2, v_1} \big]\Big\|_{\op}  \geq  t\, c_2 - \alpha^K (K - 1), 
	\end{align*}
    where we recall $c_1, c_2$ are from Assumption \ref{ass:concentration}. 
	On the other hand, if $\mu \in \cM_l^{\rm nc}$, then straightforwardly, 
	\begin{align*}
		\sum_{x = 1}^{t} \sum_{v \in \cM_l^{\infty}} \sum_{\ell = 0}^{x} \sum_{v_1 \to v_2 \to \cdots \to v_{x} \to \mu \in \cL_{x, \ell, v}} \E\big[ \cT_{t,  \mu, v_{x}} \big] \E \big[ \cT_{t - 1, v_{x}, v_{x - 1}}\big] \cdots \E \big[\cT_{t - x + 1, v_2, v_1} \big] = 0_{d \times d}. 
	\end{align*}
\end{lem}

\begin{proof}[Proof of \cref{lemma:sum-grow}]

We prove \cref{lemma:sum-grow} in Appendix \ref{proof:lemma:sum-grow}. 
\end{proof}

Substituting the bounds from Lemmas \ref{lemma:sum-decay} and \ref{lemma:sum-grow} into \cref{eq:two-sums}, 
we conclude that for $t \geq K$ and $\mu \in \cM_l^{\rm c}$, 
\begin{align}
\label{eq:A-lower}
\begin{split}
	\Big\|\,\E\Big[ \sum_{i = 1}^{N_t} A_{t, i, \mu} \Big]	\,\Big\|_{\op} \geq & \, t \, c_2 - \alpha^K(K - 1) - \sum_{x = 1}^t \Big(1 + \frac{\delta + 2\gamma_1}{\alpha}\Big)^{K} \left[ \Big(1 + \frac{\delta + 2\gamma_1}{\alpha}\Big)^{K} - \alpha^{K} \right]^{\lceil x/ K\rceil} \\
    \geq & \, t \, c_2 - \alpha^K(K - 1) - \sum_{x = 1}^t 2 c(\alpha, K)^{\lceil x/ K\rceil} \\
	\geq & \, t \, c_2 - \alpha^K (K - 1) - 2K \frac{c(\alpha, K)}{1 - c(\alpha, K)} \\
    \geq & \, t \, c_2 - \alpha^K (K - 1) - 4 K \alpha^{-K},
\end{split}
\end{align}
where the last inequality follows from taking $c(\alpha, K) = 1 - \alpha^K / 2$, as pointed out in \cref{ass:concentration}.
%

We next upper bound the operator norm of 
\begin{equation}\label{eq:sumA-indicator2}
    \E \Big[ \sum_{i = 1}^{N_t} A_{t, i, \mu} \max_{v \in \cM} \big( 1 - 1 \{X_{s_{t, i}, v}^{\top} X_{s_{t, i}, v} / n_{s_{t, i}, v} \preceq \kappa I_d \} \big) \Big].
\end{equation}
The following lemma provides a more tractable characterization of this sum. 
\begin{lem}
\label{lemma:sumA2}
	Under Assumptions \ref{assumption:data-distribution} and \ref{assumption:full-column-rank}, for all $t \in \N_+$ we have 
\begin{align*}
    & \sum_{i = 1}^{N_t} A_{t, i, \mu} \max_{v \in \cM} \big( 1 - 1 \{X_{s_{t, i}, v}^{\top} X_{s_{t, i}, v} / n_{s_{t, i}, v} \preceq \kappa I_d \} \big) \\
    & = \sum_{s = 1}^t \sum_{\nu \in \cM_l} \sum_{m \in \cM} \Omega_{t, s + 1, \mu, \nu} \cT_{s, \nu, m} \max_{v \in \cM} \big( 1 - 1 \{X_{s_{t, i}, v}^{\top} X_{s_{t, i}, v} / n_{s_{t, i}, v} \preceq \kappa I_d \} \big) .
\end{align*}
\end{lem}
\begin{proof}[Proof of \cref{lemma:sumA2} ]
The proof is completely the same as that of \cref{lemma:sumA}, by noting that $s_{t, i} = s$ in the summation.
\end{proof}

We then leverage \cref{lemma:sumA2} to upper bound the operator norm of \eqref{eq:sumA-indicator2}.

\begin{lem}
\label{lemma:A-indicator-upper-bound}
	Under Assumptions \ref{ass:independent_covariate}-\ref{ass:sample-size-ratio}, for any $\mu \in \cM_l$ and $t \in \N_+$, it holds that 
 	\begin{align*}
 		\norm{\E \Big[ \sum_{i = 1}^{N_t} A_{t, i, \mu} \max_{v \in \cM} \big( 1 - 1 \{X_{s_{t, i}, v}^{\top} X_{s_{t, i}, v} / n_{s_{t, i}, v} \preceq \kappa I_d \} \big) \Big]}_{\op} \leq \, \frac{9\,t\, { K^3 \gamma_2}}{\alpha^{K+1}}.
 	\end{align*}
\end{lem}

\begin{proof}[Proof of \cref{lemma:A-indicator-upper-bound}]
	We prove \cref{lemma:A-indicator-upper-bound} in Appendix \ref{proof:lemma:A-indicator-upper-bound}.
\end{proof}

We then combine \cref{lemma:A-indicator-upper-bound} and \cref{eq:A-lower} to obtain the following operator norm lower bound for $\mu \in \cM_l^{\rm c}$: 
\begin{align*}
	& \Big\|\E\Big[ \sum_{i = 1}^{N_t} A_{t, i, \mu} \min_{v \in \cM} 1 \{X_{s_{t, i}, v}^{\top} X_{s_{t, i}, v} / n_{s_{t, i}, v} \preceq \kappa I_d \} \Big] \Big\|_{\op} \\
	& \geq \Big\|\,\E\Big[ \sum_{i = 1}^{N_t} A_{t, i, \mu} \Big]	\,\Big\|_{\op} - \Big\|\E\Big[ \sum_{i = 1}^{N_t} A_{t, i, \mu} \max_{v \in \cM} \big( 1 - 1 \{X_{s_{t, i}, v}^{\top} X_{s_{t, i}, v} / n_{s_{t, i}, v} \preceq \kappa I_d \} \big) \Big] \Big\|_{\op}  \\
	& \geq t \, c_2 - \alpha^K (K - 1) - 4K \alpha^{-K} - \frac{9\,t\, { K^3 \gamma_2}}{\alpha^{K+1}} \ge \left( c_2 - \frac{9 K^3 \gamma_2}{\alpha^{K+1}} \right) t - \frac{5 K}{\alpha^K}. 
\end{align*}
Leveraging Eqs. \eqref{eq:beta-lower}, \eqref{eq:Tr-lower-bound}  and the above lower bound, we see that 
\begin{align*}
	& \E\big[\|\hat\beta_{0, \mu} - \hbeta_{t, \mu}\|^2\big] \\
    & \geq\,  \rho_2 \kappa^{-1} \sum_{i = 1}^{N_t} n_{s_{t, i}, \max}^{-1}  \E\Big[  \Tr \big[A_{t, i, \mu}^{\top} A_{t, i, \mu} \min_{v \in \cM} 1 \{X_{s_{t, i}, v}^{\top} X_{s_{t, i}, v} / n_{s_{t, i}, v} \preceq \kappa I_d \} \big] \Big] \\
	& \geq \,  \rho_2 \kappa^{-1} \big( \sup_{s \in \{0\} \cup  [t]} n_{s, \max} \big)^{-1} \sum_{i = 1}^{N_t} \E\Big[  \Tr \big[A_{t, i, \mu}^{\top} A_{t, i, \mu} \min_{v \in \cM} 1 \{X_{s_{t, i}, v}^{\top} X_{s_{t, i}, v} / n_{s_{t, i}, v} \preceq \kappa I_d \} \big] \Big]  \\ 
	& = \,  \rho_2 \kappa^{-1} \big( \sup_{s \in \{0\} \cup  [t]} n_{s, \max} \big)^{-1} \sum_{i = 1}^{N_t} \E\Big[   \big\| A_{t, i, \mu} \min_{v \in \cM} 1 \{X_{s_{t, i}, v}^{\top} X_{s_{t, i}, v} / n_{s_{t, i}, v} \preceq \kappa I_d \} \big\|_F^2 \Big] \\
	& \geq \,  \rho_2 \kappa^{-1} \big( \sup_{s \in \{0\} \cup  [t]} n_{s, \max} \big)^{-1} N_t^{-1}  \Big\|\E\Big[ \sum_{i = 1}^{N_t} A_{t, i, \mu} \min_{v \in \cM} 1 \{X_{s_{t, i}, v}^{\top} X_{s_{t, i}, v} / n_{s_{t, i}, v} \preceq \kappa I_d \} \Big] \Big\|_{F}^2 \\
	& \geq \,  \rho_2 \kappa^{-1} \big( \sup_{s \in \{0\} \cup  [t]} n_{s, \max} \big)^{-1} N_t^{-1}  \Big\|\E\Big[ \sum_{i = 1}^{N_t} A_{t, i, \mu} \min_{v \in \cM} 1 \{X_{s_{t, i}, v}^{\top} X_{s_{t, i}, v} / n_{s_{t, i}, v} \preceq \kappa I_d \} \Big] \Big\|_{\op}^2 \\
	& \geq \,  \rho_2 \kappa^{-1} \big( \sup_{s \in \{0\} \cup  [t]} n_{s, \max} \big)^{-1} N_t^{-1}  \bigg( \Big( c_2 - \frac{9 K^3 \gamma_2}{\alpha^{K+1}} \Big) t - \frac{5 K}{\alpha^K} \bigg)^2. 
\end{align*} 
Recall $N_t = O(t)$ by \cref{lemma:data-transmission}, and $\sup_{s \in \{0\} \cup [t]} n_{s, \max} / t \to 0$ as $t \to \infty$ by Assumption \ref{ass:sample-size-ratio}.
This implies that the last line above tends to infinity as $t \to \infty$ under the current set of assumptions. 
By the temporal independence assumption (Assumption \ref{ass:independent_covariate}), we have 
\begin{align}
\label{eq:Mc-numerator}
	\liminf_{t \to \infty}\E[\|\beta_{\ast} - \hat\beta_{t, \mu}\|_2^2] = \liminf_{t \to \infty} \E\big[\|\hat\beta_{0, \mu} - \hbeta_{t, \mu}\|_2^2\big]  + \E\big[\|\beta_{\ast} - \hat\beta_{0, \mu}\|_2^2\big] \to \infty \,\,\, \mbox{ as }t \to \infty.  
\end{align} 
Under Assumptions \ref{ass:independent_covariate}\,--\,\ref{assumption:full-column-rank}, we have 
\begin{align}
\label{eq:Mc-denominator}
	\limsup_{t \to \infty}\E[\|\beta_{\ast} - \hat\beta_{t, \mu}^{\ast}\|_2^2] 
    \leq \limsup_{t \to \infty} \rho_1 \,\E\Big[\, \Tr\Big[ \Big( \sum_{\nu \in \Nin{\mu}} X_{t, \nu \to \mu}^{\top} X_{t, \nu \to \mu}\Big)^{-1} \Big]\, \Big] < \infty. 
\end{align}
\cref{eq:Mc-numerator,eq:Mc-denominator} together complete the proof of the theorem.

\subsection{Proof of Theorem \ref{thm:linear-non-collapse}}
\label{appendix:thm:linear-non-collapse}

Recall that for $t \ge s$, we define $\Omega_{t,s} = \cT_t \cT_{t - 1} \cdots \cT_s \in \R^{dK \times dK}$, with the convention that $\Omega_{t,t+1} = I_{dK}$. The matrix $\Omega_{t,s}$ has a $K \times K$ block structure, where the $(i,j)$-th block is denoted by $\Omega_{t,s,\mu_i,\mu_j} \in \mathbb{R}^{d \times d}$.
According to the definition of $\cM_l^{\rm nc}$ in \cref{eq:cM_cnc}, for every $\mu \in \cM_l^{\rm nc}$, there is no path originating in $\cM_l^{\rm c}$ and terminating at $\mu$. 
Invoking this observation together with \cref{eq:hat_beta_expression}, we conclude that for all $\mu \in \cM_l^{\rm nc}$,
\begin{align}
\label{eq:hbeta-starbeta2}
    \hbeta_{t, \mu} - \hbeta_{0, \mu} = \sum_{s = 1}^t \sum_{\nu \in \cM \backslash \cM_l^{\rm c}} \Omega_{t, s + 1, \mu, \nu} v_{s, \nu}, 
\end{align}
where we recall that $v_{s, \nu}$ is defined in \cref{eq:v-t-mu}.  
Under Assumption \ref{ass:independent_covariate}, $v_s$ is independent of $\Omega_{t,s+1}$. Combining this observation with \cref{eq:hbeta-starbeta2} yields
\begin{align}
\label{eq:three-lines}
\begin{split}
    &\E \Big[ \norm{\hbeta_{t, \mu} - \hbeta_{0, \mu}}_2^2 \Big] = \,  \E \Big[\, \Big\|{\sum_{s = 1}^{t} \sum_{\nu \in \cM \backslash \cM_l^{\rm c}} {\Omega}_{t, s + 1, \mu, \nu} v_{s, \nu}}\Big\|_2^2\, \Big] \\
    & \overset{(i)}{\leq} \,  \left( \sum_{s = 1}^t \sum_{\nu \in \cM \backslash \cM_l^{\rm c}} \E\Big[ v_{s, \nu}^{\top} {\Omega}_{t, s + 1, \mu, \nu}^{\top} {\Omega}_{t, s + 1, \mu, \nu} v_{s, \nu} \Big]^{1/2} \right)^2 \\
    & \overset{(ii)}{=} \,  \left( \sum_{s = 1}^t \sum_{\nu \in \cM \backslash \cM_l^{\rm c}} \E\Big[ v_{s, \nu}^{\top} \,\E[\, {\Omega}_{t, s + 1, \mu, \nu}^{\top} {\Omega}_{t, s + 1, \mu, \nu} \,]\, v_{s, \nu} \Big]^{1/2} \right)^2,
\end{split}
\end{align}
where $(i)$ follows from the  Cauchy–Schwarz inequality, and $(ii)$ follows from the fact that $v_s$ is independent of $\Omega_{t,s+1}$.
The remainder of this section is devoted to establishing an upper bound for 
\begin{align*}
    E_{t, s, \mu} = \sum_{\nu \in \cM \backslash \cM_l^{\rm c}} \E\Big[ v_{s, \nu}^{\top} \,\E[\, {\Omega}_{t, s + 1, \mu, \nu}^{\top} {\Omega}_{t, s + 1, \mu, \nu} \,]\, v_{s, \nu} \Big]^{1/2}
\end{align*}
for all $t \in \N_+$, $s \in [t]$ and $\mu \in \cM_l^{\rm nc}$. 
Next, we employ an induction argument to show that this quantity decays exponentially in $t$. 
Specifically, we prove that $E_{t+t_0, s, \mu} \leq \theta E_{t, s, \mu}$ for suitably chosen constants $\theta \in (0,1)$ and $t_0 \in \N_+$.

Recall that $\Omega_{t, s + 1} = \cT_t \cT_{t - 1} \cdots \cT_{s+1}$.
Therefore, for any $t_0 \in \N_+$, we have  $\Omega_{t + t_0, s + 1} = \Omega_{t + t_0, t + 1} \Omega_{t ,s + 1}$. 
Therefore, for all $\mu \in \cM_l^{\rm nc}$ and $\nu \in \cM \backslash \cM_l^{\rm c}$, 
\begin{align}
\label{eq:tilde-Omega}
\Omega_{t + t_0, s + 1, \mu, \nu} v_{s, \nu} = \sum_{m \in \cM} \Omega_{t + t_0, t + 1, \mu, m} \, \Omega_{t, s + 1, m, \nu} v_{s, \nu} = \sum_{m \in \cM_l^{\rm nc}} \Omega_{t + t_0, t + 1, \mu, m}\, \Omega_{t, s + 1, m, \nu} v_{s, \nu}.
\end{align}
Substituting \cref{eq:tilde-Omega} into the definition of $E_{t,s,\mu}$, we obtain
\begin{align*}
    E_{t + t_0, s, \mu} =\, & \sum_{\nu \in \cM \backslash \cM_l^{\rm c}} \Big(\, \E\Big[ \sum_{m_1, m_2 \in \cM_l^{\rm nc}}\, v_{s, \nu}^{\top}   \Omega_{t, s+1, m_1, \nu}^{\top} \, \Omega_{t+t_0, t+1, \mu, m_1}^{\top}  \, \Omega_{t+t_0, t+1, \mu, m_2} \, \Omega_{t, s+1, m_2, \nu} v_{s, \nu} \, \Big]\,\Big)^{1/2} \\
    =\, & \sum_{\nu \in \cM \backslash \cM_l^{\rm c}} \Big(\, \sum_{m_1, m_2 \in \cM_l^{\rm nc}} \E\Big[  \, v_{s, \nu}^{\top}  \Omega_{t, s+1, m_1, \nu}^{\top} \, \E \big[\,  \Omega_{t+t_0, t+1, \mu, m_1}^{\top}  \,  \Omega_{t+t_0, t+1, \mu, m_2}\, \big] \, \Omega_{t, s+1, m_2, \nu} v_{s, \nu} \, \Big]\,\Big)^{1/2},
\end{align*}
where the last equality above follows from \cref{ass:independent_covariate} and the observation that $(\Omega_{t,s+1}, v_s)$ is a function of $(\cX_i, \eps_i)_{s \le i \le t}$, and $\Omega_{t+t_0,t+1}$ is a function of $(\cX_i)_{t+1 \le i \le t+t_0}$.
From the last line above, we further conclude that 
\begin{align}
\label{eq:EE-bounds}
    E_{t + t_0, s, \mu} \leq & \sum_{\nu \in \cM \backslash \cM_l^{\rm c}} \Big(\hspace{-0.5em}\sum_{m_1, m_2 \in \cM_l^{\rm nc}}  \hspace{-0.5em}\E\big[ \| v_{s, \nu}^{\top}   \Omega_{t, s+1, m_1, \nu}^{\top} \|_2 \| \Omega_{t, s+1, m_2, \nu} v_{s, \nu} \|_2 \big]    \big\| \E \big[  \Omega_{t+t_0, t+1, \mu, m_1}^{\top}   \Omega_{t+t_0, t+1, \mu, m_2} \big] \big\|_{\op} \Big)^{1/2} \nonumber \\
    \overset{(i)}{\leq} & \sum_{\nu \in \cM \backslash \cM_l^{\rm c}} \sum_{m \in \cM_l^{\rm nc}} \E \Big[ v_{s, \nu}^{\top} \Omega_{t, s + 1, m, \nu}^{\top} \, \Omega_{t, s + 1, m, \nu} v_{s, \nu} \Big]^{1/2} \, \Big\|\, \E\big[ \, \Omega_{t + t_0, t + 1, \mu, m}^{\top}\,  \Omega_{t + t_0, t + 1, \mu, m} \big] \Big\|_{\op}^{1/2} \nonumber \\
    = & \sum_{m \in \cM_l^{\rm nc}} E_{t, s, m}\, \Big\|\,  \E\big[  \Omega_{t + t_0, t + 1, \mu, m}^{\top} \, \Omega_{t + t_0, t + 1, \mu, m} \big] \Big\|_{\op}^{1/2},
\end{align}
where $(i)$ is obtained by applying the matrix Cauchy-Schwarz inequality as follows:
\begin{align*}
    &\E\big[ \| v_{s, \nu}^{\top}   \Omega_{t, s+1, m_1, \nu}^{\top}  \|_2 \|  \Omega_{t, s+1, m_2, \nu} v_{s, \nu} \|_2 \big] \leq  \E \big[ v_{s, \nu}^{\top} \Omega_{t, s + 1, m_1, \nu}^{\top}  \Omega_{t, s + 1, m_1, \nu} v_{s, \nu} \big]^{1/2} \E \big[ v_{s, \nu}^{\top}  \Omega_{t, s + 1, m_2, \nu}^{\top}  \Omega_{t, s + 1, m_2, \nu} v_{s, \nu} \big]^{1/2}, \\
    & \big\| \E \big[ \Omega_{t+t_0, t+1, \mu, m_1}^{\top}  \Omega_{t+t_0, t+1, \mu, m_2}\, \big] \big\|_{\op} \leq   \big\| \E\big[  \Omega_{t + t_0, t + 1, \mu, m_1}^{\top}  \Omega_{t + t_0, t + 1, \mu, m_1} \big] \big\|_{\op}^{1/2}  \big\| \E\big[ \Omega_{t + t_0, t + 1, \mu, m_2}^{\top}  \Omega_{t + t_0, t + 1, \mu, m_2} \big] \big\|_{\op}^{1/2}. 
\end{align*}
Using \cref{eq:EE-bounds}, we have
\begin{align}
\label{eq:supE-supE}
    \sup_{\mu \in \cM_l^{\rm nc}} E_{t + t_0, s, \mu} \leq \sup_{\mu \in \cM_l^{\rm nc}} E_{t, s, \mu}  \times \sup_{\mu \in \cM_l^{\rm nc}} \Bigg\{ \sum_{m \in \cM_l^{\rm nc}} \Big\|\, \E\big[ \Omega_{t + t_0, t + 1, \mu, m}^{\top} \Omega_{t + t_0, t + 1, \mu, m} \big] \Big\|_{\op}^{1/2} \Bigg\}.
\end{align}
To establish that $\sup_{\mu \in \cM_l^{\rm nc}} E_{t,s,\mu}$ decays exponentially in $t$, we leverage Assumption~\ref{ass:concentration}, which characterizes the concentration behavior of $\cT_t$.
Recall $\cT_t^{\ast}$ is defined in \cref{eq:cT-star-def}. 
Define $\Omega_{t,s+1}^{\ast} = \cT_t^{\ast} \cT_{t-1}^{\ast} \cdots \cT_{s+1}^{\ast}$ for $0 \le s \le t$, with the convention that $\Omega_{t,t+1}^{\ast} = I_{dK}$.
The following lemma shows that, under \cref{ass:concentration}, $\Omega_{t,s+1}^{\ast}$ provides a close approximation to $\Omega_{t,s+1}$ with high probability.
\begin{lem}\label{prop:concentration_Omega}
    Recall $N_{\max} = \sup_{\mu \in \cM} | \Nin{\mu} |$. 
    Under \cref{ass:concentration}, for any $0 \le s \le t$ and $\nu_1, \nu_2 \in \cM$,
\begin{equation}\label{eq:err_bound_Omega}
        \norm{\Omega_{t, s + 1, \nu_1, \nu_2}^* - \Omega_{t, s + 1, \nu_1, \nu_2}}_{\op} \le \, \gamma_1 \sum_{j = 1}^{t - s} N_{\max}^j 
    \end{equation}
    with probability at least $1 - \delta \sum_{j = 1}^{t - s} N_{\max}^j$.
\end{lem}
\begin{proof}[Proof of Lemma \ref{prop:concentration_Omega}]
    We prove Lemma \ref{prop:concentration_Omega} in Appendix \ref{appendix:prop:concentration_Omega}. 
\end{proof}

For all $t \ge s \ge 0$ and $\mu, \nu \in \cM$, define the event:
\begin{equation*}
    \mathcal{E}_{t, s + 1, \mu, \nu} = \, \Bigg\{ \norm{\Omega_{t, s + 1, \mu, \nu}^* - \Omega_{t, s + 1, \mu, \nu}}_{\op} \le \, \gamma_1 \sum_{j = 1}^{t - s} N_{\max}^j \Bigg\}.
\end{equation*}
By Lemma \ref{prop:concentration_Omega}, we have
\begin{equation*}
    \P \left( \mathcal{E}_{t, s + 1, \mu, \nu} \right) \ge 1 - \delta \sum_{j = 1}^{t - s} N_{\max}^j.
\end{equation*}
On the event $\mathcal{E}_{t, s + 1, \mu, \nu}$, we have
\begin{align*}
    & \norm{ \Omega_{t, s + 1, \mu, \nu}^{\top} \Omega_{t, s + 1, \mu, \nu} - \Omega_{t, s + 1, \mu, \nu}^{*\top} \Omega_{t, s + 1, \mu, \nu}^{*} }_{\op} \\ &\stackrel{(i)}{\le} \,  2 \norm{\Omega_{t, s + 1, \mu, \nu} - \Omega_{t, s + 1, \mu, \nu}^* }_{\op} + \norm{\Omega_{t, s + 1, \mu, \nu} - \Omega_{t, s + 1, \mu, \nu}^* }_{\op}^2 \\
    &\le \,  2 \gamma_1 \sum_{j = 1}^{t - s} N_{\max}^j + \gamma_1^2 \big( \sum_{j = 1}^{t - s} N_{\max}^j \big)^2,
\end{align*}
where $(i)$ follows from the fact that $\|{\Omega_{t, s + 1, \mu, \nu}^*}\|_{\op} \le 1$. 
By induction, it follows that 
\begin{equation*}
    \norm{\Omega_{t, s + 1, \mu, \nu}}_{\op} \le \, N_{\max}^{t-s}.
\end{equation*}
Combining the above estimates, we get
\begin{align*}
    & \norm{\E \left[ \Omega_{t, s + 1, \mu, \nu}^\top \Omega_{t, s + 1, \mu, \nu} \right]}_{\op} \\
    & \le \,  \norm{\E \left[ \Omega_{t, s + 1, \mu, \nu}^{\top} \Omega_{t, s + 1, \mu, \nu} \mathbbm{1}_{\mathcal{E}_{t, s + 1, \mu, \nu}} \right]}_{\op} + \norm{\E \left[ \Omega_{t, s + 1, \mu, \nu}^{\top} \Omega_{t, s + 1, \mu, \nu} \mathbbm{1}_{\mathcal{E}_{t, s + 1, \mu, \nu}^c} \right]}_{\op} \\
    & \le \, \big\| \Omega_{t, s + 1, \mu, \nu}^{\ast}  \Omega_{t, s + 1, \mu, \nu}^{\ast} \big\|_{\op}  + 2 \gamma_1 \sum_{j = 1}^{t - s} N_{\max}^j + \gamma_1^2 \big( \sum_{j = 1}^{t - s} N_{\max}^j \big)^2 + N_{\max}^{2(t - s)}\, \P( \mathcal{E}_{t, s + 1, \mu, \nu}^c )\\
    & \le \, \Big\| \Omega_{t, s + 1, \mu, \nu}^{\ast}  \Omega_{t, s + 1, \mu, \nu}^{\ast} \Big\|_{\op}  + 2 \gamma_1 \sum_{j = 1}^{t - s} N_{\max}^j + \gamma_1^2 \big( \sum_{j = 1}^{t - s} N_{\max}^j \big)^2 + \, \delta \sum_{j = 1}^{t - s} N_{\max}^{j + 2(t - s)},
\end{align*}
which implies that
\begin{equation}
\label{eq:91}
    \Big\|{\E \left[ \Omega_{t, s + 1, \mu, \nu}^\top  \Omega_{t, s + 1, \mu, \nu} \right]}\Big\|_{\op}^{1/2} \le \, \Big\|{\Omega_{t, s + 1, \mu, \nu}^*}\Big\|_{\op} + \Big( \gamma_1 + \sqrt{2\gamma_1} + \sqrt{\delta} N_{\max}^{t - s} \Big) \sum_{j = 1}^{t - s} N_{\max}^j.
\end{equation}
Leveraging \cref{eq:91}, we conclude that for all $\mu \in \cM_l^{\rm nc}$:
\begin{align}\label{eq:concentration_result}
\begin{split}
    & \sum_{m \in \cM_l^{\rm nc}} \Big\| \E\big[  \Omega_{t + t_0, t + 1, \mu, m}^{\top} \Omega_{t + t_0, t + 1, \mu, m} \big] \Big\|_{\op}^{1/2} \\
    & \le \, \sum_{m \in \cM_l^{\rm nc}} \norm{ \Omega_{t+t_0, t+1, \mu, m}^*}_{\op} + K \Big( \gamma_1 + \sqrt{2\gamma_1} + \sqrt{\delta} N_{\max}^{t_0} \Big) \sum_{j = 1}^{t_0} N_{\max}^j.
\end{split}
\end{align}
%
By \cref{eq:concentration_result}, to upper bound $\sum_{m \in \cM_l^{\rm nc}} \| \E[ \Omega_{t + t_0, t + 1, \mu, m}^{\top} \Omega_{t + t_0, t + 1, \mu, m} ] \|_{\op}^{1/2}$ for $\mu \in \cM_l^{\rm nc}$, it suffices to bound
\begin{equation*}
    \sum_{m \in \cM_l^{\rm nc}} \norm{\Omega_{t+t_0, t+1, \mu, m}^*}_{\op}
\end{equation*}
for $\mu \in \cM_l^{\rm nc}$. 
To this end, we exploit the connection between $\Omega^*$ and the Markov chain $P$ defined in \cref{eq:markov_transition_def}.
%
By definition, $P_t$ is row-stochastic matrix, since $\sum_{\nu \in \cM} P_{t,\mu,\nu} = 1$ for all $\mu \in \cM$. Consequently, the sequence $\{P_t\}_{t \ge 0}$ defines a (possibly time-inhomogeneous) Markov chain.
Recall for $t \ge s \ge 0$ we have $J_{t, s + 1} = P_t P_{t - 1} \cdots P_{s + 1}$, and $J_{t,t+1} = I_K$.
For any $\mu \in \cM_l^{\rm nc}$, we have
\begin{align}
\label{eq:bar-Omega-star}
    \sum_{m \in \cM_l^{\rm nc}} \norm{\Omega_{t+t_0, t+1, \mu, m}^*}_{\op}  = \sum_{m \in \cM_l^{\rm nc}} J_{t+t_0, t+1, \mu, m} \leq 1.  
\end{align}
From \cref{eq:94}, we know that for all $t_0 \geq T_0$, we have
\begin{align*}
    \sup_{t \in \N_+, \, \nu_1 \in \cM_l^{\rm nc}} \sum_{\nu_2 \in \cM_l^{\rm nc}} J_{t + t_0, t + 1, \nu_1, \nu_2} \leq \omega \in (0, 1). 
\end{align*}
Combining the above upper bound with \cref{eq:concentration_result,eq:bar-Omega-star}, we conclude that under Assumption \ref{ass:concentration}, for all $2T_0 - 1 \geq t_0 \geq T_0$, 
\begin{align*}
       &  \sup_{\mu \in \cM_l^{\rm nc}} \left\{ \sum_{m \in \cM_l^{\rm nc}} \norm{\E \left[ \Omega_{t+t_0, t+1, \mu, m}^\top \Omega_{t+t_0, t+1, \mu, m} \right]}_{\op}^{1/2} \right\} \\
       & \le \, \omega + K \Big( \gamma_1 + \sqrt{2\gamma_1} + \sqrt{\delta} N_{\max}^{t_0} \Big) \sum_{j = 1}^{t_0} N_{\max}^j 
       \leq \alpha_0 \in (0, 1).
\end{align*}
As a consequence of the above upper bound and \cref{eq:supE-supE}, we know that for all $T_0 \leq t_0 \leq 2 T_0 - 1$, we have
\begin{equation*}
    \sup_{\mu \in \cM_l^{\rm nc}} {E_{t+t_0, s, \mu}} \le \, \alpha_0 \times \sup_{\mu \in \cM_l^{\rm nc}} {E_{t, s, \mu}},
\end{equation*}
which implies that for all $t \ge s + T_0$:
\begin{align}
\label{eq:96}
\begin{split}
    & \sup_{\mu \in \cM_l^{\rm nc}} {E_{t, s, \mu}} \le \,  \alpha_0^{\,\lfloor(t-s)/T_0\rfloor} \times \sup_{\mu \in \cM_l} E_{s, s, \mu} = \alpha_0^{\,\lfloor(t-s)/T_0\rfloor} \times\sum_{\nu \in \cM \backslash \cM_l^{\rm c}} \E\big[\, v_{s, \nu}^{\top} v_{s, \nu} \,\big]^{1/2} \\
    & \le \,  K \alpha_0^{\,\lfloor(t-s)/T_0\rfloor} \cdot \sup_{\nu \in \cM} \E\big[\, v_{s, \nu}^{\top} v_{s, \nu} \,\big]^{1/2},
\end{split}
\end{align}
where we recall that $v_{s, \nu}$ is from \cref{eq:v-t-mu}.
On the other hand, for $s \leq t \leq s + T_0 - 1$, there exists $C' > 0$ that depends only on $(\cG, \alpha)$, such that 
\begin{align}
\label{eq:97}
    \sup_{\mu \in \cM_l^{\rm nc}} {E_{t, s, \mu}} \leq C' \cdot \sup_{\nu \in \cM} \E\big[\, v_{s, \nu}^{\top} v_{s, \nu} \,\big]^{1/2}. 
\end{align}
Combining \cref{eq:three-lines,eq:96,eq:97}, we conclude that for all $\mu \in \cM_l^{\rm nc}$, there exists $C_{\cG} > 0$ that depends only on $(\cG, \alpha)$, such that 
\begin{align*}
   \E \left[ \|{\hbeta_{t, \mu} - \beta_{\ast}}\|_2^2 \right] \leq &  2\E \left[ \|{\hbeta_{t, \mu} - \hbeta_{0, \mu}}\|_2^2 \right] + 2\E \left[ \|{\hbeta_{0, \mu} - \beta_{\ast}}\|_2^2 \right] \\
   \leq & \Big( \sum_{s = 1}^t E_{t, s, \mu} \Big)^2 + 2\E\big[\, v_{0, \mu}^{\top} v_{0, \mu} \,\big]  \\ 
   \leq &  C_{\cG} \cdot \sup_{s \in \{0 \} \cup [t]} \sup_{\nu \in \cM} \,\E\big[\, v_{s, \nu}^{\top} v_{s, \nu} \,\big]. 
\end{align*}
By assumption $\sup_{t\geq 1}({\sup_{s \in \{0 \} \cup [t]} \sup_{\nu \in \cM} \,\E\big[\, v_{s, \nu}^{\top} v_{s, \nu} \,\big]}) / ({\E[v_{t, \mu}^Tv_{t, \mu}]}) < \infty$, hence completing the proof.

\subsection{Proof of \cref{thm:m_est_variance}}\label{sec:proof_m_variance}
We prove our conclusion by induction. For $t = 0$, the consistency and asymptotic normality of $\hat{\beta}_{0}$ follows from standard arguments in empirical process theory and our assumptions. 
See Appendix \ref{sec:consistency-normality} for a more rigorous treatment.
Furthermore, we have $\sqrt{n_0} (\hat{\beta}_0 - \hat{\beta}_{-1} ) \stackrel{d}{\to} \normal (0, \Sigma_0)$, where $\Sigma_0$ is defined in \cref{eq:def_Sigma_0}.

Now we proceed with the induction step. Assume that the conclusion already holds for rounds $0, 1, \cdots, t-1$, we first show that the $\hat{\beta}_{t, \mu}$'s are all consistent estimators of $\beta_*$. For $\mu \in \cM_u$, this is automatically true since $\hat{\beta}_{t, \mu} = \hat{\beta}_{t-1, \mu}$. For $\mu \in \cM_l$, note that by definition:
\begin{align*}
    \hat{\beta}_{t, \mu} = \, \arg \min_{\beta \in \mathcal{B} } \frac{1}{\sum_{\nu \in \Nin{\mu}} n_{t, \nu \to \mu}} \sum_{\nu \in \Nin{\mu}} L \left( \beta, Z_{t, \nu \to \mu} \right),
\end{align*}
where we use the shorthand
\begin{equation*}
    L \left( \beta, Z_{t, \nu \to \mu} \right) = \, \sum_{i=1}^{n_{t, \nu \to \mu}} L(\beta, z_{t, \nu \to \mu, i}).
\end{equation*}
Recall that $z_{t, \nu \to \mu, i} = \varphi(\hat\beta_{t - 1, \nu}, \varepsilon_{t, \nu \to \mu, i})$ for $\varepsilon_{t, \nu \to \mu, i} \sim_{i.i.d.} \mu_{\varepsilon}$. To establish consistency of $\hat{\beta}_{t, \mu}$, we invoke Theorem 3.2.3 (ii) in \cite{van1996weak}. Define
\begin{align*}
    &\hat{M}_{t, \mu} (\beta) = \,  \frac{1}{\sum_{\nu \in \Nin{\mu}} n_{t, \nu \to \mu}} \sum_{\nu \in \Nin{\mu}} L \left( \beta, Z_{t, \nu \to \mu} \right), \\
    &\tilde{M}_{t, \mu} (\beta) = \,  \frac{1}{\sum_{\nu \in \Nin{\mu}} n_{t, \nu \to \mu}} \sum_{\nu \in \Nin{\mu}} \E \left[ L \left( \beta, Z_{t, \nu \to \mu} \right) \Big\vert \hat{\beta}_{t-1, \nu} \right], \\
    &M (\beta) = \,  \E \left[ L (\beta, \varphi(\beta_*, \veps)) \right].
\end{align*}
By \cref{ass:well_specified}, $M (\beta)$ is continuous and uniquely minimized at $\beta_*$. 
It then suffices to show that (i) $\sup_{\beta \in \Omega} \vert \hat{M}_{t, \mu} (\beta) - M (\beta) \vert \stackrel{P}{\to} 0$ for any compact set $\Omega$; (ii) The sequence $\hat{\beta}_{t, \mu}$ is uniformly tight. 

We first prove claim (i). For any $\delta > 0$, define the event
\begin{equation*}
    \Xi_{t-1, \delta} = \, \left\{ \norm{\hat{\beta}_{t-1, \nu} - \beta_*}_2 \le \delta, \,\, \forall \nu \in \cM \right\}. 
\end{equation*}
Then, our induction hypothesis implies that $\P (\Xi_{t-1, \delta}) \to 1$ as the sample sizes tend to infinity. Conditioned on $\Xi_{t-1, \delta}$, we have
\begin{align*}
    & \sup_{\beta \in \Omega} \left\vert \hat{M}_{t, \mu} (\beta) - \tilde{M}_{t, \mu} (\beta) \right\vert \\
    & \le \, \frac{1}{\sum_{\nu \in \Nin{\mu}} n_{t, \nu \to \mu}} \sum_{\nu \in \Nin{\mu}} \sup_{ \substack{\beta \in \Omega \\ \| \hat{\beta}_{t-1, \nu} - \beta_* \|_2 \le \delta } } \left\vert L \left( \beta, Z_{t, \nu \to \mu} \right) - \E \left[ L \left( \beta, Z_{t, \nu \to \mu} \right) \Big\vert \hat{\beta}_{t-1, \nu} \right] \right\vert.
\end{align*}
By point 1 of \cref{ass:glivenko_cantelli}, the function class
\begin{equation*}
    \left\{ \varepsilon \mapsto L \left( \beta_1,\, \varphi (\beta_2, \, \varepsilon )\right) \big\vert \, \beta_1 \in \Omega, \, \| \beta_2 - \beta_* \|_2 \le \delta \right\}
\end{equation*}
is $\mu_{\veps}$-Glivenko-Cantelli. Therefore, we can apply the uniform law of large numbers to deduce that
\begin{equation*}
    \frac{1}{n_{t, \nu \to \mu}} \sup_{ \substack{\beta \in \Omega \\ \| \hat{\beta}_{t-1, \nu} - \beta_* \|_2 \le \delta } } \left\vert L \left( \beta, Z_{t, \nu \to \mu} \right) - \E \left[ L \left( \beta, Z_{t, \nu \to \mu} \right) \Big\vert \hat{\beta}_{t-1, \nu} \right] \right\vert \stackrel{P}{\to} 0
\end{equation*}
for all $\nu \in \Nin{\mu}$. This implies that
\begin{equation*}
    \sup_{\beta \in \Omega} \left\vert \hat{M}_{t, \mu} (\beta) - \tilde{M}_{t, \mu} (\beta) \right\vert \stackrel{P}{\to} 0
\end{equation*}
as the sample sizes tend to infinity. To finish the proof of (i), it remains to show that $\sup_{\beta \in \Omega} \vert \tilde{M}_{t, \mu} (\beta) - M (\beta) \vert \stackrel{P}{\to} 0$. 
Again, conditioning on $\Xi_{t-1,\delta}$, we obtain that
\begin{align*}
    \sup_{\beta \in \Omega} \vert \tilde{M}_{t, \mu} (\beta) - M (\beta) \vert \le \, \sup_{ \substack{\beta \in \Omega \\ \| \beta' - \beta_* \|_2 \le \delta } } \left\vert \E [L (\beta, \varphi(\beta', \veps))] - \E [L (\beta, \varphi(\beta_*, \veps))] \right\vert.
\end{align*}
By \cref{ass:well_specified}, we know that the mapping $(\beta_1, \beta_2) \mapsto \E [L (\beta_1, \varphi(\beta_2, \veps))]$ is continuous, and therefore uniformly continuous on any compact set. Consequently, for any $\delta' > 0$, we may choose $\delta$ sufficiently small such that
\begin{equation*}
    \sup_{ \substack{\beta \in \Omega \\ \| \beta' - \beta_* \|_2 \le \delta } } \left\vert \E [L (\beta, \varphi(\beta', \veps))] - \E [L (\beta, \varphi(\beta_*, \veps))] \right\vert \le \delta',
\end{equation*}
which further implies that
\begin{equation*}
    \P \left( \sup_{\beta \in \Omega} \vert \tilde{M}_{t, \mu} (\beta) - M (\beta) \vert \ge \delta' \right) \to 0.
\end{equation*}
Since $\delta'$ is arbitrary, we deduce that $\sup_{\beta \in \Omega} \vert \tilde{M}_{t, \mu} (\beta) - M (\beta) \vert \stackrel{P}{\to} 0$, which completes the proof of claim (i). 

As for claim (ii), note that \cref{ass:L-convex} implies that either $L$ is convex or $\mathcal{B}$ is compact, which ensures the uniform tightness of $\hat{\beta}_{t,\mu}$. 
It follows that $\hat{\beta}_{t,\mu}$ is a consistent estimator of $\beta_*$.

We next establish asymptotic normality of $\hat{\beta}_t$ and derive the recursive formula~\eqref{eq:asym_var_recursive} for computing its asymptotic variance. 
The first step is to obtain an asymptotic equivalent expression for $\sqrt{n_t} (\hat{\beta}_t - \hat{\beta}_{-1} )$. Using \cref{ass:L-convex} and the consistency of $\hat{\beta}_{t, \mu}$, we know that $\hat{\beta}_{t, \mu} \in \operatorname{int} \mathcal{B}$ with probability $1 - o(1)$. 
Therefore, $\hat{\beta}_{t, \mu}$ must satisfy the following first-order condition ($\nabla_{\beta_1}$ denotes the gradient with respect to the first argument):
\begin{equation*}
    \frac{1}{\sum_{\nu \in \Nin{\mu}} n_{t, \nu \to \mu}} \sum_{\nu \in \Nin{\mu}} \nabla_{\beta_1} L \left( \hat{\beta}_{t, \mu}, Z_{t, \nu \to \mu} \right) = \, 0,
\end{equation*}
which further implies that
\begin{align*}
    & - \frac{1}{\sum_{\nu \in \Nin{\mu}} n_{t, \nu \to \mu}} \sum_{\nu \in \Nin{\mu}} \nabla_{\beta_1} L \left( \beta_*, Z_{t, \nu \to \mu} \right) \\
    & = \,  \frac{1}{\sum_{\nu \in \Nin{\mu}} n_{t, \nu \to \mu}} \sum_{\nu \in \Nin{\mu}} \left( \nabla_{\beta_1} L \left( \hat{\beta}_{t, \mu}, Z_{t, \nu \to \mu} \right) - \nabla_{\beta_1} L \left( \beta_*, Z_{t, \nu \to \mu} \right) \right) \\
    & = \,  \left( \frac{1}{\sum_{\nu \in \Nin{\mu}} n_{t, \nu \to \mu}} \sum_{\nu \in \Nin{\mu}} \int_{0}^{1} \nabla_{\beta_1}^2 L \left( \alpha \hat{\beta}_{t, \mu} + (1 - \alpha) \beta_* , Z_{t, \nu \to \mu} \right) \d \alpha \right) \left( \hat{\beta}_{t, \mu} - \beta_* \right).
\end{align*}
As our goal is to show $\sqrt{n_t} (\hat{\beta}_{t} - \hat{\beta}_{-1}) \stackrel{d}{\to} \normal (0, \Sigma_{t})$, we can assume without loss of generality that $\| \hat{\beta}_{t, \mu} - \beta_* \|_2 \le 1$, which happens with high probability by consistency. Using point 1 of \cref{ass:glivenko_cantelli} and the uniform law of large numbers, it follows that
\begin{align*}
    \bigg\vert \, & \frac{1}{\sum_{\nu \in \Nin{\mu}} n_{t, \nu \to \mu}} \sum_{\nu \in \Nin{\mu}} \int_{0}^{1} \nabla_{\beta_1}^2 L \left( \alpha \hat{\beta}_{t, \mu} + (1 - \alpha) \beta_* , Z_{t, \nu \to \mu} \right) \d \alpha \\
    & - \frac{1}{\sum_{\nu \in \Nin{\mu}} n_{t, \nu \to \mu}} \sum_{\nu \in \Nin{\mu}} n_{t, \nu \to \mu} \int_{0}^{1} \E \left[ \nabla_{\beta_1}^2 L \left( \alpha \hat{\beta}_{t, \mu} + (1 - \alpha) \beta_*, \varphi(\hat{\beta}_{t-1, \nu}, \veps) \right) \right] \d \alpha \bigg\vert \stackrel{P}{\to} 0
\end{align*}
as the sample sizes tend to infinity, where the expectation above is taken with respect to $\veps \sim \mu_{\veps}$. 
Since $\hat{\beta}_{t, \mu}$ and the $\hat{\beta}_{t-1, \nu}$'s are consistent estimators of $\beta_*$, we deduce from \cref{ass:population_regularity} that for all $\nu \in \Nin{\mu}$,
\begin{equation*}
    \int_{0}^{1} \E \left[ \nabla_{\beta_1}^2 L \left( \alpha \hat{\beta}_{t, \mu} + (1 - \alpha) \beta_*, \varphi(\hat{\beta}_{t-1, \nu}, \veps) \right) \right] \d \alpha \stackrel{P}{\to} \E \left[ \nabla_{\beta_1}^2 L \left( \beta_*, \varphi( \beta_*, \veps) \right) \right] = \E \left[ \nabla_{\beta}^2 L (\beta, \varphi (\beta_*, \veps) ) \right] \Big\vert_{\beta  = \beta_*}
\end{equation*}
as the sample sizes tend to infinity. Further, \cref{ass:well_specified} implies that $\E [ \nabla_{\beta_1}^2 L ( \beta_*, \varphi( \beta_*, \veps) ) ]$ is invertible, we thus obtain that
\begin{align*}
    & \sqrt{n_t} \left( \hat{\beta}_{t, \mu} - \beta_* \right) \\
    = \, & - \E \left[ \nabla_{\beta_1}^2 L \left( \beta_*, \varphi( \beta_*, \veps) \right) \right]^{-1} \frac{\sqrt{n_t}}{\sum_{\nu \in \Nin{\mu}} n_{t, \nu \to \mu}} \sum_{\nu \in \Nin{\mu}} \nabla_{\beta_1} L \left( \beta_*, Z_{t, \nu \to \mu} \right) + o_P (1).
\end{align*}
Now for each $\nu \in \Nin{\mu}$, we analyze the term $\nabla_{\beta_1} L \left( \beta_*, Z_{t, \nu \to \mu} \right)$ in more details. Note that
\begin{align*}
    & \frac{1}{\sqrt{n_{t, \nu \to \mu}}} \nabla_{\beta_1} L \left( \beta_*, Z_{t, \nu \to \mu} \right) = \, \frac{1}{\sqrt{n_{t, \nu \to \mu}}} \sum_{i=1}^{n_{t, \nu \to \mu}} \nabla_{\beta_1} L(\beta_*, z_{t, \nu \to \mu, i}) \\
    = \, & \frac{1}{\sqrt{n_{t, \nu \to \mu}}} \sum_{i=1}^{n_{t, \nu \to \mu}} \nabla_{\beta_1} L( \beta_*, \varphi(\hat\beta_{t - 1, \nu}, \varepsilon_{t, \nu \to \mu, i}) ) \\
    = \, & \frac{1}{\sqrt{n_{t, \nu \to \mu}}}  \sum_{i=1}^{n_{t, \nu \to \mu}} \left( \nabla_{\beta_1} L \left( \beta_*, \varphi(\hat\beta_{t - 1, \nu}, \varepsilon_{t, \nu \to \mu, i}) \right) - \E \left[ \nabla_{\beta_1} L \left( \beta_*, \varphi(\hat{\beta}_{t-1, \nu}, \veps) \right) \right] \right) \\
    & + \sqrt{n_{t, \nu \to \mu}} \, \E \left[ \nabla_{\beta_1} L \left( \beta_*, \varphi(\hat{\beta}_{t-1, \nu}, \veps) \right) \right] := \mbox{(I)} + \mbox{(II)}. 
\end{align*}
For (I), we note that point 2 of \cref{ass:glivenko_cantelli} implies uniform CLT for $\hat{\beta}_{t-1, \nu}$ in a neighborhood of $\beta_*$. Since the limiting Gaussian process is continuous and $\hat{\beta}_{t-1, \nu} - \beta_* = o_P (1)$, we know that
\begin{align*}
	\mbox{(I)} = \, & \frac{1}{\sqrt{n_{t, \nu \to \mu}}} \sum_{i=1}^{n_{t, \nu \to \mu}} \left( \nabla_{\beta_1} L \left( \beta_*, \varphi ( \beta_*, \varepsilon_{t, \nu \to \mu, i}) \right) - \E \left[ \nabla_{\beta_1} L \left( \beta_*, \varphi ( \beta_*, \veps) \right) \right] \right) + o_P (1) \\
	= \, & \frac{1}{\sqrt{n_{t, \nu \to \mu}}}  \sum_{i=1}^{n_{t, \nu \to \mu}} \nabla_{\beta_1} L \left( \beta_*, \varphi ( \beta_*, \varepsilon_{t, \nu \to \mu, i}) \right) + o_P (1),
\end{align*}
where the last equality follows from \cref{ass:well_specified}: $\E [ \nabla_{\beta_1} L ( \beta_*, \varphi ( \beta_*, \veps) ) ] = 0$.
For (II), we invoke \cref{ass:population_regularity} together with the induction hypothesis that $\hat{\beta}_{t-1,\nu}$ is asymptotically normal to deduce that
\begin{align*}
	\mbox{(II)} = \, & \nabla_{\beta_2} \E \left[ \nabla_{\beta_1} L \left( \beta_*, \varphi(\beta_*, \veps) \right) \right]^\top \sqrt{n_{t, \nu \to \mu}} \left( \hat{\beta}_{t-1, \nu} - \beta_* \right) + o_P (1) \\
    \stackrel{(i)}{=} \, & - \nabla_{\beta_1} \E \left[ \nabla_{\beta_1} L \left( \beta_*, \varphi(\beta_*, \veps) \right) \right]^\top \sqrt{n_{t, \nu \to \mu}} \left( \hat{\beta}_{t-1, \nu} - \beta_* \right) + o_P (1) \\
	\stackrel{(ii)}{=} \, & - \E \left[ \nabla_{\beta_1}^2 L \left( \beta_*, \varphi(\beta_*, \veps) \right) \right] \sqrt{n_{t, \nu \to \mu}} \left( \hat{\beta}_{t-1, \nu} - \beta_* \right) + o_P (1),
\end{align*}
where $(i)$ follows from \cref{ass:well_specified}: for any $\beta \in \R^d$,
\begin{equation*}
    0 = \, \frac{\d}{\d \beta} \E \left[ \nabla_{\beta_1} L \left( \beta, \varphi(\beta, \veps) \right) \right] = \nabla_{\beta_1} \E \left[ \nabla_{\beta_1} L \left( \beta, \varphi(\beta, \veps) \right) \right] + \nabla_{\beta_2} \E \left[ \nabla_{\beta_1} L \left( \beta, \varphi(\beta, \veps) \right) \right],
\end{equation*}
and $(ii)$ follows from \cref{ass:population_regularity}. Combining the above estimates, and denoting
\begin{equation*}
	H_* = \, \E \left[ \nabla_{\beta_1}^2 L \left( \beta_*, \varphi(\beta_*, \veps) \right) \right],
\end{equation*}
we obtain that
\begin{align*}
	\sqrt{n_t} \left( \hat{\beta}_{t, \mu} - \beta_* \right) = \, & - \frac{\sqrt{n_t}}{\sum_{\nu \in \Nin{\mu}} n_{t, \nu \to \mu}} \sum_{\nu \in \Nin{\mu}} H_*^{-1} \nabla_{\beta_1} L \left( \beta_*, Z_{t, \nu \to \mu} \right) + o_P (1) \\
	= \, & - \frac{\sqrt{n_t}}{\sum_{\nu \in \Nin{\mu}} n_{t, \nu \to \mu}} \sum_{\nu \in \Nin{\mu}} \sum_{i=1}^{n_{t, \nu \to \mu}} H_*^{-1} \nabla_{\beta_1} L \left( \beta_*, \varphi ( \beta_*, \varepsilon_{t, \nu \to \mu, i}) \right) \\
	& + \frac{\sqrt{n_t}}{\sum_{\nu \in \Nin{\mu}} n_{t, \nu \to \mu}} \sum_{\nu \in \Nin{\mu}} n_{t, \nu \to \mu} \left( \hat{\beta}_{t-1, \nu} - \beta_* \right) + o_P (1).
\end{align*}
For $\mu \in \cM_{l}$, define
\begin{equation*}
	u_{t, \mu} = \, - \frac{1}{\sum_{\nu \in \Nin{\mu}} n_{t, \nu \to \mu}} \sum_{\nu \in \Nin{\mu}} \sum_{i=1}^{n_{t, \nu \to \mu}} H_*^{-1} \nabla_{\beta_1} L \left( \beta_*, \varphi ( \beta_*, \varepsilon_{t, \nu \to \mu, i}) \right)
\end{equation*}
and recall the definition of $\bar{P}_t$ from \cref{eq:bar-P}. Using \cref{ass:limit_proportion}, the above expression for $\sqrt{n_t} ( \hat{\beta}_{t, \mu} - \beta_* )$ can be further simplified to
\begin{equation}\label{eq:simplified_recursion}
	\sqrt{n_t} \left( \hat{\beta}_{t, \mu} - \beta_* \right) = \, \sqrt{n_t} \sum_{\nu \in \cM} \bar{P}_{t, \mu, \nu} \left( \hat{\beta}_{t-1, \nu} - \beta_* \right) + \sqrt{n_t} \, u_{t, \mu} + o_P (1).
\end{equation}
Since $\hat{\beta}_{t, \mu} = \hat{\beta}_{t-1, \mu}$ for $\mu \in \cM_u$, as long as we define
\begin{equation*}
	u_{t, \mu} = \, 0, \quad \forall \mu \in \cM_u,
\end{equation*}
\cref{eq:simplified_recursion} will still be valid for any $\mu \in \cM_u$. It finally follows that
\begin{equation*}
	\sqrt{n_t} \left( \hat{\beta}_t - \hat{\beta}_{-1} \right) = \, (\bar{P}_t \otimes I_p) \sqrt{n_t} \left( \hat{\beta}_{t-1} - \hat{\beta}_{-1} \right) + \sqrt{n_t} u_t + o_P (1).
\end{equation*}
According to our data generating process, $u_t$ is asymptotically normal (see Appendix \ref{sec:consistency-normality}) and independent of $\hat{\beta}_{t-1}$. In particular, we have that $\sqrt{n_t} u_t \stackrel{d}{\to} \normal (0, V_t)$ as the sample sizes tend to infinity, where we recall that $V_t$ is defined according to~\cref{eq:def_V_t}. Further, since $n_t / n_{t-1} \to b_{t, t-1}$ by \cref{ass:limit_proportion}, we deduce from our induction hypothesis that $\sqrt{n_t} ( \hat{\beta}_{t-1} - \hat{\beta}_{-1} ) \stackrel{d}{\to} b_{t, t-1} \Sigma_{t-1}$. We finally conclude that $\hat{\beta}_t$ is also asymptotically normal, in the sense that
\begin{equation*}
    \sqrt{n_t} \left( \hat{\beta}_t - \hat{\beta}_{-1} \right) \stackrel{d}{\to} \normal (0, \Sigma_t), \quad \Sigma_t = \, b_{t, t-1} (\bar{P}_t \otimes I_d) \Sigma_{t-1} (\bar{P}_t \otimes I_d) ^\top + V_t.
\end{equation*}
This closes the induction argument and completes the proof of \cref{thm:m_est_variance}.

\subsection{Proof of \cref{thm:network_collapse}}\label{sec:proof_m_collapse}

To facilitate analysis, we first establish some useful lower and upper bounds on $\Sigma_{T, \mu, \mu}$ in the lemma below.

\begin{lem}\label{thm:m_est_collapse}
    For $t \ge s \ge 1$, define $J_{t, s} = \bar{P}_t \bar{P}_{t-1} \cdots \bar{P}_s$, with the convention that $J_{t, t + 1} = I_{K}$. 
    Recall that $V_*$ is defined in \cref{eq:defn_V_star}, and that $\bar{p}_{0, \mu}$, $\bar{p}_{t, \nu \to \mu}$ and $b_{t, s}$ are defined in \cref{ass:limit_proportion}. Then, the following holds:
    \begin{enumerate}
        \item [(a)] For any $\mu \in \cM_l$ and $T \in \mathbb{N}$, we have
        \begin{equation*}
            \Sigma_{T, \mu, \mu} \succeq \, \left( \sum_{t=1}^{T} \frac{ b_{T, t} (\sum_{\nu \in \cM_l} J_{T, t+1, \mu, \nu})^2}{\sum_{\nu \in \cM_l} \sum_{m \in \Nin{\nu}} \bar{p}_{t, m \to \nu}} + \frac{b_{T, 0}}{\sum_{\nu \in \cM} \bar{p}_{0, \nu}} \right) V_*.
        \end{equation*}
        
        \item [$(b)$] For any $\mu \in \cM_l$ and $T \in \mathbb{N}$, we have
        \begin{equation*}
            \Sigma_{T, \mu, \mu} \preceq \, \left( \sum_{t=1}^{T} \frac{b_{T,t} (\sum_{\nu \in \cM_l} J_{T, t+1, \mu, \nu})^2}{\inf_{\nu \in \cM_l} \sum_{m \in \Nin{\nu}} \bar{p}_{t, m \to \nu}} + \frac{b_{T, 0}}{\inf_{\nu \in \cM} \bar{p}_{0, \nu}} \right) V_*.
        \end{equation*}
    \end{enumerate}
\end{lem}

\begin{proof}[Proof of \cref{thm:m_est_collapse}]
    We prove \cref{thm:m_est_collapse} in \cref{sec:proof_lem_m_collapse}.
\end{proof}

We now continue the proof of \cref{thm:network_collapse}. Under \cref{ass:limit_proportion}, our lower and upper bounds in \cref{thm:m_est_collapse} implies that, for all $\mu \in \cM_l$ and $T \in \mathbb{N}$:
\begin{equation*}
    \Sigma_{T, \mu, \mu} \simeq \left( \sum_{t=1}^{T}  \left( \sum_{\nu \in \cM_l} J_{T, t+1, \mu, \nu} \right)^2 + 1 \right) V_* \implies \frac{\Tr (\Sigma_{T, \mu, \mu})}{\Tr (V_*)} \simeq \, \sum_{t=1}^{T}  \left( \sum_{\nu \in \cM_l} J_{T, t+1, \mu, \nu} \right)^2 + 1,
\end{equation*}
where $\simeq$ hides constants that are independent of $T$. Therefore, a model $\mu \in \cM_l$ collapses as $T \to \infty$ if and only if
\begin{equation*}
    \limsup_{T \to \infty} \sum_{t=1}^{T}  \left( \sum_{\nu \in \cM_l} J_{T, t+1, \mu, \nu} \right)^2 = \infty.
\end{equation*}
We first show that $\forall \mu \in \cM_l^{\infty}$ will collapse. By definition of $\cM_l^{\infty}$, we know that for all $t \ge s$ and $\nu \in \cM_u$, $J_{t, s, \mu, \nu} = 0$. Therefore, for all $t \ge s$ we have
    \begin{equation*}
        \sum_{\nu \in \cM_l} J_{t, s, \mu, \nu} = 1,
    \end{equation*}
    and hence $\sum_{t=1}^{T} ( \sum_{\nu \in \cM_l} J_{T, t+1, \mu, \nu} )^2 = T \to \infty$ as $T \to \infty$.

We next consider $\mu \in \cM_l \backslash \cM_l^{\infty}$. If $\mu \in \cM_l^{\rm c}$, i.e., there is a directed path from some $\nu' \in \cM_{l}^{\infty}$ to $\mu$, then \cref{ass:limit_proportion} implies that there exists $\veps > 0$ and $T_0 \in \mathbb{N}$, such that
    \begin{equation*}
        J_{t+T_0, t, \mu, \nu'} \ge \veps, \,\, \forall t \in \mathbb{N}.
    \end{equation*}
    This further implies that as long as $t - s \ge T_0$, we have
    \begin{align*}
        \sum_{\nu \in \cM_l} J_{t, s, \mu, \nu} \ge \, \sum_{\nu \in \cM_l} J_{t, t-T_0, \mu, \nu'} J_{t-T_0 - 1, s, \nu', \nu} = J_{t, t-T_0, \mu, \nu'} \ge \veps.
    \end{align*}
    Therefore, as $T \to \infty$,
    \begin{equation*}
        \sum_{t=1}^{T}  \left( \sum_{\nu \in \cM_l} J_{T, t+1, \mu, \nu} \right)^2 \ge \, (T - T_0) \veps^2 \to \infty,
    \end{equation*}
    which means that model $\mu$ will collapse. This proves part (a) of \cref{thm:network_collapse}.

    Finally, we consider $\mu \in \cM_l^{\rm nc}$, i.e., $\mu \notin \cM_l^{\infty}$ and there is no directed path from $\cM_l^{\infty}$ to $\mu$. By definition, there is a directed path from $\cM_u$ to $\mu$, meaning that there exists $\veps(\mu) > 0$ and $T_0 (\mu) \in \mathbb{N}$, such that
    \begin{equation*}
        \sum_{\nu \in \cM_u} J_{t+T_0 (\mu), t, \mu, \nu} \ge \veps (\mu) \Longleftrightarrow \sum_{\nu \in \cM_l} J_{t+T_0 (\mu), t, \mu, \nu} \le 1 - \veps (\mu)
    \end{equation*}
    for all $t \in \mathbb{N}$.
    Since $J_{t, s, \nu, \nu} = 1$ for all $\nu \in \cM_u$, the above inequality remains true if $T_0(\mu)$ is replaced by any $T \ge T_0 (\mu)$. Of course, the above inequality also holds for all $\mu \in \cM_u$, so it actually holds for all $\mu \notin \cM_l^{\rm c}$. Let us now define
    \begin{equation*}
        \veps = \min_{\mu \notin \cM_l^{\rm c}} \veps (\mu), \quad T_0 = \max_{\mu \notin \cM_l^{\rm c}} T_0 (\mu).
    \end{equation*}
    We deduce that for all $\mu \notin \cM_l^{\rm c}$ and $t - s \ge T_0$:
    \begin{equation*}
        \sum_{\nu \in \cM_l} J_{t, s, \mu, \nu} \le 1 - \veps.
    \end{equation*}
    Using this inequality, we know that for all $\mu \notin \cM_l^{\rm c}$ and $t - s \ge 2 T_0 + 1$,
    \begin{align*}
        \sum_{\nu \in \cM_l} J_{t, s, \mu, \nu} = \, &  \sum_{\nu' \notin \cM_l^{\rm c} \cup \cM_u} J_{t, t-T_0, \mu, \nu'} \sum_{\nu \in \cM_l} J_{t-T_0-1, s, \nu', \nu} \\
        \le \, & (1 - \veps) \sum_{\nu' \notin \cM_l^{\rm c} \cup \cM_u} J_{t, t-T_0, \mu, \nu'} \le (1 - \veps)^2.
    \end{align*}
    Iterating further, we can use induction to show that for all $\mu \notin \cM_l^{\rm c}$ and $t - s \ge k T_0 + k-1$,
    \begin{equation*}
        \sum_{\nu \in \cM_l} J_{t, s, \mu, \nu} \le (1 - \veps)^k.
    \end{equation*}
    This immediately implies that
    \begin{equation*}
        \limsup_{T \to \infty} \sum_{t=1}^{T}  \left( \sum_{\nu \in \cM_l} J_{T, t+1, \mu, \nu} \right)^2 < \infty,
    \end{equation*}
    so that model $\mu$ will not collapse. This proves part (b) of \cref{thm:network_collapse}.

\subsection{Proof of \cref{thm:glm_example}}\label{sec:proof_glm_example}
To begin with, we state some general conditions under which Assumptions~\ref{ass:well_specified}--\ref{ass:population_regularity} are satisfied.
In this proof, we assume that the response variable $y$ can be represented as follows:
\begin{equation*}
    y = \, \varphi (\beta_*^\top x, \veps).
\end{equation*}
Note that, with a slight abuse of notation, we still use $\varphi$ to denote the link function and $\mu_{\varepsilon}$ to denote the joint distribution of $(x, \varepsilon)$. 
The above representation of $y$ yields the following expression for the loss function: 
\begin{equation*}
    L \left( \beta_1, \varphi (\beta_2^\top x, \veps) \right) = \, - \varphi (\beta_2^\top x, \veps) \cdot \beta_1^\top x + A (\beta_1^\top x).
\end{equation*}
\begin{prop}\label{prop:verify_glm_assumption}
    For a GLM as specified in \cref{eq:glm_specification}, assume that
    \begin{itemize}
        \item [(i)] $\E [ \| x \|_2 ] < \infty$. Further, for any compact set $\Omega$, 
        $$
            \E \Big[ \sup_{\beta \in \Omega} \| x \|_2 \vert \varphi (\beta^\top x, \veps) \vert \Big] < \infty, \quad \E \Big[ \sup_{\beta \in \Omega} \vert A (\beta^\top x) \vert \Big] < \infty.
        $$
        \item [(ii)] The function classes $ \{(x, \varepsilon) \mapsto \varphi(\beta^\top x, \veps): \beta \in \Omega\}$, $\{x \mapsto A (\beta^\top x): \beta \in \Omega\}$ and $\{x \mapsto A'' (\beta^\top x) x x^\top: \beta \in \Omega \}$ are all $\mu_{\veps}$-Glivenko-Cantelli for any compact set $\Omega$. 
        \item [(iii)] The function class $\{(x, \varepsilon) \mapsto  \varphi (\beta^\top x, \veps) x: \beta \in \Omega\}$ is $\mu_{\veps}$-Donsker for any compact set $\Omega$. 
        \item [(iv)] For all $\beta \in \R^d$, both $\E [A'(\beta^\top x)^2 x x^\top]$ and $\E [A'' (\beta^\top x) x x^\top]$ are finite. Further, the mapping $\beta \mapsto \E [A'' (\beta^\top x) x x^\top]$ is continuous.
    \end{itemize}
    Then, Assumptions~\ref{ass:well_specified}-\ref{ass:population_regularity} are satisfied.
\end{prop}
\begin{proof}
    We prove \cref{prop:verify_glm_assumption} in \cref{sec:proof_glm_assumption}.
\end{proof}

We next use \cref{prop:verify_glm_assumption} to show that the three GLMs (linear regression, logistic regression and Poisson regression) all satisfy Assumptions~\ref{ass:well_specified}-\ref{ass:population_regularity}. Throughout, for a compact set $\Omega \in \R^d$, let $r_\Omega > 0$ be such that $\Omega \subset \mathsf{B}(0, r_\Omega)$.

\paragraph{Linear regression.} In this setting, we have $A(\xi) = \xi^2 / 2$ and $\varphi (\beta^\top x, \veps) = \beta^\top x + \veps$ with $\E [\veps] = 0$, $\E [\veps^2] < \infty$. 
We next show that Assumptions~\ref{ass:well_specified}-\ref{ass:population_regularity} are satisfied, provided that $\E [\norm{x}_2^4] < \infty$. 
To this end, it suffices to verify conditions (i)-(iv) in \cref{prop:verify_glm_assumption}.
Condition (i) can be verified by noting that
\begin{equation*}
    \sup_{\beta \in \Omega} \| x \|_2 \vert \varphi (\beta^\top x, \veps) \vert \le r_\Omega \norm{x}_2^2 + \norm{x}_2 \vert \veps \vert, \quad \sup_{\beta \in \Omega} \vert A (\beta^\top x) \vert \le \frac{1}{2} r_\Omega^2 \norm{x}_2^2.
\end{equation*}
Condition (ii) can be verified by applying \cref{lem:preserve_gc} to the function classes $\{(x, \veps) \mapsto \beta^\top x \vert \beta \in \Omega \}$ and $\{(x, \veps) \mapsto \veps \}$, and noting that $\E [\norm{x}_2^4]$ is finite. Similarly, (iii) follows from \cref{lem:preserve_donsker}, the compactness of $\Omega$, and our assumption $\E [\norm{x}_2^4] < \infty$. Finally, (iv) can be verified by direct calculation.

\paragraph{Logistic regression.} In this setting, we have $A(\xi) = \log (1 + e^{\xi})$, and 
$$
\varphi (\beta^\top x, \veps) = \bone \{\veps \le A' (\beta^\top x) \} = \bone \{ \veps \le e^{\beta^\top x} / (1 + e^{\beta^\top x} ) \},
$$ 
where $\veps \sim \Unif [0, 1]$. 
We show that Assumptions~\ref{ass:well_specified}-\ref{ass:population_regularity} are satisfied, provided that $\E [\| x \|_2^3] < \infty$. 
To this end, it suffices to verify conditions (i)-(iv) in the statement of \cref{prop:verify_glm_assumption}. Condition (i) can be easily verified, since $\varphi$ is uniformly bounded by $1$ and $A(\xi) \le 1 + \vert \xi \vert$ for all $\xi \in \R$. To verify condition (ii), we first note that the function class $\{ (x, \veps) \mapsto \beta^\top x : \beta \in \Omega \}$ is $\mu_{\veps}$-Glivenko-Cantelli since $\E [\| x \|_2^3] < \infty$, and $A'' (\xi) \in [0, 1/4]$ for all $\xi \in \R$. We can then invoke \cref{lem:preserve_gc} to show that the function classes
$$
\{ (x, \veps) \mapsto A (\beta^\top x): \beta \in \Omega \}, \quad \{ (x, \veps) \mapsto A'' (\beta^\top x) x x^\top: \beta \in \Omega \}, \quad \{ (x, \veps) \mapsto \varphi(\beta^\top x, \veps):  \beta \in \Omega \}
$$ 
are $\mu_{\veps}$-Glivenko-Cantelli. 
To verify condition (iii), we use \cref{lem:indicator_donsker} to show that the function class 
$$
\{ (x, \veps) \mapsto \varphi(\beta^\top x, \veps) x:  \beta \in \Omega \}
$$ 
is $\mu_{\veps}$-Donsker. 
To apply \cref{lem:indicator_donsker}, we set $f_1(t) = A'(t)$ and $f_k(t) = 0$ for $k \geq 2$. We also set $L(t) = 1$, $f(t) = f_1(t)$ and $g(x) = x$.  
Since $A' (z) \in [0, 1]$ and $A''(z) \in [0, 1/4]$ for all $z \in \R$, conditions (i) and (ii) in \cref{lem:indicator_donsker} are automatically satisfied. 
Further, condition (iii) in \cref{lem:indicator_donsker} is verified by our assumption $\E [\| x \|_2^3] < \infty$. Finally, condition (iv) can be established using the boundedness of $A'$ and $A''$ and the dominated convergence theorem.

\paragraph{Poisson regression.} In this setting, we have $A(\xi) = e^{\xi}$, and $\varphi (\beta^\top x, \veps) = \sum_{k=1}^{\infty} \bone_{\veps \le f_k (\beta^\top x)}$, where $\veps \sim \Unif [0, 1]$, and
\begin{equation*}
    f_k (z) := \, \sum_{m=k}^{\infty} \frac{ \exp(m z - e^z) }{m!} = \P \left( \operatorname{Poisson} (e^z) \ge k \right).
\end{equation*}
We show that Assumptions~\ref{ass:well_specified}-\ref{ass:population_regularity} are satisfied, provided that $\E [\exp (R \| x \|_2)] < \infty$ for any $R > 0$. 
To this end, it suffices to verify conditions (i)-(iv) in \cref{prop:verify_glm_assumption}. 
For condition (i), note that our assumption already implies $\E[\|x\|_2] < \infty$.
Further, by the  definition of $\varphi$ and the monotonicity of the $f_k$'s, we have
\begin{align*}
    &\E \Big[ \sup_{\beta \in \Omega} \| x \|_2 \vert \varphi (\beta^\top x, \veps) \vert \Big] \le \,  \E \Big[ \| x \|_2 \sum_{k=1}^{\infty} f_k \big( r_\Omega \| x \|_2 \big) \Big] = \E \Big[ \| x \|_2 \exp \big( r_\Omega \| x \|_2 \big) \Big] < \infty, \\
    & \E \Big[ \sup_{\beta \in \Omega} \vert A (\beta^\top x) \vert \Big] \le \,  \E \Big[ \exp \big( r_\Omega \| x \|_2 \big) \Big] < \infty,
\end{align*}
where $r_{\Omega} = \sup\{\| \beta \|_2: \beta \in \Omega\}$. 
Condition (iv) follows immediately from our moment assumption.
To verify conditions (ii) and (iii), we first use the same argument as that in logistic regression to establish that the function classes
$$\{ (x, \veps) \mapsto A (\beta^\top x): \beta \in \Omega \}, \quad \{ (x, \veps) \mapsto A'' (\beta^\top x) x x^\top: \beta \in \Omega \}
$$ 
are $\mu_{\veps}$-Glivenko-Cantelli. It now remains to show that the function classes
$$
\{ (x, \veps) \mapsto \varphi(\beta^\top x, \veps): \beta \in \Omega \}, \quad \{ (x, \veps) \mapsto \varphi(\beta^\top x, \veps) x: \beta \in \Omega \}
$$ are $\mu_{\veps}$-Donsker, since any $\mu_{\veps}$-Donsker function class is also $\mu_{\veps}$-Glivenko-Cantelli. To this end, we invoke \cref{lem:indicator_donsker}, with $g(x) = 1$ or $x$. 
We next show that conditions (i)-(iii) in the statement of \cref{lem:indicator_donsker} are satisfied. 
To this end, we set $L(t) = L$ and $f(t) = e^{3t} + 3 e^{2t} + e^{t}$ for an absolute constant $L > 0$ (to be determined later). 
This already gives $\E [ g(x)^2 \| x \|_2 L ( R \| x \|_2 ) ] < \infty$ and $\E[ g(x)^2 f ( R \| x \|_2 ) ] < \infty$ for any $R > 0$, completing the proof of (iii) in \cref{lem:indicator_donsker}.
To establish upper bounds on $f_k'$, note that by direct calculation:
\begin{equation*}
    f_k' (z) = \, \frac{ \exp(k z - e^z) }{(k-1)!} \le \frac{1}{(k-1)!} \sup_{\lambda > 0} \lambda^k e^{- \lambda} \stackrel{(i)}{=} \frac{(k / e)^k}{(k-1)!} \stackrel{(ii)}{\le} L \sqrt{k},
\end{equation*}
where $(i)$ follows from direct computation, and $(ii)$ is due to Stirling's formula. Taking $c_1 = -1/2$ then verifies condition (i) in \cref{lem:indicator_donsker}. To establish upper bounds on $f_k$, we resort to Markov inequality:
\begin{equation*}
    f_k (z) = \, \P \left( \operatorname{Poisson} (e^z) \ge k \right) \le k^{-3} \E \big[ \operatorname{Poisson} (e^z)^3 \big] = k^{-3} (e^{3z} + 3 e^{2 z} + e^{z}) = k^{-3} f(z),
\end{equation*}
we can then take $c_2 = 3$. This concludes the proof.

\section{Technical lemmas for the linear regression setting}

In this section, we provide proofs of the technical lemmas supporting our linear regression results.

\subsection{Proof of \cref{lemma:data-transmission}}
\label{proof:lemma:data-transmission}
By \cref{assumption:data-distribution}, we know that $(X_{s, \nu \to \mu}, \veps_{s, \mu \to \nu})$ is a subset of $(X_{s, \nu}, \veps_{s, \nu})$. Therefore, for any fixed $s \in [t]$ and $\nu \in \cM$, there exists a partition of $(X_{s, \nu}, \veps_{s, \nu})$:
\begin{equation*}
    (X_{s, \nu}, \veps_{s, \nu}) = \, \cup_{j=1}^{N_{s, \nu}} (X_{s, \nu, j}, \veps_{s, \nu, j}),
\end{equation*}
such that $N_{s, \nu} \le 2^{\vert \cM \vert} = 2^K$, and for each $\mu \in \cM$, there exists a subset $A_{s, \nu \to \mu} \subset [N_{s, \nu}]$, such that
\begin{equation*}
    (X_{s, \nu \to \mu}, \veps_{s, \mu \to \nu}) = \, \cup_{j \in A_{s, \nu \to \mu}} (X_{s, \nu, j}, \veps_{s, \nu, j}).
\end{equation*}
(If $(\nu, \mu) \notin \cE$, one can simply take $A_{s, \nu \to \mu} = \emptyset$.) We can therefore define
\begin{equation*}
    \{ (Z_{t, i}, \zeta_{t, i}): 1 \le i \le N_t \} := \{ (X_{s, \nu, j}, \veps_{s, \nu, j}): s \in [t], \, \nu \in \cM, \, j \in A_{s, \nu \to \mu} \mbox{ for some } (\nu, \mu) \in \cE \}.
\end{equation*}
This proves part 1. To show part 2, note that for any fixed pair $(s, \nu)$, we have $N_{s, \nu} \le 2^K$. Hence, $N_t \le C t$, where the constant $C = K 2^K$ does not depend on $t$. Finally, by our construction, we know that $\zeta_t = (\zeta_{t, 1}^{\top}, \, \zeta_{t, 2}^{\top}, \, \cdots, \zeta_{t, N_t}^{\top})^{\top}$ is a subvector of $\veps_t$, which naturally implies that $\cov (\zeta_t) \succeq \rho_2 I$ by \cref{assumption:noise-vectors}. This completes the proof of this lemma.


\subsection{Proof of \cref{lemma:hat_beta_expression}}
\label{proof:lemma:hat_beta_expression}
By \cref{eq:hat_beta_expression}, we have for all $\mu \in \cM_l$:
\begin{align*}
    & \hbeta_{t, \mu} - \hbeta_{0, \mu} = \, \sum_{s=1}^{t} \sum_{\nu \in \cM} \Omega_{t, s+1, \mu, \nu} v_{s, \nu} \\
    &\stackrel{(i)}{=} \,  \sum_{s=1}^{t} \sum_{\nu \in \cM_l} \Omega_{t, s+1, \mu, \nu} \Big( \sum_{m \in \Nin{\nu}} X_{s, m \to \nu}^{\top} X_{s, m \to \nu} \Big)^{-1} \sum_{m \in \Nin{\nu}} X_{s, m \to \nu}^{\top} \eps_{s, m \to \nu} \\
    &= \,  \sum_{s=1}^{t} \sum_{\nu \in \cM_l} \sum_{m \in \Nin{\nu}} \Omega_{t, s+1, \mu, \nu} \cT_{s, \nu, m} (X_{s, m \to \nu}^{\top} X_{s, m \to \nu })^{-1} X_{s, m \to \nu}^{\top} \eps_{s, m \to \nu} \\
    &= \,  \sum_{s=1}^{t} \sum_{\nu \in \cM_l} \sum_{m \in \Nin{\nu}} \Omega_{t, s+1, \mu, \nu} \cT_{s, \nu, m} (X_{s, m \to \nu}^{\top} X_{s, m \to \nu })^{-1} \sum_{h = 1}^{h_{s, m \to \nu}} X_{s, m \to \nu, h}^{\top} \eps_{s, m \to \nu, h} \\
    &\stackrel{(ii)}{=} \,  \sum_{i=1}^{N_t} \sum_{(s, m \to \nu, h) \in \cP_{t, i}} \Omega_{t, s+1, \mu, \nu} \cT_{s, \nu, m} (X_{s, m \to \nu}^{\top} X_{s, m \to \nu })^{-1} Z_{t, i}^\top \zeta_{t, i} \\
    &= \,  \sum_{i = 1}^{N_t} A_{t, i, \mu} (Z_{t, i}^{\top} Z_{t, i})^{\dagger} Z_{t, i}^{\top} \zeta_{t, i},
\end{align*}
where $(i)$ follows from the definition of $v_{s, \nu}$, $(ii)$ follows from \cref{lemma:data-transmission} and \ref{lemma:Pti}, and the last line is due to the definition of $A_{t, i, \mu}$, and the identity $Z_{t, i}^{\top} Z_{t, i} (Z_{t, i}^{\top} Z_{t, i})^{\dagger} Z_{t, i}^{\top} = Z_{t, i}^{\top}$. This completes the proof.

\subsection{Proof of \cref{lemma:sumA}}
\label{proof:lemma:sumA}
By definition, we have
\begin{align*}
    \sum_{i = 1}^{N_t} A_{t, i, \mu} = \, & \sum_{i = 1}^{N_t} \sum_{(s, m \to \nu, h) \in \cP_{t, i}} \Omega_{t, s + 1, \mu, \nu} \cT_{s, \nu, m} (X_{s, m \to \nu}^{\top} X_{s, m \to \nu })^{-1} Z_{t, i}^{\top} Z_{t, i} \\
    = \, & \sum_{s=1}^{t} \sum_{\nu \in \cM_l} \sum_{m \in \Nin{\nu}} \Omega_{t, s+1, \mu, \nu} \cT_{s, \nu, m} (X_{s, m \to \nu}^{\top} X_{s, m \to \nu })^{-1} \sum_{h = 1}^{h_{s, m \to \nu}} X_{s, m \to \nu, h}^{\top} X_{s, m \to \nu, h} \\
    = \, & \sum_{s=1}^{t} \sum_{\nu \in \cM_l} \sum_{m \in \Nin{\nu}} \Omega_{t, s+1, \mu, \nu} \cT_{s, \nu, m} (X_{s, m \to \nu}^{\top} X_{s, m \to \nu })^{-1} (X_{s, m \to \nu}^{\top} X_{s, m \to \nu }) \\
    = \, & \sum_{s=1}^{t} \sum_{\nu \in \cM_l} \sum_{m \in \cM} \Omega_{t, s+1, \mu, \nu} \cT_{s, \nu, m},
\end{align*}
where the last line follows from the fact that $\cT_{s, \nu, m} = 0_{d \times d}$ if $m \notin \Nin{\nu}$. This completes the proof.

\subsection{Proof of \cref{lemma:sum-decay}}
\label{proof:lemma:sum-decay}

In this section we prove \cref{lemma:sum-decay}. 
The first inequality follows immediately from triangle inequality:  
\begin{align}
	\label{eq:path-small-1}
		& \Big \| \sum_{v_1 \to v_2 \to \cdots \to v_{x} \to \mu \in \cL_{x}^{\ast}} \E\big[ \cT_{t,  \mu, v_{x}} \big] \E \big[ \cT_{t - 1, v_{x}, v_{x - 1}}\big] \cdots \E \big[\cT_{t - x + 1, v_2, v_1} \big]\Big\|_{\op} \nonumber \\
		& \leq \sum_{v_1 \to v_2 \to \cdots \to v_{x} \to \mu \in \cL_{x}^{\ast}} \big\| \E[ \cT_{t, \mu,  v_{x}}]\big\|_{\op} \, \big\| \E[\cT_{t - 1, v_{x}, v_{x - 1}}] \big\|_{\op} \cdots \big\| \E[ \cT_{t - x + 1, v_2,  v_1}] \big\|_{\op}.  
	\end{align}
	%
By Assumption \ref{ass:concentration}, for all $s \in \N_+$ and $(\nu_1, \nu_2) \in \cE$, 
the two matrices $\E[\cT_{s, \nu_2, \nu_1} ]$ and $p_{s, \nu_1 \to \nu_2} I_d$ are close, and we provide an upper bound for their difference:  
\begin{align}
\label{eq:cT-op-concentrate}
\begin{split}
	 & \left\| \E\big[\cT_{s, \nu_2, \nu_1} \big] - p_{s, \nu_1 \to \nu_2} I_d \right\|_{\op}  \\
	 & \overset{(i)}{\leq} \E\left[ \| \cT_{s, \nu_2, \nu_1} - p_{s, \nu_1 \to \nu_2} I_d  \|_{\op} \right] \\
	 & \overset{(ii)}{\leq}  \E\left[ \left\| \big( \cT_{s, \nu_2, \nu_1}  - p_{s, \nu_1 \to \nu_2} I_d \big) 1 \{ \|\cT_{s, \nu_2, \nu_1} - p_{s, \nu_1 \to \nu_2} I_d\|_{\op} \leq \delta \}\right\|_{\op}\right]     \\
	 & \qquad +  \E\left[\left\|\big( \cT_{s, \nu_2, \nu_1}  - p_{s, \nu_1 \to \nu_2} I_d \big) 1 \{ \|\cT_{s, \nu_2, \nu_1}  - p_{s, \nu_1 \to \nu_2} I_d\|_{\op} > \delta \} \right\|_{\op}  \right]\\
	 & \overset{(iii)}{\leq}  \, \delta + 2 \P \left( \|\cT_{s, \nu_2, \nu_1}  - p_{s, \nu_1 \to \nu_2} I_d\|_{\op} > \delta \right)  \\
	 & \leq \delta + 2 \gamma_1,  
\end{split}
\end{align}
where we recall $p_{s, \nu_1 \to \nu_2}$ is from Assumption \ref{ass:sample-size-ratio}.
In the above display, $(i)$ is by Jensen's inequality, $(ii)$ is by triangle inequality, and $(iii)$ is because $\|\cT_{s, \nu_2, \nu_1}\|_{\op} \leq 1$. 
Leveraging \cref{eq:cT-op-concentrate}, we derive the following upper bound for the quantity displayed in the last line of \cref{eq:path-small-1}: 
\begin{align}
	\label{eq:path-small-2}
		& \sum_{v_1 \to v_2 \to \cdots \to v_{x} \to \mu \in \cL_{x }^{\ast}} \big\| \E[ \cT_{t, \mu, v_{x}}\big\|_{\op} \, \big\|\E[\cT_{t - 1,v_{x}, v_{x - 1}} ]\big\|_{\op} \cdots \big\| \E[ \cT_{t - x + 1, v_2,  v_1} ]\big\|_{\op} \nonumber  \\
		& \leq \sum_{v_1 \to v_2 \to \cdots \to v_{x} \to \mu \in \cL_{x }^{\ast}} (p_{t, v_{x} \to \mu} + \delta + 2\gamma_1)(p_{t - 1, v_{x - 1} \to v_{x}} + \delta + 2\gamma_1) \cdots (p_{t - x + 1, v_1 \to v_2} + \delta + 2\gamma_1). 
	\end{align}
	Recall that $K$ is the number of models. 
	We let $a_x = K\lceil \frac{x}{K} \rceil$, and 
	note that 
	\begin{align}
	\label{eq:path-small-3}
	\begin{split}
		& \sum_{v_1 \to v_2 \to \cdots \to v_{x} \to \mu \in \cL_{x }^{\ast}} (p_{t, v_{x} \to \mu} + \delta + 2\gamma_1)(p_{t - 1, v_{x - 1} \to v_{x}} + \delta + 2\gamma_1) \cdots (p_{t - x + 1, v_1 \to v_2} + \delta + 2\gamma_1) \\
		& \leq \left[ \sup_{v_{K + 1} \in \cM_l \backslash \cM_l^{\infty}, \, s \geq K} \sum_{v_1 \to v_2 \cdots \to v_{K} \to v_{K + 1} \in \cL_{K}^{\ast}} (p_{s, v_{K} \to v_{K + 1}} + \delta + 2\gamma_1) \cdots (p_{s - K + 1, v_1 \to v_{2}} + \delta + 2\gamma_1) \right]^{\lceil x  / K\rceil} \\
		& \, \times \Bigg[ \sup_{v_{x - a_x + 1} \in \cM_l \backslash \cM_l^{\infty}, \, s \geq x - a_x}  \sum_{v_1 \to v_2 \cdots \to v_{x - a_x} \to v_{x - a_x + 1}  \in \cL_{x - a_x}^{\ast}} (p_{s, v_{x - a_x} \to v_{x - a_x + 1} } + \delta + 2\gamma_1) \cdots \\
		& \qquad \qquad  \cdots (p_{s - x + a_x + 1, v_1 \to v_{2}} + \delta + 2\gamma_1) \Bigg] \vee 1 \\
		& \leq \left[ \sup_{v_{K + 1} \in \cM_l \backslash \cM_l^{\infty}, \, s \geq K} \sum_{v_1 \to v_2 \cdots \to v_{K} \to v_{K + 1}  \in \cL_{K}^{\ast}} p_{s, v_{K} \to v_{K + 1} } \cdots p_{s - K + 1, v_1 \to v_{2}} + \sum_{i = 1}^{K} {{K} \choose i}\Big( \frac{\delta + 2 \gamma_1}{\alpha}\Big)^i \right]^{\lceil x / K\rceil} \\
		& \, \times \Bigg[ \sup_{v_{x - a_x + 1} \in \cM_l \backslash \cM_l^{\infty}, \, s \geq x - a_x}  \sum_{v_1 \to v_2 \cdots \to v_{x - a_x} \to v_{x - a_x + 1}  \in \cL_{x - a_x}^{\ast}} p_{s, v_{x - a_x} \to v_{x - a_x + 1} }  \cdots p_{s - x + a_x + 1, v_1 \to v_{2}} \\
		&  \qquad  + \sum_{i = 1}^{K} {{K} \choose i}\Big( \frac{\delta + 2 \gamma_1}{\alpha}\Big)^i  \Bigg] \vee 1. 
	\end{split}
	\end{align}
	where the last inequality is because under Assumption \ref{ass:sample-size-ratio}, for both $k \in \{K, x - a_x\}$ (note that $x - a_x \in \{0, 1, \cdots, K - 1\}$) and a fixed $v_{k + 1}$ (that is to say, we do not sum over $v_{k + 1}$), it holds that 
	\begin{align*}
		& \sum_{v_1 \to v_2 \cdots \to v_{k} \to v_{k + 1} \in \cL_{k}^{\ast}} (p_{s, v_{k} \to v_{k + 1}} + \delta + 2\gamma_1) \cdots (p_{s - k + 1, v_1 \to v_{2}} + \delta + 2\gamma_1) \\
		& = \sum_{v_1 \to v_2 \cdots \to v_{k} \to v_{k + 1} \in \cL_{k}^{\ast}} p_{s, v_k \to v_{k + 1}} \cdots p_{s - k + 1, v_1 \to v_2}  \\
		& + \sum_{v_1 \to v_2 \cdots \to v_{k} \to v_{k + 1} \in \cL_{k}^{\ast}} p_{s, v_k \to v_{k + 1}} \cdots p_{s - k + 1, v_1 \to v_2}  \sum_{i = 1}^k \sum_{1 \leq j_1 <  \cdots <  j_i \leq k } \prod_{r = 1}^i \frac{\delta + 2 \gamma_1}{p_{s - k + j_r, v_{j_r} \to v_{j_r + 1}}} \\
		& \leq \sum_{v_1 \to v_2 \cdots \to v_{k} \to v_{k + 1} \in \cL_{k}^{\ast}} p_{s, v_k \to v_{k + 1}} \cdots p_{s - k + 1, v_1 \to v_2}  + \sum_{i = 1}^K {K \choose i} \Big( \frac{\delta + 2 \gamma_1}{\alpha}\Big)^i, 
	\end{align*}
	where the last inequality is because $\sum_{\nu \in \Nin{\mu}} p_{t, \nu \to \mu} = 1$ for all $\mu \in \cM_l$, hence 
	\begin{align}
	\label{eq:sum-p-1}
		\sum_{v_1 \to v_2 \cdots \to v_{k} \to v_{k + 1} \in \cL_{k}^{\ast}} p_{s, v_k \to v_{k + 1}} \cdots p_{s - k + 1, v_1 \to v_2}  \leq 1\,\,\,\,\,\,\,\,\,\,\mbox{ for }k \in \{K, \,x - a_x\}. 
	\end{align}
	In addition, 
	under Assumption \ref{ass:sample-size-ratio}, 
	\begin{align*}
		\prod_{r = 1}^i \frac{\delta + 2 \gamma_1}{p_{s - k + j_r, v_{j_r} \to v_{j_r + 1}}} \leq \Big( \frac{\delta + 2 \gamma_1}{\alpha} \Big)^i. 
	\end{align*}
	For any fixed $v_{K + 1}$ and a path $v_1 \to v_2 \cdots \to v_{K} \to v_{K + 1} \in \cL_{K}^{\ast}$, by definition of $\cL_K^{\ast}$ we know $v_{K + 1} \in \cM_l \backslash \cM_l^{\infty}$. 
	Therefore, there exists a path $u_1 \to u_2 \to \cdots \to u_J \to v_{K + 1}$ with $J \leq K - 1$, such that $u_1 \in \cM_u$. 
	Observe that (the sum below is taken over $v_1, v_2, \cdots, v_K$)
	\begin{align*}
		& \sum_{v_1 \to v_2 \cdots \to v_{K} \to v_{K + 1} \in \cL_{K}^{\ast}} p_{s, v_K \to v_{K + 1}}\,  \cdots p_{s - K + 1, v_1 \to v_2} + p_{s, u_J \to v_{K + 1}}\, p_{s - 1, u_{J - 1} \to u_J} \cdots p_{s - J + 1, u_1 \to u_2}  \leq 1, \\
		& p_{s, u_J \to v_{K + 1}}\, p_{s - 1, u_{J - 1} \to u_J}  \cdots p_{s - J + 1, u_1 \to u_2} \geq \alpha^K, \qquad \qquad \qquad \qquad \qquad  \qquad    \mbox{ by Assumption \ref{ass:sample-size-ratio}}. 
	\end{align*}
	We then conclude that fixing $v_{K + 1}$, 
	\begin{align}
	\label{eq:sum-p-2}
		\sum_{v_1 \to v_2 \cdots \to v_{K} \to v_{K + 1} \in \cL_{K}^{\ast}} p_{s, v_K \to v_{K + 1}} \cdots p_{s - K + 1, v_1 \to v_2}  \leq 1 - \alpha^K. 
	\end{align}
	Putting together \cref{eq:path-small-3,eq:sum-p-1,eq:sum-p-2}, we arrive at the following upper bound: 
	\begin{align}
	\label{eq:exponential-decay-bound}
	\begin{split}
		& \sum_{v_1 \to v_2 \to \cdots \to v_{x} \to \mu \in \cL_{x }^{\ast}} (p_{t, v_{x} \to \mu} + \delta + 2\gamma_1)(p_{t - 1, v_{x - 1} \to v_{x}} + \delta + 2\gamma_1) \cdots (p_{t - x + 1, v_1 \to v_2} + \delta + 2\gamma_1) \\
		& \leq \Big(1 + \frac{\delta + 2\gamma_1}{\alpha}\Big)^{K} \left[ \Big(1 + \frac{\delta + 2\gamma_1}{\alpha}\Big)^{K} - \alpha^{K} \right]^{\lceil x/ K\rceil},  
	\end{split}
	\end{align}
	which under Assumption \ref{ass:concentration} is no larger than $2 c(\alpha, K)^{\lceil x / K \rceil}$. 
	The proof then follows via putting together \cref{eq:path-small-1,eq:path-small-2,eq:exponential-decay-bound}. Finally, the ``on the other hand'' part follows directly from the definition of $\cL_x^{\ast}$.

\subsection{Proof of \cref{lemma:sum-grow}}
\label{proof:lemma:sum-grow}	

The proof proceeds by analyzing the first sum in \cref{eq:two-sums} separately for different values of $\ell$.
We first introduce a decomposition of this sum: 
\begin{align}
\label{eq:23}
\begin{split}
& \mbox{First sum in \cref{eq:two-sums}} \\
& = \sum_{\ell = 1}^t \sum_{x = \ell}^t \sum_{v \in \cM_l^{\infty}} \sum_{v_1 \to v_2 \to \cdots \to v_{x} \to \mu \in \cL_{x, \ell, v}} \E \big[ \cT_{t, \mu, v_x} \big] \E \big[ \cT_{t - 1, v_{x}, v_{x - 1}} \big]\cdots \E \big[ \cT_{t - x + 1, v_2, v_1} \big] \\
&\,\,\,\,\,\, + 1 \{\mu \in \cM_l^{\infty}\} \sum_{x = 1}^t \sum_{v_1 \to v_2 \to \cdots \to v_{x} \to \mu \in \cL_{x, 0, \mu}} \E \big[ \cT_{t, \mu, v_x} \big] \E \big[ \cT_{t - 1, v_{x}, v_{x - 1}} \big]\cdots \E \big[ \cT_{t - x + 1, v_2, v_1} \big]. 
\end{split}
\end{align}
We then separately consider two cases $\mu \in \cM_l^{\rm c} \backslash \cM_l^{\infty}$ and $\mu \in \cM_l^{\infty}$.

\paragraph{Case I: \,$\mu \in \cM_l^{\rm c} \backslash \cM_l^{\infty}$. }
In this case, we only need to consider terms associated with a strictly positive $\ell$ in \cref{eq:23}, as the second sum is zero. 
Note that for a fixed  $\ell \in \N_+$, 
	\begin{align}
	\label{eq:many-cT-positive-ell}
	\begin{split}
		& \sum_{x = \ell}^t \sum_{v \in \cM_l^{\infty}} \sum_{v_1 \to v_2 \to \cdots \to v_{x} \to \mu \in \cL_{x, \ell, v}} \E \big[ \cT_{t, \mu, v_x} \big] \E \big[ \cT_{t - 1, v_{x}, v_{x - 1}} \big]\cdots \E \big[ \cT_{t - x + 1, v_2, v_1} \big] \\
		& = (t - \ell + 1) \sum_{v \in \cM_l^{\infty}} \sum_{v_1 \to v_2 \to \cdots \to v_{\ell - 1} \to \mu \in \cL_{\ell - 1}^{\ast},  \, v \in \Nin{v_1}} \E[ \cT_{t,\mu, v_{\ell - 1}}] \E[\cT_{t - 1, v_{\ell - 1},  v_{\ell - 2}  }] \\
		& \qquad  \cdots \E[\cT_{t - \ell + 2, v_2,  v_1}] \E[\cT_{t - \ell + 1, v_1,  v}]. 
	\end{split}
	\end{align}
	We then prove that the matrix given in \cref{eq:many-cT-positive-ell} can be well approximated by a scalar times an identity matrix.
	Observe that  
	\begin{align}
	\label{eq:K-gamma-delta-alpha}
		& \Big\| \sum_{v \in \cM_l^{\infty}} \sum_{v_1 \to v_2 \to \cdots \to v_{\ell - 1} \to \mu \in \cL_{\ell - 1}^{\ast},  \, v \in \Nin{v_1}}\big( \E[ \cT_{t, \mu, v_{\ell - 1}}] \E[\cT_{t - 1, v_{\ell - 1}, v_{\ell - 2}}] \cdots \E[\cT_{t - \ell + 2, v_2,  v_1}] \E[\cT_{t - \ell + 1, v_1, v}] \nonumber \\
		& \qquad  - p_{t, v_{\ell - 1} \to \mu}\, p_{t - 1, v_{\ell - 2} \to v_{\ell - 1}} \cdots p_{t - \ell + 2, v_1 \to v_2} p_{t - \ell + 1, v \to v_1} I_d \big)\Big\|_{\op} \nonumber \\
		& \leq \sum_{v \in \cM_l^{\infty}} \sum_{v_1 \to v_2 \to \cdots \to v_{\ell - 1} \to \mu \in \cL_{\ell - 1}^{\ast},  \, v \in \Nin{v_1}}  \Big\|  \E[ \cT_{t, \mu, v_{\ell - 1}}] \E[\cT_{t - 1, v_{\ell - 1}, v_{\ell - 2}}] \cdots \E[\cT_{t - \ell + 2, v_2,  v_1}] \E[\cT_{t - \ell + 1, v_1, v}].  \\
		& \qquad - p_{t, v_{\ell - 1} \to \mu}\, p_{t - 1, v_{\ell - 2} \to v_{\ell - 1}} \cdots p_{t - \ell + 2, v_1 \to v_2} p_{t - \ell + 1, v \to v_1} I_d  \Big\|_{\op}. \nonumber   
	\end{align}
	For $i \in \{0, 1, \cdots, \ell - 1\}$,
	we express $\E[\cT_{t - i, \nu_{\ell - i}, \nu_{\ell - i - 1}}]$ (treating $\mu$ as $v_{\ell}$ and $v$ as $v_0$) as the sum of $p_{t - i, \nu_{\ell - i - 1} \to \nu_{\ell - i}} I_d$ and $\E[\cT_{t - i, \nu_{\ell - i}, \nu_{\ell - i - 1}}] - p_{t - i, \nu_{\ell - i - 1} \to \nu_{\ell - i}}I_d$. 
	We then apply the operator norm bound $\|\mathbb{E}[\mathcal{T}_{t - i, \nu_{\ell - i}, \nu_{\ell - i - 1}}] - p_{t - i, \nu_{\ell - i - 1} \to \nu_{\ell - i}} I_d \|_{\text{op}} \leq \delta + 2 \gamma_1$ from \cref{eq:cT-op-concentrate}, and obtain an upper bound for the final line of \cref{eq:K-gamma-delta-alpha}. 
	Specifically, let $D_i = \E[\cT_{t - i, \nu_{\ell - i}, \nu_{\ell - i - 1}}] - p_{t - i, \nu_{\ell - i - 1} \to \nu_{\ell - i}}I_d$ for $i \in \{0, 1, \cdots, \ell - 1\}$, then
	\begin{align}
	\label{eq:26}
		& \mbox{The last line of \cref{eq:K-gamma-delta-alpha}} \nonumber \\
		& = \sum_{v \in \cM_l^{\infty}} \sum_{v_1 \to v_2 \to \cdots \to v_{\ell - 1} \to \mu \in \cL_{\ell - 1}^{\ast},  \, v \in \Nin{v_1}} \Big\| (D_0 + p_{t, v_{\ell - 1} \to v_{\ell}} I_d) \cdots (D_{\ell - 1} + p_{t - \ell + 1, v_0 \to v_1} I_d) \nonumber  \\
		& \qquad - p_{t, v_{\ell - 1} \to v_{ \ell}} \cdots p_{t - \ell + 1, v_0 \to v_1} I_d \Big\|_{\op} \nonumber  \\
		& \leq  \sum_{v \in \cM_l^{\infty}} \sum_{v_1 \to v_2 \to \cdots \to v_{\ell - 1} \to \mu \in \cL_{\ell - 1}^{\ast},  \, v \in \Nin{v_1}} \prod_{i = 0}^{\ell - 1} p_{t - i, v_{\ell - i - 1} \to v_{\ell - i}}    \sum_{j = 1}^{\ell} \sum_{0 \leq h_1 < \cdots < h_j \leq \ell - 1} \prod_{b = 1}^j \frac{\delta + 2 \gamma_1}{p_{t - h_b, \nu_{\ell - h_b - 1} \to \nu_{\ell - h_b}}} \nonumber  \\
		& \leq  \sum_{v \in \cM_l^{\infty}} \sum_{v_1 \to v_2 \to \cdots \to v_{\ell - 1} \to \mu \in \cL_{\ell - 1}^{\ast},  \, v \in \Nin{v_1}} \prod_{i = 0}^{\ell - 1} p_{t - i, v_{\ell - i - 1} \to v_{\ell - i}} \sum_{j = 1}^{\ell} {\ell \choose j} \Big( \frac{\delta + 2 \gamma_1}{\alpha} \Big)^j. 
	\end{align}
	Let $a_{\ell - 1} = K \lceil (\ell - 1) / K \rceil$. 
	Note that when $\ell - 1 \geq K$, similar to the derivation of \cref{eq:sum-p-2}, it holds that
	\begin{align}
	\label{eq:35}
		& \sum_{v \in \cM_l^{\infty}} \sum_{v_1 \to v_2 \to \cdots \to v_{\ell - 1} \to \mu \in \cL_{\ell - 1}^{\ast},  \, v \in \Nin{v_1}} \prod_{i = 0}^{\ell - 1} p_{t - i, v_{\ell - i - 1} \to v_{\ell - i}} \nonumber \\
		& \leq K \sum_{v_1 \to v_2 \to \cdots \to v_{\ell - 1} \to \mu \in \cL_{\ell - 1}^{\ast}} \prod_{i = 0}^{\ell - 2} p_{t - i, v_{\ell - i - 1} \to v_{\ell - i}}\nonumber \\
		& \leq K \left[ \sup_{v_{K + 1} \in \cM_l \backslash \cM_l^{\infty}, \, s \geq K} \sum_{v_1 \to v_2 \cdots \to v_{K} \to v_{K + 1} \in \cL_{K}^{\ast}} p_{s, v_{K} \to v_{K + 1}} p_{s - 1, v_{K - 1} \to v_K}  \cdots p_{s - K + 1, v_1 \to v_{2}}  \right]^{\lceil (\ell - 1)  / K\rceil}\nonumber \\
		& \, \times \Bigg[ \sup_{v_{\ell - a_{\ell - 1}} \in \cM_l \backslash \cM_l^{\infty}, \, s \geq \ell - 1 - a_{\ell - 1}}  \sum_{v_1 \to \cdots \to v_{\ell - a_{\ell - 1}}  \in \cL_{\ell - 1 - a_{\ell - 1}}^{\ast}} p_{s, v_{\ell - 1 - a_{\ell - 1}} \to v_{\ell - a_{\ell - 1}} }   \cdots p_{s - \ell + a_{\ell - 1} + 2, v_1 \to v_{2}} \Bigg] \vee 1\nonumber \\
		& \leq K (1 - \alpha^K)^{\lceil (\ell - 1) / K \rceil }. 
	\end{align}
	Substituting \cref{eq:35} into \cref{eq:26}, we conclude that
	\begin{align}
	\label{eq:36}
	\begin{split}
		 \mbox{the last line of \cref{eq:K-gamma-delta-alpha}} \leq & K (1 - \alpha^K)^{\lceil (\ell - 1) / K \rceil } \left[  \Big( 1 + \frac{\delta + 2 \gamma_1}{\alpha} \Big)^{\ell} - 1 \right] \\
		 \leq & \frac{K}{(1 - \alpha^K)^{1 / K + 1}} \left[ \Big( 1 + \frac{\delta + 2 \gamma_1}{\alpha} \Big)^{\ell} (1 - \alpha^K)^{\ell / K} - (1 - \alpha^K)^{\ell / K} \right]. 
		\end{split}
	\end{align}	
	%
	%
	%
	%

	For the sake of simplicity, for $\ell \in \N_+$, we define the following $d \times d $ matrices: 
	\begin{align*}
		& Q_{\ell} = \sum_{v \in \cM_l^{\infty}} \sum_{v_1 \to v_2 \to \cdots \to v_{\ell - 1} \to \mu \in \cL_{\ell - 1}^{\ast},  \, v \in \Nin{v_1}}  \E[ \cT_{t, \mu, v_{\ell - 1}}] \E[\cT_{t - 1, v_{\ell - 1}, v_{\ell - 2}}] \cdots \E[\cT_{t - \ell + 2, v_2, v_1}] \E[\cT_{t - \ell + 1, v_1, v}], \\
		& Q_{\ell}^{\ast} = \sum_{v \in \cM_l^{\infty}} \sum_{v_1 \to v_2 \to \cdots \to v_{\ell - 1} \to \mu \in \cL_{\ell - 1}^{\ast},  \, v \in \Nin{v_1}} p_{t, v_{\ell - 1} \to \mu}\, p_{t - 1, v_{\ell - 2} \to v_{\ell - 1}}  \cdots p_{t - \ell + 2, v_1 \to v_2} \,p_{t - \ell + 1, v \to v_1} I_d. 
	\end{align*} 
	Under Assumption \ref{ass:concentration}, $Q_\ell^{\ast}$ can be regarded as a population version of $Q_{\ell}$. 
	Note that 
	\begin{align*}
		\sum_{\ell = 1}^t \sum_{x = \ell}^t \sum_{v \in \cM_l^{\infty}} \sum_{v_1 \to v_2 \to \cdots \to v_{x} \to \mu \in \cL_{x, \ell, v}} \E \big[ \cT_{t, \mu, v_x} \big] \E \big[ \cT_{t - 1, v_{x}, v_{x - 1}} \big]\cdots \E \big[ \cT_{t - x + 1, v_2, v_1} \big] = \sum_{\ell = 1}^t (t - \ell + 1) Q_{\ell}. 
	\end{align*}
	As $\mu \in \cM_l^{\rm c} \backslash \cM_l^{\infty}$, by definition there exists a path $v \to v_1 \to v_2 \to \cdots \to v_{\ell - 1} \to \mu$ of length $\ell \in [K]$ from some $v \in \cM_l^{\infty}$ to $\mu$, 
	where the intermediate nodes $v_1, v_2, \cdots, v_{\ell - 1} \notin \cM_l^{\infty}$. 
	Therefore, $\sum_{\ell = 1}^K \| Q_{\ell}^{\ast} \|_{\op} \geq \alpha^K$. 
	Together, these results yield a lower bound for the matrix largest eigenvalue (recall $\sigma_{\max}(A)$ refers to the largest eigenvalue of $A$):
	\begin{align*}
		& \sigma_{\max} \Big( \sum_{\ell = 1}^t (t - \ell + 1) Q_{\ell} \Big) \\
		& \overset{(i)}{\geq} \sum_{\ell = 1}^t (t - \ell + 1) \|Q_{\ell}^{\ast}\|_{\op} - \sum_{\ell = 1}^t (t - \ell + 1) \|Q_{\ell} - Q_{\ell}^{\ast}\|_{\op} \\
		& \overset{(ii)}{\geq} (t - K + 1) \alpha^K - t \sum_{\ell = 1}^{\infty} \frac{K}{(1 - \alpha^K)^{1 / K + 1}} \left[ \Big( 1 + \frac{\delta + 2 \gamma_1}{\alpha} \Big)^{\ell} (1 - \alpha^K)^{\ell / K} - (1 - \alpha^K)^{\ell / K} \right] \\
		& \overset{(iii)}{\geq} (t - K + 1) \alpha^K - \frac{tK (\delta + 2 \gamma_1)}{\alpha(1 - \alpha^K)^{1 / K + 1} (1 - c_1)^2} \\
		& = t\, \Big( \alpha^K - \frac{K (\delta + 2 \gamma_1)}{\alpha(1 - \alpha^K)^{1 / K + 1} (1 - c_1)^2}  \Big) - \alpha^K (K - 1) = t \, c_2 - \alpha^K (K - 1), 
	\end{align*}
	where $(i)$ uses the fact that $Q_\ell^{\ast}$ is a positive scalar times an identity matrix, 
    $(ii)$ is by \cref{eq:36}, and $(iii)$ is by Assumption \ref{ass:concentration}. 
	The proof is done.

	\paragraph{Case II: $\mu \in \cM_l^{\infty}$.}
	In this case,  the second sum in \cref{eq:23} is zero. On the other hand, when $\mu \in \cM_l^{\infty}$, it holds that 
	\begin{align}
	\label{eq:ell-part2}
	\begin{split}
		\sum_{x = 1}^t \sum_{v_1 \to v_2 \to \cdots \to v_{x} \to \mu \in \cL_{x, 0, \mu}} \E \big[ \cT_{t, \mu, v_x} \big] \E \big[ \cT_{t - 1, v_{x}, v_{x - 1}} \big]\cdots \E \big[ \cT_{t - x + 1, v_2, v_1} \big] = t I_d. 
	\end{split}	
	\end{align}
	Therefore,   
	\begin{align*}
		& \Big\| \,\sum_{x = 1}^{t} \sum_{v \in \cM_l^{\infty}} \sum_{\ell = 0}^{x} \sum_{v_1 \to v_2 \to \cdots \to v_{x} \to \mu \in \cL_{x, \ell, v}} \E\big[ \cT_{t,  \mu, v_{x}} \big] \E \big[ \cT_{t - 1, v_{x}, v_{x - 1}}\big] \cdots \E \big[\cT_{t - x + 1, v_2, v_1} \big]\,\Big\|_{\op} \\
		& \geq t > t \, c_2 - \alpha^K (K - 1).  
	\end{align*}
	The proof is done. Finally, the ``on the other hand'' part follows directly from the definition of $\cM_l^{\rm nc}$.


\subsection{Proof of \cref{lemma:A-indicator-upper-bound}}
\label{proof:lemma:A-indicator-upper-bound}

In this section we prove \cref{lemma:A-indicator-upper-bound}. 
According to  \cref{lemma:sumA2}, 
	\begin{align*}
	\begin{split}
		& \E \Big[ \sum_{i = 1}^{N_t} A_{t, i, \mu} \max_{v \in \cM} \big( 1 - 1 \{X_{s_{t, i}, v}^{\top} X_{s_{t, i}, v} / n_{s_{t, i}, v} \preceq \kappa I_d \} \big) \Big] \\
		&  = \E\Big[ \sum_{s = 1}^t \sum_{\nu \in \cM_l} \sum_{m \in \cM} \Omega_{t, s + 1, \mu, \nu} \cT_{s, \nu, m} \max_{v \in \cM} \big( 1 - 1 \{X_{s, v}^{\top} X_{s, v} / n_{s, v} \preceq \kappa I_d \} \big) \Big] \\
		& = \sum_{x = 1}^t \sum_{v_1  \to \cdots \to v_x \to \mu \in \cL_x} \hspace{-1em}\E\Big[  \cT_{t, \mu, v_x} \cT_{t - 1, v_x, v_{x - 1}} \cdots \cT_{t - x + 1, v_2, v_1} \max_{v \in \cM} \big( 1 - 1 \{X_{t - x + 1, v}^{\top} X_{t - x + 1, v} / n_{t - x + 1, v} \preceq \kappa I_d \} \big) \Big] \\
    & =  \sum_{x = 1}^t \sum_{v_1 \to \cdots \to v_x \to \mu \in \cL_x} \hspace{-1em} \E[  \cT_{t, \mu, v_x}] \E[ \cT_{t - 1, v_x, v_{x - 1}}] \cdots \E[ \cT_{t - x + 1, v_2, v_1}  \max_{v \in \cM} \big( 1 - 1 \{X_{t - x + 1, v}^{\top} X_{t - x + 1, v} / n_{t - x + 1, v} \preceq \kappa I_d \} \big) ]. 
	\end{split}
	\end{align*}
Recall $\cL_{x, \ell, v}$ and $\cL_x^{\ast}$ are defined in \cref{eq:Lxellv,eq:Lxstar}. 
With these definitions, we have
\begin{align}
\label{eq:24}
	\begin{split}
	& \E\Big[ \sum_{i = 1}^{N_t} A_{t, i, \mu} \max_{v \in \cM} \big( 1 - 1 \{X_{s_{t, i}, v}^{\top} X_{s_{t, i}, v} / n_{s_{t, i}, v} \preceq \kappa I_d \} \big) \Big] \\
	& =  \sum_{x = 1}^{t} \sum_{v \in \cM_l^{\infty}} \sum_{\ell = 0}^{x} \sum_{v_1  \to \cdots \to v_{x} \to \mu \in \cL_{x, \ell, v}} \E\big[ \cT_{t,  \mu, v_{x}} \big]  \cdots \E \big[\cT_{t - x + 1, v_2, v_1} \max_{v \in \cM} \big( 1 - 1 \{X_{t - x + 1, v}^{\top} X_{t - x + 1, v} / n_{t - x + 1, v} \preceq \kappa I_d \} \big) \big] \\
	&\qquad  + \sum_{x = 1}^{t } \sum_{v_1  \to \cdots \to v_{x} \to \mu \in \cL_{x}^{\ast}} \E\big[ \cT_{t,  \mu, v_{x}} \big]  \cdots \E \big[\cT_{t - x + 1, v_2, v_1} \max_{v \in \cM} \big( 1 - 1 \{X_{t - x + 1, v}^{\top} X_{t - x + 1, v} / n_{t - x + 1, v} \preceq \kappa I_d \} \big) \big].
\end{split}
\end{align}
We denote by $\mathtt{S_1}$ and $\mathtt{S_2}$ the first and second sums in \cref{eq:24}, respectively.
We separately upper bound these two terms below, starting from $\mathtt{S_2}$.

\paragraph{Upper bounding $\mathtt{S_2}$. } 

When $\mu \in \cM_l^{\infty}$ we have $\mathtt{S}_2 = 0_{d \times d}$. 
In what follows, we assume $\mu \in \cM_l \backslash \cM_l^{\infty}$. 
Note that for all $s \in \N_+$ and $(\nu_1, \nu_2) \in \cE$, by \cref{eq:cT-op-concentrate} we have $\|\E[\cT_{s, \nu_2, \nu_1} - p_{s, \nu_1 \to \nu_2} I_d]\|_{\op} \leq \delta + 2 \gamma_1$. 
In addition, 
\begin{align*}
	 \big\|\, \E\big[ \cT_{s, \nu_2, \nu_1} \max_{v \in \cM} \big( 1 - 1 \{X_{s, v}^{\top} X_{s, v} / n_{s, v} \preceq \kappa I_d \} \big) \big] \, \big\|_{\op} & \overset{(i)}{\leq} \E\big[ \|\cT_{s, \nu_2, \nu_1}\|_{\op} \max_{v \in \cM} \big( 1 - 1 \{X_{s, v}^{\top} X_{s, v} / n_{s, v} \preceq \kappa I_d \} \big) \big]  \\
	& \overset{(ii)}{\leq} 1 - \P\big( \cap_{v \in \cM} \{X_{s, v}^{\top} X_{s, v} / n_{s, v} \preceq \kappa I_d \}  \big) \\
	& \overset{(iii)}{\leq} K \gamma_2, 
\end{align*}
where $(i)$ is by Jensen's inequality, $(ii)$ is because $\|\cT_{s, \nu_2, \nu_1}\|_{\op} \leq 1$, and $(iii)$ is by Assumption \ref{ass:concentration}. 
As a consequence, 
\begin{align*}
 & \|\mathtt{S_2}\|_{\op}\\
 & \leq \sum_{x = 1}^{t } \sum_{v_1  \to \cdots \to v_{x} \to \mu \in \cL_{x}^{\ast}} \| \E\big[ \cT_{t,  \mu, v_{x}} \big] \|_{\op} \cdots \| \E \big[\cT_{t - x + 1, v_2, v_1} \max_{v \in \cM} \big( 1 - 1 \{X_{t - x + 1, v}^{\top} X_{t - x + 1, v} / n_{t - x + 1, v} \preceq \kappa I_d \} \big) \big] \|_{\op}\\
& \leq \sum_{x = 1}^t \sum_{v_1  \to \cdots \to v_{x} \to \mu \in \cL_{x }^{\ast}} K \gamma_2 \,  (p_{t, v_{x} \to \mu} + \delta + 2\gamma_1)(p_{t - 1, v_{x - 1} \to v_{x}} + \delta + 2\gamma_1) \cdots (p_{t - x + 2, v_2 \to v_3} + \delta + 2\gamma_1) \\
& \leq \frac{K \gamma_2}{\alpha} \sum_{x = 1}^t \sum_{v_1  \to \cdots \to v_{x} \to \mu \in \cL_{x }^{\ast}}  (p_{t, v_{x} \to \mu} + \delta + 2\gamma_1)(p_{t - 1, v_{x - 1} \to v_{x}} + \delta + 2\gamma_1) \cdots (p_{t - x + 1, v_1 \to v_2} + \delta + 2\gamma_1). 
\end{align*}
Similar to the derivation of \cref{eq:exponential-decay-bound}, the above quantity is no larger than 
\begin{align*}
	\frac{K \gamma_2}{\alpha}\sum_{x = 1}^t  \Big(1 + \frac{\delta + 2\gamma_1}{\alpha}\Big)^{K} \left[ \Big(1 + \frac{\delta + 2\gamma_1}{\alpha}\Big)^{K} - \alpha^{K} \right]^{\lceil x/ K\rceil}, 
\end{align*}
which is further no larger than $\frac{4K^2 \gamma_2}{\alpha^{K+1}}$ under Assumption \ref{ass:concentration}. 
In summary, for any $\mu \in \cM_{l}$:
\begin{align}
\label{eq:ttS2}
	\|\mathtt{S}_2\|_{\op} \leq \frac{4 K^2 \gamma_2}{\alpha^{K+1}}. 
\end{align}

\paragraph{Upper bounding $\mathtt{S_1}$. } 

We then upper bound $\mathtt{S_1}$. 
When $\mu \in \cM_l^{\rm nc}$, we immediately see that $\mathtt{S}_1 = 0_{d \times d}$. 
Below we consider $\mu \in \cM_l^{\rm c}$. 

Note that $\mathtt{S_1} = \mathtt{M_1} + \mathtt{M_2}$, where 
\begin{align*}
	& \mathtt{M_1} = \sum_{x = 1}^{t} \sum_{v \in \cM_l^{\infty}} \sum_{\ell = 1}^{x} \sum_{v_1  \to \cdots \to v_{x} \to \mu \in \cL_{x, \ell, v}} \E\big[ \cT_{t,  \mu, v_{x}} \big]  \cdots \E \big[\cT_{t - x + 1, v_2, v_1} \max_{v \in \cM} \big( 1 - 1 \{X_{t - x + 1, v}^{\top} X_{t - x + 1, v} / n_{t - x + 1, v} \preceq \kappa I_d \} \big) \big], \\
	& \mathtt{M_2} =   \sum_{x = 1}^t \sum_{v_1 \to v_2 \to \cdots \to v_{x} \to \mu \in \cL_{x, 0, \mu}} \E \big[ \cT_{t, \mu, v_x} \big] \cdots \E \big[ \cT_{t - x + 1, v_2, v_1}  \max_{v \in \cM} \big( 1 - 1 \{X_{t - x + 1, v}^{\top} X_{t - x + 1, v} / n_{t - x + 1, v} \preceq \kappa I_d \} \big) \big]. 
\end{align*}
When $\mu \in \cM_l^{\infty}$, $\mathtt{M_1} = 0_{d \times d}$. As for $\mathtt{M_2}$, observe that 
\begin{align*}
	& \mathtt{M}_2 = \sum_{x = 1}^t \E\Big[ \max_{v \in \cM} \big( 1 - 1 \{X_{t - x + 1, v}^{\top} X_{t - x + 1, v} / n_{t - x + 1, v} \preceq \kappa I_d \} \big) I_d  \Big], \\
	 \Longrightarrow \,\, & \|\mathtt{M}_2\|_{\op} \leq \sum_{x = 1}^t \Big( 1 - \P \big( \cap_{x \in \cM}  \{X_{t - x + 1, v}^{\top} X_{t - x + 1, v} / n_{t - x + 1, v} \preceq \kappa I_d \} \big) \Big) \leq t \, K \gamma_2. 
\end{align*}
On the other hand, when $\mu \in \cM_l^{\rm c} \backslash \cM_l^{\infty}$, then $\mathtt{M_2} = 0_{d \times d}$. 
We then upper bound $\|\mathtt{M_1}\|_{\op}$. 
Note that 
\begin{align*}
	& \mathtt{M_1} = \sum_{\ell = 1}^t \mathtt{M_{1, \ell}}, \\
	& \mathtt{M_{1, \ell}} = \sum_{x = \ell}^t \sum_{v \in \cM_l^{\infty}} \sum_{v_1  \to \cdots \to v_{x} \to \mu \in \cL_{x, \ell, v}} \hspace{-1.2em} \E \big[ \cT_{t, \mu, v_x} \big] \cdots \E \big[ \cT_{t - x + 1, v_2, v_1}  \max_{v \in \cM} \big( 1 - 1 \{X_{t - x + 1, v}^{\top} X_{t - x + 1, v} / n_{t - x + 1, v} \preceq \kappa I_d \} \big)  \big].
\end{align*}
In addition, for any $\ell \in [t - 1]$, 
\begin{align*}
	&\mathtt{M_{1, \ell}} = (t - \ell + 1) \sum_{v \in \cM_l^{\infty}} \sum_{v_1 \to v_2 \to \cdots \to v_{\ell - 1} \to \mu \in \cL_{\ell - 1}^{\ast},  \, v \in \Nin{v_1}} \E[ \cT_{t,\mu, v_{\ell - 1}}] \E[\cT_{t - 1, v_{\ell - 1},  v_{\ell - 2}  }] \\
		& \qquad  \cdots \E[\cT_{t - \ell + 2, v_2,  v_1}] \E[\cT_{t - \ell + 1, v_1,  v}] \E \Big[ \max_{\eta \in \cM} \big( 1 - 1 \{X_{t - x + 1, \eta}^{\top} X_{t - x + 1, \eta} / n_{t - x + 1, \eta} \preceq \kappa I_d \} \big) \Big]. 
\end{align*}
As a consequence, 
\begin{align}
\label{eq:Ml-upper}
	& \frac{\|\mathtt{M_{1, \ell}}\|_{\op}}{K \gamma_2 (t - \ell + 1)} \\
	& \leq  \Big\| \sum_{v \in \cM_l^{\infty}} \sum_{v_1 \to v_2 \to \cdots \to v_{\ell - 1} \to \mu \in \cL_{\ell - 1}^{\ast},  \, v \in \Nin{v_1}}  \E[ \cT_{t, \mu, v_{\ell - 1}}] \E[\cT_{t - 1, v_{\ell - 1}, v_{\ell - 2}}] \cdots \E[\cT_{t - \ell + 2, v_2, v_1}] \E[\cT_{t - \ell + 1, v_1, v}] \Big\|_{\op}. \nonumber 
\end{align}
Leveraging \cref{lemma:sum-decay},    we see that 
	\begin{align}
		\label{eq:K-gamma-delta-alpha2}
		& \Big\| \sum_{v \in \cM_l^{\infty}} \sum_{v_1 \to v_2 \to \cdots \to v_{\ell - 1} \to \mu \in \cL_{\ell - 1}^{\ast},  \, v \in \Nin{v_1}}  \E[ \cT_{t, \mu, v_{\ell - 1}}] \E[\cT_{t - 1, v_{\ell - 1}, v_{\ell - 2}}] \cdots \E[\cT_{t - \ell + 2, v_2, v_1}] \E[\cT_{t - \ell + 1, v_1, v}] \Big\|_{\op} \nonumber \\
		& \leq \sum_{v \in \cM_l^{\infty}} \sum_{v_1 \to v_2 \to \cdots \to v_{\ell - 1} \to \mu \in \cL_{\ell - 1}^{\ast},  \, v \in \Nin{v_1}} \big\| \E[ \cT_{t, \mu, v_{\ell - 1}}]  \big\|_{\op}  \cdots  \big\|\E[\cT_{t - \ell + 2,v_2, v_1}] \big\|_{\op} \big\| \E[\cT_{t - \ell + 1, v_1,  v}]  \big\|_{\op} \nonumber \\
		& \overset{(i)}{\leq} K \sum_{v_1 \to v_2 \to \cdots \to v_{\ell - 1} \to \mu \in \cL_{\ell - 1}^{\ast}} \big\| \E[\cT_{t, \mu, v_{\ell - 1}}] \big\|_{\op}\, \big\| \E[\cT_{t - 1, v_{\ell - 1}, v_{\ell - 2}}]  \big\|_{\op} \cdots \big\| \E[\cT_{t - \ell + 2, v_2, v_1 }] \big\|_{\op} 
		\nonumber  \\
		& \leq K \Big(1 + \frac{\delta + 2\gamma_1}{\alpha}\Big)^{K} \left[ \Big(1 + \frac{\delta + 2\gamma_1}{\alpha}\Big)^{K} - \alpha^{K} \right]^{\lceil  (\ell - 1) / K\rceil}, 
	\end{align}
	where $(i)$ is because for any $s \in \N_+$ and $(\mu_1, \mu_2) \in \cE$, we have $\|\E[\cT_{s, \mu_2, \mu_1}]\|_{\op} \leq 1$.
	Combining \cref{eq:Ml-upper,eq:K-gamma-delta-alpha2}, we conclude that for all $\ell \in [t - 1]$, 
	\begin{align}
	\label{eq:M1l-upper1}
		\|\mathtt{M_{1, \ell}}\|_{\op} \leq K^2 \gamma_2 (t - \ell + 1) \Big(1 + \frac{\delta + 2\gamma_1}{\alpha}\Big)^{K} \left[ \Big(1 + \frac{\delta + 2\gamma_1}{\alpha}\Big)^{K} - \alpha^{K} \right]^{\lceil  (\ell - 1) / K\rceil}. 
	\end{align}
	When $\ell = t$, note that 
	\begin{align*}
		& \mathtt{M_{1, t}} = \sum_{v \in \cM_l^{\infty}} \sum_{v_1 \to v_2 \to \cdots \to v_{t - 1} \to \mu \in \cL_{t - 1}^{\ast},  \, v \in \Nin{v_1}} \E[ \cT_{t,\mu, v_{t - 1}}] \E[\cT_{t - 1, v_{t- 1},  v_{t- 2}  }] \\
		& \qquad  \cdots \E[\cT_{2, v_2,  v_1}] \E \Big[ \cT_{1, v_1,  v} \max_{\eta \in \cM} \big( 1 - 1 \{X_{1, \eta}^{\top} X_{1, \eta} / n_{1, \eta} \preceq \kappa I_d \} \big) \Big].
	\end{align*}
	Using \cref{lemma:sum-decay} again, we see that 
	\begin{align}
	\label{eq:M1l-upper2}
	\begin{split}
		 \|\mathtt{M_{1, t}}\|_{\op} \leq & \sum_{v \in \cM_l^{\infty}} \sum_{v_1 \to v_2 \to \cdots \to v_{t - 1} \to \mu \in \cL_{t - 1}^{\ast},  \, v \in \Nin{v_1}} \big\| \E [ \cT_{t,\mu, v_{t - 1}}] \big\|_{\op} \big\| \E[\cT_{t - 1, v_{t- 1},  v_{t- 2}  }] \big\|_{\op} \\
		& \qquad  \cdots \times  \| \E [ \cT_{2, v_2,  v_1}] \big\|_{\op} \Big\| \E \Big[ \cT_{1, v_1,  v} \max_{\eta \in \cM} \big( 1 - 1 \{X_{1, \eta}^{\top} X_{1, \eta} / n_{1, \eta} \preceq \kappa I_d \} \big) \Big] \Big\|_{\op} \\
		\overset{(i)}{\leq} & \, K^2 \gamma_2 \sum_{v_1 \to v_2 \to \cdots \to v_{t - 1} \to \mu \in \cL_{t - 1}^{\ast}} \big\| \E[\cT_{t, \mu, v_{t - 1}}] \big\|_{\op}\, \big\| \E[\cT_{t - 1, v_{t - 1}, v_{t - 2}}]  \big\|_{\op} \cdots \big\| \E[\cT_{2, v_2, v_1 }] \big\|_{\op} \\
		\leq & \,  K^2 \gamma_2\Big(1 + \frac{\delta + 2\gamma_1}{\alpha}\Big)^{K} \left[ \Big(1 + \frac{\delta + 2\gamma_1}{\alpha}\Big)^{K} - \alpha^{K} \right]^{\lceil  (\ell - 1) / K\rceil}.
	\end{split} 
	\end{align}
	Putting together Assumption \ref{ass:concentration}, \cref{eq:M1l-upper1,eq:M1l-upper2}, we conclude that for any $\mu \in \cM_l$, 
	\begin{align}
	\label{eq:ttS1}
		\|\mathtt{M_1}\|_{\op} \leq \sum_{\ell = 1}^t \|\mathtt{M_{1, \ell}}\|_{\op} \leq \frac{4\,t\, { K^3 \gamma_2}}{\alpha^K}.
 	\end{align}
 	Finally, by \cref{eq:ttS2,eq:ttS1} we see that 
 	\begin{align*}
 		\Big\|\,\E\Big[ \sum_{i = 1}^{N_t} A_{t, i, \mu} \max_{v \in \cM} \big( 1 - 1 \{X_{s_{t, i}, v}^{\top} X_{s_{t, i}, v} / n_{s_{t, i}, v} \preceq \kappa I_d \} \big) \Big] \,\Big\|_{\op} \leq \frac{4\,t\, { K^3 \gamma_2}}{\alpha^K} + \frac{4 K^2 \gamma_2}{\alpha^{K+1}} + t K \gamma_2 \le \frac{9 \,t\, { K^3 \gamma_2}}{\alpha^{K+1}},
 	\end{align*}
    since $K \ge 1$ and $\alpha < 1$.
 	The proof is done. 

\subsection{Proof of Lemma \ref{prop:concentration_Omega}}
\label{appendix:prop:concentration_Omega}

If $\nu_1 \in \cM_u$, then one can verify that $\Omega_{t, s + 1, \nu_1, \nu_2}^{\ast} =  \Omega_{t, s + 1, \nu_1, \nu_2} = \mathbbm{1}_{\nu_1 = \nu_2} I_d$. 
In the rest parts of the proof we consider $\nu_1 \in \cM_l$. 
The proof proceeds by induction on $t-s$. 
If $t-s = 0$, then by definition $\Omega_{s, t + 1, \mu, \nu}^{\ast} = \Omega_{s, t + 1, \mu, \nu} = \mathbbm{1}_{\mu = \nu} I_d $. 
Now suppose \cref{eq:err_bound_Omega} holds for $t-s = k$. 
When $t-s = k + 1$, by definition
    \begin{align*}
        & \Omega_{t, s + 1, \nu_1, \nu_2}^* = \, \sum_{m \in \cM} \cT_{t, \nu_1, m}^* \Omega_{t-1, s + 1, m, \nu_2}^* = \sum_{m \in \Nin{\nu_1}}  \cT_{t, \nu_1, m}^* \Omega_{t-1, s + 1, m, \nu_2}^*, \\
        & \Omega_{t, s + 1, \nu_1, \nu_2} = \, \sum_{m \in \cM} \cT_{t, \nu_1, m} \Omega_{t-1, s + 1, m, \nu_2} = \sum_{m \in \Nin{\nu_1}}  \cT_{t, \nu_1, m} \Omega_{t-1, s + 1, m, \nu_2}.
    \end{align*}
    Therefore,
    \begin{align*}
        & \norm{\Omega_{t, s + 1, \nu_1, \nu_2}^* -  \Omega_{t, s + 1, \nu_1, \nu_2}}_{\op} = \, \Big\|{\sum_{m \in \Nin{\nu_1}} \cT_{t, \nu_1, m}^* \Omega_{t-1, s + 1, m, \nu_2}^* - \sum_{m \in \Nin{\nu_1}} \cT_{t, \nu_1, m} \Omega_{t-1, s + 1, m, \nu_2}}\Big\|_{\op} \\
        & \le \,  \sum_{m \in \Nin{\nu_1}} \left( \big\|{ \cT_{t, \nu_1, m}^* - \cT_{t, \nu_1, m}}\big\|_{\op} \big\|{ \Omega_{t-1, s + 1, m, \nu_2}^*}\big\|_{\op} + \big\|{\Omega_{t-1, s + 1, m, \nu_2}^* - \Omega_{t-1, s + 1, m, \nu_2}}\big\|_{\op} \big\|{ \cT_{t, \nu_1, m}}\big\|_{\op} \right).
    \end{align*}
    By \cref{ass:concentration}, we know that with probability at least $1 - \sum_{\mu \in \Nin{\nu_1}} \delta \geq 1 - N_{\max} \delta$,
    \begin{equation*}
        \big\|{\cT_{t, \nu_1, m}^* - \cT_{t, \nu_1, m}}\big\|_{\op} \le \, \gamma_1, \qquad \forall m \in \Nin{\nu_1}.
    \end{equation*}
    By the induction hypothesis, with probability at least $1 - \delta \sum_{j = 2}^{k + 1} N_{\max}^j$,
    \begin{equation*}
        \norm{\Omega_{t-1, s + 1, m, \nu_2}^* - \Omega_{t-1, s + 1, m, \nu_2}}_{\op} \le \, \gamma_1 \sum_{j = 1}^k N_{\max}^j, \qquad \forall m \in \Nin{\nu_1}.
    \end{equation*}
    Combining the above arguments, we see that with probability at least $1 - \delta \sum_{j = 1}^{k + 1} N_{\max}^j$, 
    we have
    \begin{align*}
        \norm{\Omega_{t, s + 1, \nu_1, \nu_2}^* -  \Omega_{t, s + 1,  \nu_1, \nu_2}}_{\op} \le \, & \sum_{m \in \Nin{\nu_1}} \left( \gamma_1 \norm{ \Omega_{t-1, s + 1, m, \nu_2}^*}_{\op} + \gamma_1 \sum_{j = 1}^k N_{\max}^j \norm{\cT_{t, \nu_1, m}}_{\op} \right) \\
        \stackrel{(i)}{\le} \, & \,\vert \Nin{\nu_1} \vert \left(  \gamma_1 + \gamma_1\sum_{j = 1}^k N_{\max}^j \right) \veps \le \gamma_1 \sum_{j = 1}^{k + 1} N_{\max}.
    \end{align*}
    In the above display, $(i)$ follows from the fact that
    \begin{equation*}
        \norm{\Omega_{t-1, s + 1, m, \nu_2}^*}_{\op} \le 1, \qquad \norm{\cT_{t, \nu_1, m}}_{\op} \le 1,
    \end{equation*}
    which can be verified by straightforward calculation.
    This completes the induction step and the proof of \cref{eq:err_bound_Omega}.

\section{Technical lemmas for the M-estimation setting}

\subsection{Consistency and asymptotic normality of $\hat\beta_{t, \mu}^{\ast}$}
\label{sec:consistency-normality}

In this section, we prove that $\hat\beta_{t, \mu}^{\ast}$ is consistent and asymptotically normal under Assumptions \ref{ass:iid_data}-\ref{ass:population_regularity}.
By definition, 
\begin{align*}
    \hat\beta_{t, \mu}^{\ast} = \arg\min_{\beta \in \mathcal{B}} \,  \frac{1}{\sum_{\nu \in \Nin{\mu}} n_{t, \nu \to \mu} } \sum_{\nu \in \Nin{\mu}} \sum_{i = 1}^{n_{t, \nu \to \mu}} L(\beta, \varphi(\beta_{\ast},  \varepsilon_{t, \nu \to \mu, i})).  
\end{align*}
We first prove consistency. 

\paragraph{Consistency.}
If $\mathcal{B}$ is compact, then 
under Assumption \ref{ass:glivenko_cantelli}, as the sample sizes tend to infinity we have 
\begin{align*}
    \sup_{\beta_1, \beta_2 \in \mathcal{B}}\Bigg| \frac{1}{\sum_{\nu \in \Nin{\mu}} n_{t, \nu \to \mu} } \sum_{\nu \in \Nin{\mu}} \sum_{i = 1}^{n_{t, \nu \to \mu}} L(\beta_1, \varphi(\beta_{2},  \varepsilon_{t, \nu \to \mu, i})) - \E[L(\beta_1, \varphi(\beta_2, \varepsilon))]\, \Bigg| \overset{P}{\to} 0.   
\end{align*}
Therefore, 
\begin{align*}
     &\E[L(\hat\beta_{t, \mu}^{\ast}, \varphi(\beta_{\ast}, \varepsilon))]  \leq  \frac{1}{\sum_{\nu \in \Nin{\mu}} n_{t, \nu \to \mu} } \sum_{\nu \in \Nin{\mu}} \sum_{i = 1}^{n_{t, \nu \to \mu}} L(\hat\beta_{t, \mu}^{\ast}, \varphi(\beta_{\ast},  \varepsilon_{t, \nu \to \mu, i})) \\
    & + \sup_{\beta_1, \beta_2 \in \mathcal{B}}\Bigg| \frac{1}{\sum_{\nu \in \Nin{\mu}} n_{t, \nu \to \mu} } \sum_{\nu \in \Nin{\mu}} \sum_{i = 1}^{n_{t, \nu \to \mu}} L(\beta_1, \varphi(\beta_{2},  \varepsilon_{t, \nu \to \mu, i})) - \E[L(\beta_1, \varphi(\beta_2, \varepsilon))]\, \Bigg| \\
    & \leq \frac{1}{\sum_{\nu \in \Nin{\mu}} n_{t, \nu \to \mu} } \sum_{\nu \in \Nin{\mu}} \sum_{i = 1}^{n_{t, \nu \to \mu}} L(\beta_{\ast}, \varphi(\beta_{\ast},  \varepsilon_{t, \nu \to \mu, i})) \\
    & + \sup_{\beta_1, \beta_2 \in \mathcal{B}}\Bigg| \frac{1}{\sum_{\nu \in \Nin{\mu}} n_{t, \nu \to \mu} } \sum_{\nu \in \Nin{\mu}} \sum_{i = 1}^{n_{t, \nu \to \mu}} L(\beta_1, \varphi(\beta_{2},  \varepsilon_{t, \nu \to \mu, i})) - \E[L(\beta_1, \varphi(\beta_2, \varepsilon))]\, \Bigg| \\
    & \leq \E[L(\beta_{\ast}, \varphi(\beta_{\ast}, \varepsilon)]+ 2\sup_{\beta_1, \beta_2 \in \mathcal{B}}\Bigg| \frac{1}{\sum_{\nu \in \Nin{\mu}} n_{t, \nu \to \mu} } \sum_{\nu \in \Nin{\mu}} \sum_{i = 1}^{n_{t, \nu \to \mu}} L(\beta_1, \varphi(\beta_{2},  \varepsilon_{t, \nu \to \mu, i})) - \E[L(\beta_1, \varphi(\beta_2, \varepsilon))]\, \Bigg| \\
    & = \E[L(\beta_{\ast}, \varphi(\beta_{\ast}, \varepsilon)] + o_P(1). 
\end{align*}
By Assumption \ref{ass:well_specified}, $\beta_{\ast}$ is the unique minimizer of $\beta \mapsto \E[L(\beta, \varphi(\beta_{\ast}, \varepsilon))]$, and the mapping $\beta \mapsto \E[L(\beta, \varphi(\beta_{\ast}, \varepsilon))]$ is continuous. 
Therefore, we conclude that $\hat\beta_{t, \mu}^{\ast} \overset{P}{\to} \beta_{\ast}$. 

We next prove consistency when $L$ is convex and $\mathcal{B}$ is not necessarily compact. 
Since $\beta_{\ast} \in \mathrm{int}\, \mathcal{B}$, there exists $r > 0$, such that $\{x: \|x - \beta_{\ast}\|_2 \leq r\} \subseteq \mathcal{B}$.
Under Assumption \ref{ass:glivenko_cantelli}, we have 
\begin{align*}
    \sup_{\beta_1, \beta_2 \in \{x: \|x - \beta_{\ast}\|_2 \leq r\}}\Bigg| \frac{1}{\sum_{\nu \in \Nin{\mu}} n_{t, \nu \to \mu} } \sum_{\nu \in \Nin{\mu}} \sum_{i = 1}^{n_{t, \nu \to \mu}} L(\beta_1, \varphi(\beta_{2},  \varepsilon_{t, \nu \to \mu, i})) - \E[L(\beta_1, \varphi(\beta_2, \varepsilon))]\, \Bigg| \overset{P}{\to} 0.   
\end{align*}
We next show that for all $k \in \mathbb{N}_+$ satisfying $1 / k < r$, $\mathbb{P}(\|\hat\beta_{t, \mu}^{\ast} - \beta_{\ast}\|_2 > 1 / k) \to 0$ as the sample sizes tend to infinity. 
observe that if $\|\hat\beta_{t, \mu}^{\ast} - \beta_{\ast}\|_2 > 1 / k$, then by convexity, there exists $\tilde \beta$ that satisfies $\|\tilde \beta - \beta_{\ast}\|_2 = 1 / k$, such that 
\begin{align*}
    \frac{1}{\sum_{\nu \in \Nin{\mu}} n_{t, \nu \to \mu} } \sum_{\nu \in \Nin{\mu}} \sum_{i = 1}^{n_{t, \nu \to \mu}} L(\tilde \beta, \varphi(\beta_{\ast},  \varepsilon_{t, \nu \to \mu, i})) \leq \frac{1}{\sum_{\nu \in \Nin{\mu}} n_{t, \nu \to \mu} } \sum_{\nu \in \Nin{\mu}} \sum_{i = 1}^{n_{t, \nu \to \mu}} L(\beta_{\ast}, \varphi(\beta_{\ast},  \varepsilon_{t, \nu \to \mu, i})). 
\end{align*}
By triangle inequality, 
\begin{align*}
    & \E[L(\tilde \beta, \varphi(\beta_{\ast}, \varepsilon))] - \E[L(\beta_{\ast}, \varphi(\beta_{\ast}, \varepsilon))] \\
    & \leq 2\sup_{\beta_1, \beta_2 \in \{x: \|x - \beta_{\ast}\|_2 \leq r\}}\Bigg| \frac{1}{\sum_{\nu \in \Nin{\mu}} n_{t, \nu \to \mu} } \sum_{\nu \in \Nin{\mu}} \sum_{i = 1}^{n_{t, \nu \to \mu}} L(\beta_1, \varphi(\beta_{2},  \varepsilon_{t, \nu \to \mu, i})) - \E[L(\beta_1, \varphi(\beta_2, \varepsilon))]\, \Bigg| = o_P(1). 
\end{align*}
By Assumption \ref{ass:well_specified} we see that  $\inf_{\|\beta - \beta_{\ast}\|_2 = 1/ k} \E[L(\beta, \varphi(\beta_{\ast}, \varepsilon))] > \E[L(\beta_{\ast}, \varphi(\beta_{\ast}, \varepsilon))] $. 
As a consequence, the above event occurs with vanishingly small probability, thus $\mathbb{P}(\|\hat\beta_{t, \mu}^{\ast} - \beta_{\ast}\|_2 > 1 / k) \to 0$ as the sample sizes tend to infinity. 
This holds for all $k \in \mathbb{N}_+$, hence completing the proof of consistency.  

\paragraph{Asymptotic normality. }We then prove asymptotic normality. 
By consistency, we know that $\mathbb{P}(\|\hat\beta_{t, \mu}^{\ast} - \beta_{\ast}\|_2 < r) \to 1$ as the sample sizes tend to infinity. 
Therefore, with probability $1 - o(1)$ the following first-order condition holds: 
\begin{align*}
    \sum_{\nu \in \Nin{\mu}} \sum_{i = 1}^{n_{t, \nu \to \mu}} \nabla_{\beta} L(\beta, \varphi(\beta_{\ast},  \varepsilon_{t, \nu \to \mu, i})) \Big|_{\beta = \hat\beta_{t, \mu}^{\ast}} = 0. 
\end{align*}
Rewriting the expression in integral form, we obtain
\begin{align*}
    & \frac{1}{\sum_{\nu \in \Nin{\mu}} n_{t, \nu \to \mu} }\sum_{\nu \in \Nin{\mu}} \sum_{i = 1}^{n_{t, \nu \to \mu}} \int_0^1 \nabla_{\beta}^2 L(\beta,\, \varphi(\beta_{\ast},  \varepsilon_{t, \nu \to \mu, i})) \Big|_{\beta = (1 - t)\beta_{\ast} + t \hat\beta_{t, \mu}^{\ast}} ( \hat\beta_{t, \mu}^{\ast} - \beta_{\ast}) \mathrm{d} t \\
    & = - \frac{1}{\sum_{\nu \in \Nin{\mu}} n_{t, \nu \to \mu} } \sum_{\nu \in \Nin{\mu}} \sum_{i = 1}^{n_{t, \nu \to \mu}} \nabla_{\beta} L(\beta, \varphi(\beta_{\ast},  \varepsilon_{t, \nu \to \mu, i})) \Big|_{\beta = \beta_{\ast}}. 
\end{align*}
By Assumption \ref{ass:glivenko_cantelli},
\begin{align*}
    \frac{1}{\sum_{\nu \in \Nin{\mu}} n_{t, \nu \to \mu} }\sum_{\nu \in \Nin{\mu}} \sum_{i = 1}^{n_{t, \nu \to \mu}} \int_0^1 \nabla_{\beta}^2 L(\beta,\, \varphi(\beta_{\ast},  \varepsilon_{t, \nu \to \mu, i})) \Big|_{\beta = (1 - t)\beta_{\ast} + t \hat\beta_{t, \mu}^{\ast}}  \mathrm{d} t \overset{P}{\to} \E\big[\nabla_{\beta}^2 L(\beta, \varphi(\beta_{\ast}, \varepsilon)) \big|_{\beta = \beta_{\ast}} \big]. 
\end{align*}
By the central limit theorem, 
\begin{align*}
    & - \frac{1}{\sqrt{\sum_{\nu \in \Nin{\mu}} n_{t, \nu \to \mu} }} \sum_{\nu \in \Nin{\mu}} \sum_{i = 1}^{n_{t, \nu \to \mu}} \nabla_{\beta} L(\beta, \varphi(\beta_{\ast},  \varepsilon_{t, \nu \to \mu, i})) \Big|_{\beta = \beta_{\ast}} \\
    & \overset{d}{\to} \mathsf{N} \Big( 0, \E[\nabla_{\beta}L(\beta, \varphi(\beta_\ast, \varepsilon))  \nabla_{\beta}L(\beta, \varphi(\beta_\ast, \varepsilon))^{\top} ] \big|_{\beta = \beta_{\ast}}  \Big). 
\end{align*}
The proof is done.

\subsection{Proof of \cref{thm:m_est_collapse}}\label{sec:proof_lem_m_collapse}
\noindent \textbf{Proof of $(a)$.} Using the recursive formula of \cref{thm:m_est_variance}, we obtain that
    \begin{align*}
        \Sigma_T = \, \sum_{t=1}^{T} b_{T, t} (J_{T, t+1} \otimes I_p) V_t (J_{T, t+1} \otimes I_p) ^\top + b_{T, 0} (J_{T, 1} \otimes I_p) \Sigma_0 (J_{T, 1} \otimes I_p)^\top.
    \end{align*}
    Therefore, for any $\mu \in \cM_l$, 
    \begin{align*}
        \Sigma_{T, \mu, \mu} = \, & \sum_{t=1}^{T} b_{T, t} \sum_{\nu_1, \nu_2 \in \cM_l} J_{T, t+1, \mu, \nu_1} J_{T, t+1, \mu, \nu_2} V_{t, \nu_1, \nu_2} + b_{T, 0} \sum_{\nu_1, \nu_2 \in \cM} J_{T, 1, \mu, \nu_1} J_{T, 1, \mu, \nu_2} \Sigma_{0, \nu_1, \nu_2},
    \end{align*}
    where we use the fact that $V_{t, \nu_1, \nu_2} = 0$ if $\nu_1 \in \cM_u$ or $\nu_2 \in \cM_u$.
    Since by definition, $V_{t, \nu_1, \nu_2} \succeq 0$ for all $\nu_1, \nu_2 \in \cM$, and $\Sigma_{0, \nu_1, \nu_2} = 0$ if $\nu_1 \neq \nu_2$, we deduce that
    \begin{align*}
        \Sigma_{T, \mu, \mu} \succeq \, & \sum_{t=1}^{T} b_{T, t} \sum_{\nu \in \cM_l} J_{T, t+1, \mu, \nu}^2  V_{t, \nu, \nu} + b_{T, 0} \sum_{\nu \in \cM} J_{T, 1, \mu, \nu}^2 \Sigma_{0, \nu, \nu} \\
        = \, & \left( \sum_{t=1}^{T} b_{T, t} \sum_{\nu \in \cM_l} \frac{1}{\sum_{m \in \Nin{\nu}} \bar{p}_{t, m \to \nu}} J_{T, t+1, \mu, \nu}^2 + b_{T, 0} \sum_{\nu \in \cM} \frac{1}{\bar{p}_{0, \nu}} J_{T, 1, \mu, \nu}^2 \right) V_* \\
        \stackrel{(i)}{\succeq} \, & \left( \sum_{t=1}^{T} b_{T, t} \frac{(\sum_{\nu \in \cM_l} J_{T, t+1, \mu, \nu})^2}{\sum_{\nu \in \cM_l} \sum_{m \in \Nin{\nu}} \bar{p}_{t, m \to \nu}} + b_{T, 0} \frac{(\sum_{\nu \in \cM} J_{T, 1, \mu, \nu})^2}{\sum_{\nu \in \cM} \bar{p}_{0, \nu}} \right) V_* \\
        \stackrel{(ii)}{=} \, & \left( \sum_{t=1}^{T} \frac{ b_{T, t} (\sum_{\nu \in \cM_l} J_{T, t+1, \mu, \nu})^2}{\sum_{\nu \in \cM_l} \sum_{m \in \Nin{\nu}} \bar{p}_{t, m \to \nu}} + \frac{b_{T, 0}}{\sum_{\nu \in \cM} \bar{p}_{0, \nu}} \right) V_*,
    \end{align*}
    where in $(i)$ we use Cauchy-Schwarz inequality, and $(ii)$ follows from the fact that $\sum_{\nu \in \cM} J_{T, 1, \mu, \nu} = 1$. This completes the proof of part (a).

    \vspace{0.25em}
    
    \noindent \textbf{Proof of $(b)$.} Following our calculations in the proof of part $(a)$, we have for all $\mu \in \cM_l$:
    \begin{align*}
        \sum_{\nu_1, \nu_2 \in \cM_l} J_{T, t+1, \mu, \nu_1} J_{T, t+1, \mu, \nu_2} V_{t, \nu_1, \nu_2} \preceq \,  \frac{(\sum_{\nu \in \cM_l} J_{T, t+1, \mu, \nu})^2}{\inf_{\nu \in \cM_l} \sum_{m \in \Nin{\nu}} \bar{p}_{t, m \to \nu}} V_*.
    \end{align*}
    Similarly,
    \begin{align*}
        \sum_{\nu_1, \nu_2 \in \cM} J_{T, 1, \mu, \nu_1} J_{T, 1, \mu, \nu_2} \Sigma_{0, \nu_1, \nu_2} = \, \sum_{\nu \in \cM} J_{T, 1, \mu, \nu}^2 \Sigma_{0, \nu, \nu} \preceq \, \frac{(\sum_{\nu \in \cM} J_{T, 1, \mu, \nu})^2}{\inf_{\nu \in \cM} \bar{p}_{0, \nu}} V_* = \frac{1}{\inf_{\nu \in \cM} \bar{p}_{0, \nu}} V_*.
    \end{align*}
    It finally follows that
    \begin{align*}
        \Sigma_{T, \mu, \mu} \preceq \, & \left( \sum_{t=1}^{T} \frac{b_{T,t} (\sum_{\nu \in \cM_l} J_{T, t+1, \mu, \nu})^2}{\inf_{\nu \in \cM_l} \sum_{m \in \Nin{\nu}} \bar{p}_{t, m \to \nu}} + \frac{b_{T, 0}}{\inf_{\nu \in \cM} \bar{p}_{0, \nu}} \right) V_*,
    \end{align*}
    completing the proof of part (b).

\subsection{Proof of \cref{prop:verify_glm_assumption}}\label{sec:proof_glm_assumption}

We verify Assumptions~\ref{ass:well_specified}-\ref{ass:population_regularity} one by one below.
    
    \paragraph{Verifying \cref{ass:well_specified}.}
    By definition of $L$, we know that the twice differentiability of the mapping $\beta \mapsto L (\beta, z)$ follows from the twice differentiability of $A$. Further, we have
    $$
    \E [L (\beta_1, \varphi(\beta_2, \veps))] = \, \E \left[ - A' (\beta_2^\top x) \beta_1^\top x + A (\beta_1^\top x) \right],
    $$
    whose continuity follows from the continuous differentiability of $A'$ and dominated convergence theorem. Next, we have for any fixed $\beta$:
    \begin{align*}
        \E \left[ \nabla_{\beta_1} L ( \beta, \varphi(\beta^\top x, \veps) ) \right] = \, \E \left[ - \varphi (\beta^\top x, \veps)  x + A' (\beta^\top x) x \right] = \E \left[ x ( A' (\beta^\top x) - \E [ y \vert x ] ) \right] = 0,
    \end{align*}
    since $\E [y \vert x] = A' (\beta^\top x)$. Combining this with the convexity of $L$ verifies that $\beta$ is the unique minimizer of $\beta' \mapsto \E [ L ( \beta', \varphi(\beta, \varepsilon)) ]$. Further, the population Hessian is positive definite since $A$ is strictly convex. Finally, the integrability conditions in \cref{ass:well_specified} can be directly verified.

    \paragraph{Verifying \cref{ass:L-convex}.} This is easily verified because $\mathcal{B} = \R^d$, and the convexity of $L$ follows directly from the convexity of $A$.

    \paragraph{Verifying point 1 of \cref{ass:glivenko_cantelli}.} We first show that the function class
    \begin{equation}\label{eq:loss_fct_class}
        \left\{ (x, \veps) \mapsto - \varphi (\beta_2^\top x, \veps) \cdot \beta_1^\top x + A (\beta_1^\top x) \big\vert \, (\beta_1, \beta_2) \in \Omega^2 \right\}
    \end{equation}
    is $\mu_{\veps}$-Glivenko-Cantelli. Since $\E [ \| x \|_2 ] < \infty$, the function class $\{x \mapsto \beta_1^\top x \vert \beta_1 \in \Omega\}$ is $\mu_{\veps}$-Glivenko-Cantelli. By our assumption, $\{x \mapsto - \varphi (\beta_2^\top x, \veps) \vert \beta_2 \in \Omega\}$ and $\{x \mapsto A (\beta_1^\top x) \vert \beta_1 \in \Omega\}$ are also $\mu_{\veps}$-Glivenko-Cantelli. 
Let $r_\Omega > 0$ be such that $\Omega \subset \mathsf{B} (0, r_\Omega)$, then
    \begin{align*}
        \E \Big[ \sup_{(\beta_1, \beta_2) \in \Omega^2} \left\vert L \left( \beta_1, \varphi (\beta_2, \veps) \right) \right\vert \Big] \le \, r_\Omega \E \Big[ \sup_{\beta \in \Omega} \| x \|_2 \vert \varphi (\beta^\top x, \veps) \vert \Big] + \E \Big[ \sup_{\beta \in \Omega} \vert A (\beta^\top x) \vert \Big] < \infty,
    \end{align*}
    meaning that the function class~\eqref{eq:loss_fct_class} has an $L^1$-integrable envelope. Applying \cref{lem:preserve_gc} shows that the function class \eqref{eq:loss_fct_class} is $\mu_{\veps}$-Glivenko-Cantelli. 
    
    Finally, note that $\nabla_{\beta_1}^2 L \left( \beta_1, \varphi (\beta_2^{\top} x, \veps) \right) = A'' (\beta_1^\top x) x x^\top$, which is assumed to be $\mu_{\veps}$-Glivenko-Cantelli over compact sets. This completes the verification of point 1 of \cref{ass:glivenko_cantelli}.

    \paragraph{Verifying point 2 of \cref{ass:glivenko_cantelli}.} Note that by definition:
    \begin{equation*}
        \nabla_{\beta_1} L \left( \beta_1, \varphi (\beta_2^{\top} x, \veps) \right) = \, x \left( A' (\beta_1^\top x) - \varphi (\beta_2^\top x, \veps) \right).
    \end{equation*}
    Since $\E [A' (\beta_1^\top x)^2 x x^\top] < \infty$, we know that the function class $\{x \mapsto A' (\beta_1^\top x) x \vert \beta_2 \in \Omega\}$ is $\mu_{\veps}$-Donsker. Further, we have assumed that $\{(x, \varepsilon) \mapsto \varphi (\beta_2^\top x, \veps) x \vert \beta_2 \in \Omega\}$ is $\mu_{\veps}$-Donsker, applying \cref{lem:preserve_donsker} completes the verification of point 2 of \cref{ass:glivenko_cantelli}.

    \paragraph{Verifying \cref{ass:population_regularity}.} By direct calculation, we know that
    $$
    \nabla_{\beta_1}^2 L ( \beta_1,\, \varphi (\beta_2, \, \varepsilon )) = \, A'' (\beta_1^\top x) x x^\top.
    $$
    Therefore, \cref{ass:population_regularity} directly follows from point (iv) in the statement of \cref{prop:verify_glm_assumption} and the dominated convergence theorem.

\subsection{Auxiliary lemmas from empirical process theory}
We need the following auxiliary results for the preservation of Glivenko-Cantelli and Donsker classes under composition of functions.
\begin{lem}[Preservation of Glivenko-Cantelli classes, Theorem 3 in \cite{van2000preservation}]\label{lem:preserve_gc}
    Let $\mathcal{F}_1, \dots, \mathcal{F}_k$ be $P$-Glivenko-Cantelli classes of measurable functions, and let $\varphi: \mathbb{R}^k \to \mathbb{R}$ be a continuous function. Define the class $\mathcal{H}$ as the image of these classes under $\varphi$:
    \begin{equation*}
        \mathcal{H} = \left\{ x \mapsto \varphi(f_1(x), \dots, f_k(x)) : f_j \in \mathcal{F}_j \text{ for } j=1,\dots,k \right\}.
    \end{equation*}
    Then, the class $\mathcal{H}$ is $P$-Glivenko-Cantelli, provided that $\mathcal{H}$ has an integrable envelope $H$ (i.e., $H \in L_1(P)$ such that $|h(x)| \le H(x)$ for all $h \in \mathcal{H}$).
\end{lem}
\begin{lem}[Preservation of Donsker classes, Theorem 2.10.6 in \cite{van1996weak}]\label{lem:preserve_donsker}
    Let $\mathcal{F}_1, \dots, \mathcal{F}_k$ be $P$-Donsker classes of measurable functions. Let $\varphi: \mathbb{R}^k \to \mathbb{R}$ be an $L$-Lipschitz function for some constant $L > 0$. Define the class $\mathcal{H}$ as:
    \[
        \mathcal{H} = \left\{ x \mapsto \varphi(f_1(x), \dots, f_k(x)) : f_j \in \mathcal{F}_j \right\}.
    \]
    If there exists at least one tuple $(f_1, \dots, f_k)$ such that $\varphi(f_1, \dots, f_k)$ is square-integrable (i.e., belongs to $L_2(P)$), then the class $\mathcal{H}$ is a $P$-Donsker class.
\end{lem}
The next lemma establishes the $P$-Donsker property for a wide range of function classes:
\begin{lem}\label{lem:indicator_donsker}
    Let $K \subseteq \R^p$ be a compact set, $\veps \sim \Unif [0, 1]$ be a random variable that is independent of the random vector $x \in \R^p$, and $g : \R^p \to \R$ be a measurable function. Let $\{ f_k \}_{k=1}^{\infty}$ be a sequence of non-decreasing functions taking values in $[0, 1]$. We assume the following statements hold: 
    \begin{itemize}
        \item [(i)] There exist a non-decreasing function $L(t) > 0$ and a constant $c_1 < 2$, such that for all $k \ge 1$, $\vert f_k' (t) \vert \le k^{-c_1} L (t)$ for all $t \in \R$;
        \item [(ii)] There exist a constant $c_2 > 2$ and a non-negative non-decreasing function $f$, such that $f_k (t) \le k^{-c_2} f(t)$ for all $k \ge 1$;
        \item [(iii)] For any $R > 0$:
        \begin{equation*}
            \E \left[ g(x)^2 \| x \|_2 L \big( R \| x \|_2 \big) \right] < \infty, \quad \E \left[ g(x)^2 f \big( R \| x \|_2 \big) \right] < \infty.
        \end{equation*}
    \end{itemize}
     Then, the function class
    \begin{equation*}
        \mathcal{H}_K = \, \Big\{ (x, \veps) \mapsto \sum_{k=1}^{\infty} g(x) \bone_{\veps \le f_k (\beta^\top x)}:  \beta \in K \Big\}
    \end{equation*}
    is $P$-Donsker, where $P$ represents the joint distribution of $(x, \varepsilon)$. 
\end{lem}
\begin{proof}
    We invoke the Donsker theorem (cf. Section 2.5.2 of \cite{van1996weak}) to show that $\mathcal{H}_K$ is a $P$-Donsker class. 
    Without loss of generality, we may assume that $g$ is nonnegative. 
    Otherwise, we decompose $g$ as $g = g_+ - g_-$, and note that the function classes obtained by replacing $g$ with either $g_+$ or $g_-$ in the definition of $\mathcal{H}_K$ are both $P$-Donsker. It then follows from \cref{lem:preserve_donsker} that $\mathcal{H}_K$ is also $P$-Donsker.
    To prove the lemma, it suffices to show
    \begin{equation}\label{eq:finite_bracket_integral}
        \int_0^1 \sqrt{\log N_{[ \, ]}\left(\gamma, \mathcal{H}_K \cup\{0\}, L^2 (P)\right)} \, \d \gamma < \, \infty,
    \end{equation}
    where $N_{[ \, ]}$ denotes the bracketing number. 
    We prove \cref{eq:finite_bracket_integral} by upper bounding the bracketing number $N_{[ \, ]} ( \gamma, \mathcal{H}_K \cup\{0\}, L^2 (P) )$ for all $\gamma > 0$. 
    For $\delta > 0$ (to be determined later), let $N (\delta, K) = \{ \beta_1, \cdots, \beta_M \}$ be a $\delta$-net of $K$ (under $\ell^2$-norm). Then, we know that
    \begin{equation*}
        M = \left\vert N (\delta, K) \right\vert \le \left( \frac{C_K}{\delta} \right)^p,
    \end{equation*}
    where $C_K > 0$ is a constant that depends only on $K$. 
    We next upper bound the bracketing number by constructing brackets.
    For $j = 1, \cdots, M$, define
    \begin{equation*}
        l_j (x, \veps) = \, \sum_{k=1}^{\infty} g(x) \bone_{\veps \le f_k (\beta_j^\top x - \delta \| x \|_2)}, \quad u_j (x, \veps) = \, \sum_{k=1}^{\infty} g(x) \bone_{\veps \le f_k (\beta_j^\top x + \delta \| x \|_2)}.
    \end{equation*}
    We next show that $\mathcal{H}_K \subset \cup_{j=1}^{M} [l_j, u_j]$. To this end, note that for any $\beta \in K$, there exists $j \in [M]$ such that $\| \beta - \beta_j \|_2 \le \delta$, which implies
    \begin{equation*}
        \beta_j^\top x - \delta \| x \|_2 \le \beta^\top x \le \beta_j^\top x + \delta \| x \|_2.
    \end{equation*}
    The claim then follows from our assumption that $f_k$ is non-decreasing and $g$ is non-negative. 
    Next, we upper bound $\| u_j - l_j \|_{L^2(P)}$. 
    Observe that for any fixed positive integer $B$, we have
    \begin{align*}
        & \| u_j - l_j \|_{L^2(P)} \\
        &\le \,  \sum_{k=1}^B \norm{ g(x) \big( \bone_{\veps \le f_k (\beta_j^\top x + \delta \| x \|_2) } - \bone_{\veps \le f_k (\beta_j^\top x - \delta \| x \|_2)} \big)}_{L^2 (P)} + \norm{\sum_{k=B+1}^{\infty} g(x) \bone_{\veps \le f_k (\beta_j^\top x + \delta \| x \|_2)} }_{L^2(P)} \\
        &\le \,  \sum_{k=1}^B \E \left[ g(x)^2 \big(f_k (\beta_j^\top x + \delta \| x \|_2) - f_k (\beta_j^\top x - \delta \| x \|_2 \big) \right]^{1/2} + \bigg( 2 \sum_{k=B+1}^{\infty} k \E \left[ g(x)^2 f_k (\beta_j^\top x + \delta \| x \|_2) \right] \bigg)^{1/2} \\
        &\le \,  \sum_{k=1}^B \sqrt{ 2k^{-c_1} \delta} \sqrt{\E \big[ g(x)^2 \| x \|_2 L (\beta_j^\top x + \delta \| x \|_2) \big]} + \bigg( 2 \sum_{k=B+1}^{\infty} k^{-c_2 + 1} \E \big[ g(x)^2 f (\beta_j^\top x + \delta \| x \|_2) \big] \bigg)^{1/2} \\
        &\leq \,  C B^{1 - c_1/2} \sqrt{\delta \E \big[ g(x)^2 \| x \|_2 L \big( (r_K + \delta) \| x \|_2 \big) \big]} + C B^{1 - c_2/2} \sqrt{\E \left[ g(x)^2 f \big( (r_K + \delta) \| x \|_2 \big) \right]},
    \end{align*}
    where $C > 0$ is a constant that depends only on $c_1$ and $c_2$, and $r_K > 0$ is a constant that depends only on $K$. 
    Note that setting $B \asymp \gamma^{-c_1'}$ and $\delta \asymp \gamma^{c_2'}$ for some $c_1', c_2' > 0$ gives $\| u_j - l_j \|_{L^2(P)} \leq \gamma$. We finally deduce that
    \begin{equation*}
        \log N_{[ \, ]}\left(\gamma, \mathcal{H}_K \cup\{0\}, L^2 (P)\right) \lesssim p \log \frac{1}{\delta} \lesssim p \log \frac{1}{\gamma}.
    \end{equation*}
    This establishes \cref{eq:finite_bracket_integral}, and hence completes the proof of the lemma.
\end{proof}

\end{document}